\documentclass[twoside,11pt]{article}

\usepackage{jmlr2e}

\usepackage{hyperref}
\usepackage{url}
\usepackage{bm,amsmath,amsfonts}
\usepackage[inline]{enumitem}
\usepackage{subcaption}
\usepackage{mathtools}
\usepackage{chngpage}
\usepackage[utf8]{inputenc} 
\usepackage{booktabs}       
\usepackage{nicefrac}       
\usepackage{microtype}      
\usepackage{color}      
\usepackage{multirow}
\usepackage{comment}
\usepackage{float}
\usepackage{array}
\usepackage{tabularx}
\usepackage{longtable}
\usepackage{algorithm}
\usepackage[noend]{algpseudocode}
\usepackage{xparse}

\usepackage{hyperref}
\usepackage[capitalise,noabbrev]{cleveref}

\usepackage{etoolbox}
\AtBeginEnvironment{tabularx}{\setlist[itemize, 1]{wide, leftmargin=*, itemsep=0pt, before=\vspace{-\dimexpr\baselineskip +2 \partopsep}, after=\vspace{-\baselineskip}}}

\newcommand{\hparams}{\theta}

\NewDocumentCommand{ \paramv }{ o }{%
    \hat{v}_\vecw \IfValueT{#1}{{\left(#1\right)}}%
}
\NewDocumentCommand{ \paramq }{ o }{%
    \hat{q}_\vecw \IfValueT{#1}{{\left(#1\right)}}%
}
\NewDocumentCommand{ \paramh }{ o }{%
    h_\hparams \IfValueT{#1}{{\left(#1\right)}}%
}

\newcommand{\Bo}{\mathcal{T}}

\newcommand{\Hset}{\mathcal{H}}
\newcommand{\Vset}{\mathcal{F}}
\newcommand{\Fsetall}{\mathcal{F}_{\text{all}}}
\newcommand{\Hsetall}{\mathcal{H}_{\text{all}}}
\newcommand{\Vsetsubopt}{\mathcal{F}_{\text{opt}}}
\newcommand{\Vsetsub}{\mathcal{F}_{\text{sub}}}

\newcommand{\dobj}{d_{\text{eval}}}
\newcommand{\dsol}{d}
\newcommand{\Dsmat}{D}
\newcommand{\wsol}{\vecw_{\Hset, \dsol}}
\newcommand{\vsol}{v_{\vecw_{\Hset, \dsol}}}
\newcommand{\vsoleq}{v_{\vecw_{\Vset, \dsol}}}
\newcommand{\vgen}{v}
\newcommand{\contset}[1]{c_{#1,\dsol}}
\newcommand{\ctransition}{c_\dsol}

\newcommand{\Ppigamma}{\P_{\pi, \gamma}}
\newcommand{\Ppigammatilde}{\tilde{P}_{\pi, \gamma}}
\newcommand{\eye}{I}
\newcommand{\Cmat}{\mathbf{C}}
\newcommand{\Amat}{\mathbf{A}}
\newcommand{\inv}{{-1}}
\newcommand{\invhalf}{{-1/2}}
\newcommand{\powhalf}{{1/2}}
\newcommand{\sneg}{\mathrm{-}}

\newcommand{\vhatt}[2]{\hat{v}_{#2}(#1)}
\newcommand{\phivec}{\boldsymbol{\phi}}

\newcommand{\Xmat}{\mathbf{X}}
\newcommand{\Psimat}{\boldsymbol{\Psi}}
\newcommand{\Phimat}{\boldsymbol{\Phi}}

\newcommand{\msve}{\overline{\text{VE}}}
\newcommand{\mspbe}{\overline{\text{PBE}}}
\newcommand{\msbe}{\overline{\text{BE}}}
\newcommand{\ibe}{\text{Identifiable }\overline{\text{BE}}}
\newcommand{\mstde}{\overline{\text{TDE}}}
\newcommand{\mskbe}{\overline{\text{KBE}}}
\newcommand{\msre}{\overline{\text{RE}}}

\newcommand{\proofspace}{\par\vspace{-0.5cm}}

\newcommand{\mweight}{m}

\def\Re{\mathbb{R}}

\def\Nat{{\rm I\kern\pIR N}}
\def\argmax{\mathop{\rm arg\,max}}
\def\argmin{\mathop{\rm arg\,min}}

\def\log{\mathop{\rm log}}

\def\Var{{\bf Var}}

\newcommand{\EE}[1]{\exptE\left[#1\right]}
\def\diag{\mathop{\rm diag}}
\newcommand{\defeq}{\overset{\text{\tiny def}}{=}}

\newcommand{\xdim}{d}

\newcommand{\RR}{\mathbb{R}}

\def\A{{\mathcal{A}}}

\def\E{{\mathcal{E}}}

\def\P{{\mathcal{P}}}

\def\R{{\mathcal{R}}}
\def\S{{\mathcal{S}}}

\newcommand{\States}{\S}
\newcommand{\Actions}{\A}

\def\vec0{{\boldsymbol{0}}}
\def\vecb{{\boldsymbol{b}}}

\def\vecv{{\boldsymbol{v}}}
\def\vecw{{\boldsymbol{w}}}
\def\vecx{{\boldsymbol{x}}}
\def\vech{{\boldsymbol{h}}}
\def\vecy{{\boldsymbol{y}}}
\def\vecz{{\boldsymbol{z}}}

\newtheorem{assumption}[theorem]{Assumption}

\newcommand{\ra}{\rightarrow}
\newcommand{\beq}{\begin{equation}}
\newcommand{\eeq}{\end{equation}}
\newcommand{\beqa}{\begin{eqnarray}}
\newcommand{\eeqa}{\end{eqnarray}}
\newcommand{\beqan}{\begin{eqnarray*}}
\newcommand{\eeqan}{\end{eqnarray*}}
\newcommand{\ben}{\begin{eqnarray*}}
\newcommand{\een}{\end{eqnarray*}}

\def\tr{^\top \!}

\renewcommand{\EE}[2]{\mathbb{E}_{#1\! \!}\left[#2\right]}

\newcommand{\CEE}[3]{\EE{#1}{{#2}~\middle\vert~{#3}}}
\renewcommand{\CEE}[3]{\EE{#1}{{#2}\mid{#3}}}
\def\CE#1#2{\CEE{\,}{#1}{#2}}
\def\CEb#1#2{\CEE{b}{#1}{#2}}
\def\CEpi#1#2{\CEE{\pi}{#1}{#2}}

\def\E#1{\EE{\,}{#1}}
\def\Epi#1{\EE{\pi}{#1}}

\newcommand{\zerovec}{\mathbf{0}}

\let\svthefootnote\thefootnote
\newcommand\blfootnote[1]{%
  \let\thefootnote\relax%
  \footnotetext{#1}%
  \let\thefootnote\svthefootnote%
}

\firstpageno{1}

\begin{document}

\title{A Generalized Projected Bellman Error \\for Off-policy Value Estimation in Reinforcement Learning}

\author{\name Andrew Patterson \email ap3@ualberta.ca
	\AND
    \name Adam White \email amw8@ualberta.ca
    \AND
    \name Martha White \email whitem@ualberta.ca\\
    \addr Department of Computing Science and the Alberta Machine Intelligence Institute (Amii)\\
    University of Alberta\\
    Edmonton, Alberta, Canada\\
}
\editor{George Konidaris}

\maketitle

\begin{abstract}
Many reinforcement learning algorithms rely on value estimation, however, the most widely used algorithms---namely temporal difference algorithms---can diverge under both off-policy sampling and nonlinear function approximation.
Many algorithms have been developed for off-policy value estimation based on the linear mean squared projected Bellman error ($\mspbe$) and are sound under linear function approximation.
Extending these methods to the nonlinear case has been largely unsuccessful.
Recently, several methods have been introduced that approximate a different objective---the mean-squared Bellman error ($\msbe$)---which naturally facilitate nonlinear approximation.
In this work, we build on these insights and introduce a new generalized $\mspbe$ that extends the linear $\mspbe$ to the nonlinear setting.
We show how this generalized objective unifies previous work and obtain new bounds for the value error of the solutions of the generalized objective.
We derive an easy-to-use, but sound, algorithm to minimize the generalized objective, and show that it is more stable across runs, is less sensitive to hyperparameters, and performs favorably across four control domains with neural network function approximation.
\end{abstract}

\begin{keywords}
  Off-policy learning, Temporal difference learning, Reinforcement learning
\end{keywords}


\section{Introduction}

Value functions play a central role in reinforcement learning: value-based methods act greedily towards a learned action value function; policy gradient methods often employ a critic to estimate the value of the policy in order to reduce variance; model-based methods often employ learned values for tree search, model-based value prediction, or simulation based planning to update value estimates from hypothetical experience.
Fundamental improvements in algorithms for learning value functions can have a significant impact in a variety of problem settings; as such, many algorithms have been developed to improve learning values functions.
This includes a variety of variance reduction improvements for off-policy temporal difference algorithms \citep{precup2000eligibility,munos2016safe,mahmood2017multistep};
gradient TD methods with linear function approximation \citep{sutton2009fast,mahadevan2014proximal,liu2016proximal,ghiassian2020gradient} and nonlinear function approximation \citep{maei2009convergent};
and algorithms using approximations to the mean squared Bellman error ($\msbe$) \citep{dai2017learning,dai2018sbeed,feng2019kernel}.

The most widely used value function learning algorithm is also one of the oldest: temporal difference (TD) learning.
The continued popularity of TD stems from both (a) the algorithms simplicity and empirical performance, and (b) the lack of technical tools required to improve it.
TD, however, does not follow the gradient of any known objective function \citep{baird1995residual,antos2008learning}, and without a clear objective for TD, it was difficult to extend the algorithm.
Related TD-like Residual gradient algorithms directly optimize the $\msbe$, but suffers from the double sampling problem \citep{baird1995residual,scherrer2010should}.
Without a strategy to optimize the $\msbe$, in the absence of a simulator, it was difficult to pursue the $\msbe$ as an alternative.

Value function learning algorithms that converge under general conditions fall into two camps defined by the objectives they optimize.
The first is the formalization of the objective underlying TD---the mean squared projected Bellman error ($\mspbe$) \citep{antos2008learning,sutton2009fast}---which projects the Bellman error into the space spanned by the function approximator.
Several algorithms were introduced to minimize the $\mspbe$, most of which built on the originally introduced variants: GTD2 and TD with gradient corrections (TDC) \citep{sutton2009fast}.
Most of these algorithm, however, are limited to linear function approximation because the $\mspbe$ is defined only for the linear case with the exception of nonlinear GTD \citep{maei2009convergent}, which used a locally linear approximation for the projection.
The gradient computation, however, requires a Hessian-vector product and has not been widely adopted.
As such, the extension of the linear $\mspbe$ for nonlinear function approximation remains open.

The second camp of methods pursue the $\msbe$, which is naturally defined for nonlinear function approximation.
The idea is to use a conjugate form for the $\msbe$ \citep{dai2017learning}, to reformulate it into a saddlepoint problem with an auxiliary variable, removing the double sampling problem because the resulting saddlepoint uses an estimator in place of one of the samples.
The SBEED algorithm \citep{dai2018sbeed} later extended the conjugate $\msbe$ to the control case by using a smoothed Bellman optimality operator and parameterizing both the policy and value function estimates.

In this work, we bring together these threads of research and provide a novel objective for nonlinear value function estimation in both prediction and control. In particular, we introduce a generalized $\mspbe$, by using the conjugate reformulation of the $\msbe$.
To understand the role of objectives in off-policy value estimation, it is useful to first understand the algorithmic development in off-policy learning that led to several of these objectives. We first lay out this development, and then describe the contributions in this work.

\subsection{A Short History of Off-policy Temporal Difference Learning}

The story of off-policy learning begins with Q-learning \citep{watkins1989learning}.
Off-policy learning, in some sense, allows side-stepping the exploration-exploitation tradeoff: the agent makes use of an independent exploration policy to select actions while learning the value function for the optimal policy.
The exploration policy need not maximize reward, but can instead select actions in order to generate data that improves the optimal policy through learning.
Ultimately, the full potential of Q-learning---and this ability to learn about one policy from a data generated by a totally different exploration---proved limited. Baird's well-known counterexample \citep{baird1995residual} provided a clear illustration of how, under function approximation, the weights learned by Q-learning can become unstable.\footnote{The action-value star MDP can be found in the errata of Baird's paper \citep{baird1995residual}.}  Baird's counterexample highlights that divergence can occur when updating off-policy with function approximation and with bootstrapping (as in temporal difference (TD) learning);
even when learning the value function of a fixed target policy from a fixed data-generating policy.

The instability of TD methods is caused by how we correct the updates to the value function to account for the potential mismatch between the target and exploration policies.
Off-policy training involves estimating the total expected future reward (the value function) that would be observed while selecting actions according to the target policy with training data (states, actions, and rewards) generated while selecting actions according to an exploration policy.
One approach to account for the differences between the data produced by these two policies is based on using importance sampling corrections: scaling the update to the value function based on the agreement between the target and exploration policy at the current state. If the target and exploration policy would always select the same action in a state, then they completely agree. Alternatively, if they never take the same action in a state they completely disagree. More generally, there can be degrees of agreement. We call this approach \emph{posterior corrections} because the corrections account for the mismatch between policies ignoring the history of interaction up to the current time step---it does not matter what the exploration policy has done in the past.

Another approach, called \emph{prior corrections}, uses the history of agreement between the exploration and target policy in the update.
The likelihood that the trajectory could have occurred under the target policy is used to scale the update. The most extreme version of prior corrections uses the trajectory of experience from the beginning of time, corresponding to what has sometimes been referred to as the \emph{alternative life} framework.
Prior and posterior corrections can be combined to achieve stable off-policy TD updates \citep{precup2000eligibility}, though finite variance of the updates cannot be guaranteed \citep{precup2001offpolicy}.

Learning about many different policies in parallel has long been a primary motivation for off-policy learning, and this usage suggested that perhaps prior corrections are not essential.
Several approaches require learning many value functions or policies in parallel, including approaches based on option models \citep{sutton1999mdps}, predictive representations of state \citep{littman2002predictive, tanner2005td,sutton2011horde}, and auxiliary tasks \citep{jaderberg2016reinforcement}. In a parallel learning setting, it is natural to estimate the future reward achieved by following each target policy until termination from the states encountered during training---the value of taking \emph{excursions} from the behavior policy.

This excursion view of off-policy learning lead to the development of an new objective function for off-policy TD learning called the mean squared projected Bellman error ($\mspbe$).
The resultant family of \emph{Gradient}-TD methods that optimize the $\mspbe$ use posterior corrections via importance sampling and are guaranteed to be stable under function approximation \citep{sutton2009fast,maei2011gradient}.
This new excursion objective has the same fixed point as TD, thus Gradient-TD methods converge to the same solution in the cases for which TD converges.
The major critiques of these methods are (1) the additional complexity due to a second set of learned parameters, and (2) the variance due to importance sampling corrections.

A wide array of algorithms arose to address these limitations and improve sample efficiency.
The strategies include (1) using a saddlepoint reformulation that facilitates the use of optimization accelerations \citep{liu2016proximal,du2017stochastic,liu2015finitesample} and a finite sample analysis \citep{touati2018convergent}, (2) hybrid TD methods that behave like TD when sampling is on-policy \citep{hackman2013faster, white2016investigating}, and (3)
variance reduction methods for posterior corrections, using different eligibility trace parameters \citep{precup2000eligibility,munos2016safe,wang2016dueling,mahmood2017multistep}.

The second major family of off-policy methods revisits the idea of using prior corrections.
The idea is to incorporate prior corrections only from the beginning of the excursion.
In this way, the values of states that are more often visited under the target policy are emphasized, but the high variance of full prior corrections---to the beginning of the episode---is avoided.
The \emph{Emphatic} TD($\lambda$) algorithm is based on this idea\citep{sutton2016emphatic} and was later extended to reduce the variance of the emphasis weights \citep{hallak2016generalized}.
These Emphatic TD methods are guaranteed stable under both on-policy and off-policy sampling with linear function approximation \citep{sutton2016emphatic,yu2015convergence,hallak2016generalized}.

Practitioners often bemoan both the complexity and poor performance of these sound off-policy algorithms, favouring instead to combine off-policy TD with heuristics to reduce variance.
Many of the algorithms discussed above have not been developed for the case of nonlinear function approximation.
On the other hand, off-policy updating is used pervasively in deep learning architectures for learning from a replay buffer, demonstrations, and trajectories generated by other agents (i.e., asynchronous architectures).
In order to get the benefits of off-policy updates and mitigate variance we can truncate the corrections: the Retrace algorithm does exactly this~\citep{munos2016safe,wang2016dueling}.
Retrace and the much older Tree Backup algorithm can be viewed as adapting the eligibility trace to reduce variance, similar to the ABQ algorithm~\citep{mahmood2017multistep}.
These three methods can be extended to utilize gradient corrections to achieve stability, but this would only yield another batch of theoretically sound but practically ignored methods.
In this paper, we explore an alternative path: extending gradient TD methods to both control and nonlinear function approximation through the development of a new objective function.

\subsection{Revisiting the Objective for Value Estimation}

Central to all this development has been the definition of a precise objective: the linear $\mspbe$. The proliferation of algorithms, however, has resulted in a focus on comparing and discussing specific algorithms \citep{dann2014policy,white2016investigating,touati2018convergent}. This breeds confusion in exactly how the algorithms are related and what objective they are attempting to optimize. By separating the objective and the optimization strategies, we can more systematically understand the differences in solutions under objectives and how effective the optimization strategies are at reaching those solutions.

Two key questions emerge about the specification of the objective: the form of the Bellman error and the weighting across states.
The first question is really a revisitation of the long-standing question about whether we should use the $\msbe$ or the $\mspbe$. In terms of solution quality, it is mixed: the $\msbe$ avoids some counterexamples that exist for the $\mspbe$ \citep{scherrer2010should}, however the $\mspbe$ often produces a better solution \citep{dann2014policy,scherrer2010should,sutton2018reinforcement}.
Further, the $\msbe$ has been shown to have an identifiability problem \citep{sutton2018reinforcement}. Though the evidence comparing the $\msbe$ and $\mspbe$ is inconclusive, the $\mspbe$ has been the default choice because many algorithms optimize it. The $\msbe$, on the other hand, is typically avoided due to the double sampling problem, where it is unclear how to obtain an unbiased sample of the gradient without a simulator.

Recently, however, this technical challenge has been overcome with the introduction of a saddlepoint form for the $\msbe$ \citep{dai2017learning,dai2018sbeed, feng2019kernel}. The resulting algorithms are similar to the saddlepoint algorithms for the linear $\mspbe$: a second estimator is used to estimate a part of the gradient. The $\msbe$ is particularly alluring, as it equally applies to the linear and nonlinear value estimation settings. The $\mspbe$, on the other hand, was defined for the linear setting, due to the difficulty in computing the projection operator for the nonlinear setting.
These potential advantages, as well as a viable strategy for optimizing the $\msbe$, motivates the utility of answering which of these objectives might be preferable.

The second question is about the weighting on states in the objective, which must be specified for either the $\msbe$ or $\mspbe$. The importance of the weighting on states has been well-recognized for many years, and is in fact the reason TD diverges on Baird's counterexample: using the stationary distribution of the behavior policy, rather than the target policy, results in an iterative update that is no longer a contraction. The emphatic algorithms were introduced to adjust this weighting to ensure convergence.

This reweighting, however, is not only about convergence; it also changes the fixed point and the quality of the solution. There has been some work investigating the impact of the state weighting on the optimal solution, not just on the behavior of the updates themselves. The most stark result is a simple example where using the solution to the $\mspbe$ can result in an arbitrarily poor mean squared value error ($\msve$) \citep{kolter2011fixed}. Several later results extended on-policy bounds on the $\msve$ to the off-policy setting, showing that $\msve$ could be bounded in the off-policy setting using careful choices on the state weighting---namely using state weightings given by Emphatic TD \citep{hallak2016generalized, white2017unifying}. Despite these insights, the role of the weighting on the quality of the solution in practice remains open. A natural question is: how do we choose the state weighting in the objective, and how much does it matter?

\subsection{Contributions}

In this work, we focus on the question: what objective should we use for off-policy value estimation? We first summarize many existing off-policy algorithms as optimizing the linear $\mspbe$ in different ways and in some cases with different state weightings. This summary separates the optimization strategy from the definition of the objective, allowing us to move away from specific algorithms to understanding the differences in fixed points obtained under the different objectives.
We then propose a generalized $\mspbe$ that uses a generalized projection operator that both extends the $\mspbe$ to the nonlinear setting and unifies the $\msbe$ and linear $\mspbe$ under one objective. Using these insights, we provide the following specific contributions:
\begin{enumerate}[itemsep=0em ]
\item We show how the generalized $\mspbe$ helps resolve the non-identifiability of the $\msbe$, where a particular projection in the generalized $\mspbe$ provides an Identifiable $\msbe$.
\item We highlight the role of the state weighting in this generalized objective, both extending theoretical results bounding the $\msve$ and empirically showing that the emphatic weighting can significantly improve the quality of the solution.
\item We show that these insights also extend to control by defining an objective for learning action values with (soft) maximal operators. We use this objective to derive a sound gradient variant of Q-learning.
\item We exploit the connection to the linear $\mspbe$ to develop a more effective algorithm for the generalized $\mspbe$ using gradient corrections rather than the saddlepoint update.
\item Finally, we demonstrate the utility of these prediction and control algorithms in several small benchmark domains.
\end{enumerate}

\section{Problem Definition and Background}\label{sct:ProblemDefinition}

We consider the problem of learning the value function for a given policy under the Markov Decision Process (MDP) formalism. The agent interacts with the environment over a sequence of discrete time steps, $t=1, 2, 3, \ldots$. On each time step, the agent observes a partial summary of the state $S_t \in \S$ and selects an action $A_t \in \A$.  In response, the environment transitions to a new state $S_{t+1}$ according to transition function $P(S_{t+1} | S_t, A_t)$ and emits a scalar reward $R_{t+1} \in \R$. The agent selects actions according to a stochastic, stationary \emph{target policy} $\pi: \S \times \A \rightarrow [0,1]$.

We study the problem of \emph{value estimation}: the computation or estimation of the expected discounted sum of future rewards for policy $\pi$ from every state. The \emph{return} at time $t$, denoted $G_t \in \mathbb{R}$, is defined as the discounted sum of future rewards. The discount factor can be variable, dependent on the entire transition: $\gamma: \States \times \Actions \times \States \rightarrow [0,1]$, with $\gamma_{t+1} \defeq \gamma(S_t, A_t, S_{t+1})$.
The return is defined as
\begin{align*}
G_t &\defeq R_{t+1} + \gamma_{t+1}R_{t+2} + \gamma_{t+1}\gamma_{t+2}R_{t+3} + \gamma_{t+1}\gamma_{t+2}\gamma_{t+3}R_{t+4} + \hdots \nonumber \\
 &= R_{t+1} + \gamma_{t+1}G_{t+1}
.
\end{align*}
When $\gamma_t$ is constant, $\gamma_c$, we get the familiar return $G_t = R_{t+1} + \gamma_c R_{t+2} + \gamma_c^2R_{t+3} + \ldots$. Otherwise, variable $\gamma_t$ can discount per transition, including encoding termination when it is set to zero. This generalization ensures we can discuss episodic problems without introducing absorbing states \citep{white2017unifying}. It also enables the derivations and theory to apply to both the continuing and episodic settings.
The \emph{value function} $v:\S\ra\Re$ maps each state to the expected return under policy $\pi$ starting from that state
\begin{equation}
v_{\pi}(s) \defeq \CEpi{G_t}{S_t=s} \text{ , for all $s\in\S$} \label{eq:value}
\end{equation}
where the expectation operator $\Epi{\cdot}$ reflects that the distribution over future actions is given by $\pi$, to distinguish from a potentially different behavior policy.

In this paper, we are interested in problems where the value of each state cannot be stored in a table; instead we must approximate the value with a parameterized function.
The approximate value function $\paramv[s_t]$ can have arbitrary form as long as it is everywhere differentiable with respect to the weights $\vecw \in \Re^d$. Typically, the number of components in $\vecw$ is much fewer than the number of possible states ($d\ll|\S|$), thus $\hat{v}$ will generalize values across many states in $\S$.
An important special case is when the approximate value function is linear in the parameters and in features of the state. In particular, the current state $S_t$ is converted into feature vector $\vecx_t \in \Re^d$ by some fixed mapping $\vecx: \S \rightarrow \Re^d$. The value of the state can then be approximated with an inner product:
$\paramv[s_t] = \vecw\tr\vecx_t \approx v_\pi(s_t)$.
Another typical parameterization for $\paramv[s_t]$ is a neural network where $\vecw$ consists of all the weights in the network.
We refer to $\vecw$ exclusively as the \emph{weights}, or weight vector, and reserve the word \emph{parameter}
for variables like the discount-rate and stepsize parameters.

We first describe how to learn this value function for the on-policy setting, where the behavior policy equals the target policy.
Temporal difference learning (Sutton, 1988) is perhaps the best known and most successful approach for estimating $\hat{v}$ directly from samples generated while interacting with the environment. Instead of waiting until the end of a trajectory to update the value of each state, the TD($\lambda$) algorithm adjusts its current estimate of the weights toward the difference between the discounted estimate of the value in the next state and the estimated value of the current state plus the reward along the way:
\begin{equation}\label{eq:delta}
\delta_t \defeq \delta(S_t, A_t, S_{t+1}) \defeq R_{t+1} + \gamma_{t+1} \paramv[S_{t+1}] - \paramv[S_t]
.
\end{equation}
We use the value function's own estimate of future reward as a placeholder for the future rewards defining $G_t$ that are not available on time-step $t+1$. In addition, the TD($\lambda$) algorithm also maintains an eligibility trace vector $\vecz_t \in \Re^d$ that stores a fading trace of recent feature activations. The components of  $\vecw_t$ are updated on each step proportional to the magnitude of the trace vector. This simple scheme allows update information to propagate impacting the value estimates for previously encountered states.

The update equations for TD($\lambda$) are straightforward:
\begin{align*}
\vecw_{t+1} \leftarrow& ~\vecw_t + \alpha \delta_t z_t\\
\vecz_t \leftarrow& ~\gamma_{t} \lambda \vecz_{t-1}  +\nabla \paramv[S_t]
,
\end{align*}
where $\alpha\in\Re$ is the scalar stepsize parameter that controls the speed of learning and $\lambda\in\Re$ controls the length of the eligibility trace. Under linear function approximation, intermediate values of $\lambda$ between zero and one often perform best.
TD($\lambda$) is only sound for the linear function approximation setting, though TD(0) is often used outside the linear setting and often obtains good performance.

These updates need to be modified for the \emph{off-policy case},
where the agent selects actions according to a \emph{behavior policy} $b: \States \times \Actions \rightarrow [0,1]$ that is different from the target policy.
The value function for target policy $\pi$ is updated using experience generated from a behavior policy that is \emph{off}, away, or distant from the target policy.
For example, consider the most well-known off-policy algorithm, Q-learning. The target policy might be the one that maximizes future discounted reward, while the behavior is nearly identical to the target policy, but instead selects an exploratory action on each time step with some small probability. More generally, the target and behavior policies need not be so closely coupled.
The main requirement linking these two policies is that the behavior policy \emph{covers} the actions selected by the target policy in each state visited by $b$, that is: $b(a|s) > 0$ for all states and actions in which $\pi(a|s) > 0$.


\section{Off-policy Corrections and the Connection to State Weightings}

The key problem in off-policy learning is to estimate the value function for the target policy conditioned on samples produced by actions selected according to the behavior policy. This is an instance of the problem of estimating an expected value under some target distribution from samples generated by some other behavior distribution. In statistics, this problem can be addressed with importance sampling, and most methods of off-policy reinforcement learning use such corrections.

There are two distributions that we could consider correcting: the distribution over actions, given the state, and the distribution over states. When observing a transition $(S,A,S',R)$ generated by taking the action according to $b(\cdot | S)$, we can correct the update for that transition so that, in expectation, we update as if actions were taken according to $\pi(\cdot | S)$. However, these updates are still different than if we evaluated $\pi$ on-policy, because the frequency of visiting state $S$ under $b$ will be different than under $\pi$. All methods correct for the distribution over actions (posterior corrections), given the state, but several methods also correct for the distribution over states (prior corrections).

In this section, we discuss the difference in the underlying objective for these updates. We give a brief introduction to prior and posterior corrections, and provide a more in-depth example of the differences in their updates in Appendix~\ref{app_example_prior} for readers interested in a more intuitive explanation. We show that the underlying objective differs in terms of the state weighting and discuss how emphatic weightings further modify this state weighting in the objective.

\subsection{Posterior Corrections}
The most common approach to developing sound off-policy TD algorithms makes use of posterior corrections based on importance sampling. One of the simplest examples of this approach is \emph{Off-policy TD($\lambda$)}. The procedure is easy to implement and requires constant computation per time step, given knowledge of both the target and behavior policies. On the transition from $S_t$ to $S_{t+1}$ via action $A_t$, we compute the ratio between $\pi$ and $b$:
\begin{equation}
\rho_t \defeq \rho(A_t | S_t) \defeq \frac{\pi(A_t|S_t)}{b(A_t|S_t)}
.
\end{equation}
These importance sampling corrections simply weight the eligibility trace update:
\begin{align}\label{eq:offTD}
\vecw_{t+1} \leftarrow& ~\vecw_t + \alpha \delta_t \vecz^\rho_t \nonumber \\
\vecz^\rho_t \leftarrow& ~\rho_t(\gamma_t \lambda \vecz^\rho_{t-1}  +\vecx_t)
,
\end{align}
where $\delta_t$ is defined in Equation~\eqref{eq:delta}. This way of correcting the sample updates ensures that the approximate value function $\hat{v}$ estimates the expected value of the return as if the actions were selected according to $\pi$.
Posterior correction methods use the target policy probabilities for the selected action to correct the update to the value of state $S_t$ using only the
data from time step $t$ onward.
Values of $\pi$ from time steps prior to $t$ have no impact on the correction.

Many modern off-policy prediction methods use some form of posterior corrections; including the Gradient-TD methods, Tree Backup($\lambda$), V-trace($\lambda$), and Emphatic TD($\lambda$). In fact, all off-policy prediction methods with stability guarantees make use of posterior corrections via importance sampling. Only correcting the action distribution, however, does not necessarily provide stable updates, and Off-policy TD($\lambda$) is not guaranteed to converge (Baird, 1995). To obtain stable Off-policy TD($\lambda$) updates, we need to consider corrections to the state distribution as well, as we discuss next.

\subsection{Prior Corrections}
We can also consider correcting for the differences between the target and behavior policy by using the agreement between the two over a trajectory of experience. \emph{Prior correction} methods keep track of the product of either $\prod_{k=1}^t \pi(A_k|S_k)$ or $\prod_{k=1}^t \frac{\pi(A_k|S_k)}{b(A_k|S_k)}$, and correct the update to the value of $S_t$ using the current value of the product. Therefore, the value of $S_t$ is only updated if the product is not zero, meaning that the behavior policy never selected an action for which $\pi(A_k|S_k)$ was zero---the behavior never completely deviated from the target policy.

To appreciate the consequences of incorporating these prior corrections into the TD update consider a state-value variant of \citet{precup2000eligibility} Off-policy TD($\lambda$) algorithm:
\begin{align}\label{eq:precup2001}
\vecw_{t+1} \leftarrow& ~\vecw_t + \alpha \delta_t \vecz^\rho_t \nonumber \\
\vecz^\rho_t \leftarrow& ~\rho_t \left(\gamma_t \lambda \vecz_{t-1}  +\prod_{k=1}^{t-1} \rho_k \vecx_t \right)
\end{align}
where $\vecz^\rho_0 = \zerovec$.
We refer to the above algorithm as \emph{Alternative-life TD($\lambda$)}, as the product of importance sampling ratios in the trace simulate an agent experiencing an entirely different trajectory from the beginning of the episode---an alternative life.
Note that the trace is always reinitialized at the start of the episode.

The Alternative-life TD($\lambda$) algorithm has been shown to converge under linear function approximation, but in practice exhibits unacceptable variance \citep{precup2001offpolicy}. The Emphatic TD($\lambda$) algorithm, on the other hand, provides an alternative form for the prior corrections that is lower variance but still guarantees convergence.
To more clearly explain why, next we will discuss how different prior corrections account for different weightings in optimizing the mean squared Projected Bellman Error ($\mspbe$).

\subsection{The Linear $\mspbe$ under Posterior and Prior Corrections}\label{sct:ASecondPerspectiveOnOffPolicyCorrections}

In this section, we describe how different forms of prior corrections correspond to optimizing similar objectives, but with different weightings over the state. This section introduces the notation required to explain the many algorithms that optimize the linear $\mspbe$ and clarifies convergence properties of algorithms, including which algorithms converge and to which fixed point. We start with the linear $\mspbe$ as most algorithms have been designed to optimize it. In the next section, we extend beyond the linear setting to discuss the generalized $\mspbe$.

We begin by considering a simplified setting with $\lambda = 0$ and a simplified variant of the linear $\mspbe$, called the NEU (norm of the expected TD update \citep{sutton2009fast})
\begin{align}
\text{NEU}(\vecw) = \Big{\|} \sum_{s\in\S}d(s)\CEpi{\delta(S, A, S')\vecx(S)}{S=s} \Big{\|}_2^2
,
\label{eq:onPolicyTDObjective}
\end{align}
where $d: \States \rightarrow [0, \infty)$ is a positive weighting on the states. Note the transition $(S, A, S')$ is random as is the TD-error.
Equation~\ref{eq:onPolicyTDObjective} does not commit to a particular sampling strategy. If the data is sampled on-policy, then $d=d_\pi$, where $d_\pi: \States \rightarrow [0,1]$ is the stationary distribution for $\pi$ which represents the state visitation frequency under behavior $\pi$ in the MDP.\@
If the data is sampled off-policy, then the objective might instead be weighted by the state visitation frequency under $b$, i.e., $d = d_b$.

We first consider how to sample the NEU for a given a state.
The behavior selects actions in each state $s$, so the update $\delta_t \vecx_t$ must be corrected for the action selection probabilities of $\pi$ in state $s$, namely a posterior correction:
\begin{align}
\CEpi{\delta(S_t, A_t, S_{t+1}) \vecx(S_t)}{S_t=s}
&=\sum_{a\in\A}\pi(a|s)\sum_{s'\in\S}P(s' | s, a)\delta(s, a, s')\vecx(s)\nonumber \\
&=\sum_{a\in\A}\frac{b(a|s)}{b(a|s)}\pi(a|s)\sum_{s'\in\S}P(s' | s, a)\delta(s, a, s')\vecx(s)\nonumber \\
&=\sum_{a\in\A}b(a|s)\sum_{s'\in\S}P(s' | s, a)\frac{\pi(a|s)}{b(a|s)}\delta(s, a, s')\vecx(s)\nonumber \\
&=\CEb{\rho(A_t| S_t)\delta(S_t, A_t, S_{t+1})\vecx(S_t)}{S_t=s}\label{eq:EqualityOfBehaviorAndTargetExpectations}
.
\end{align}
Therefore, the update $\rho_t \delta_t \vecx_t$ provides an unbiased sample of the desired
expected update $\CEpi{\delta(S_t, A_t, S_{t+1}) \vecx(S_t)}{S_t=s}$.
Most off-policy methods (except Q-learning and Tree-backup) use these posterior corrections.

We can also adjust the state probabilities from $d_b$ to $d_\pi$, using prior corrections.
Consider the expected update using prior corrections starting in $s_0$ and taking two steps following $b$:
\begin{align*}
&\CEb{\rho_0 \rho_1 \CEpi{\delta(S_t, A_t, S_{t+1}) \vecx(S_t)}{S_t = S_2}}{S_0 = s_0}\\
&=\CEb{\rho_0 \sum_{a_1 \in \Actions} b(a_1 | S_{1}) \sum_{s_2 \in \States} P(s_2| S_1, a_1) \rho(a_1 | S_1)  \CEpi{\delta(S_t, A_t, S_{t+1}) \vecx(S_t)}{S_t=s_2}}{S_0 = s_0}\\
&=\CEb{\rho_0 \sum_{a_1 \in \Actions} \pi(a_1 | S_1) P(s_2 | S_1, a_1) \CEpi{\delta(S_t, A_t, S_{t+1}) \vecx(S_t)}{S_t=s_2}}{S_0 = s_0}\\
&=\CEb{\rho_0 \CEpi{\delta(S_t, A_t, S_{t+1}) \vecx(S_t)}{S_{t-1}=S_1}}{S_0 = s_0}\\
&=\sum_{a_0 \in \Actions} \pi(s_0, a_0) \sum_{s_1 \in \States} P(s_1 | s_0, a_0) \CEpi{\delta(S_t, A_t, S_{t+1}) \vecx(S_t)}{S_{t-1}=s_1}\\
&=\CEpi{\delta(S_t, A_t, S_{t+1}) \vecx(S_t)}{S_0 = s_0}
.
\end{align*}
More generally, we get
\begin{align*}
\mathbb{E}_b \biggl[ \rho_1 \ldots \rho_{t-1} \CEpi{\delta(S_t, A_t, S_{t+1}) \vecx(S_t)}{S_t=s}|S_0 = s_0\biggr ]
&=\CEpi{\delta(S_t, A_t, S_{t+1}) \vecx(S_t)}{S_0 = s_0}
.
\end{align*}
These corrections adjust the probability of the sequence from the beginning of the episode as if policy $\pi$ had taken actions $A_1, \ldots, A_{t-1}$ to get to $S_t$, from which we do the TD($\lambda$) update.

A natural question is which objective should be preferred: the alternative-life ($d \propto d_\pi$) or the excursions objective ($d \propto d_b$). As with all choices for objectives, there is no crisp answer. The alternative-life objective is difficult to optimize because prior corrections can become very large or zero---causing data to be discarded---and is thus high variance. There is some work that directly estimates $d_\pi(s)/d_b(s)$ and pre-multiplies to correct the update \citep{hallak2017consistent,liu2020offpolicy}; this approach, however, requires estimating the visitation distributions. On the other hand, the fixed-point solution to the excursion objective can be arbitrarily poor compared with the best value function in the function approximation class if there is a significant mismatch between the behavior and target policy \citep{kolter2011fixed}.
Better solution accuracy can be achieved using an excursion's weighting that includes $d_b$, but additionally reweights to make the states distribution closer to $d_\pi$ as is done with Emphatic TD($\lambda$). We discuss this alternative weighting in the next section.

The above discussion focused on a simplified variant of the $\mspbe$ with $\lambda = 0$, but the intuition is the same for the $\mspbe$ with $\lambda > 0$. To simplify notation we introduce a conditional expectation operator:
\begin{equation*}
\mathbb{E}_d [Y] = \sum_{s\in\States} d(s) \mathbb{E}_\pi [Y~|~S=s]
.
\end{equation*}
We can now define
\vspace{-0.1cm}
\begin{align*}
\Cmat &\defeq \mathbb{E}_d [\vecx(S) \vecx(S)^\top]\\
\Amat &\defeq -\mathbb{E}_d [(\gamma(S,A,S') \vecx(S') - \vecx(S)) \vecz(S)^\top]\\
\vecb &\defeq  \mathbb{E}_d [R(S,A,S') \vecz(S)^\top]
\end{align*}
where the expected eligibility trace $\vecz(S)\in\mathbb{R}^k$ is defined recursively $\vecz(s) \defeq \vecx(s) + \lambda \mathbb{E}_\pi[ \gamma(S_{t-1},A_{t-1},S_t)\vecz(S_{t-1}) | S_t = s]$.
We can write the TD($\lambda$) fixed point residual as:
\begin{equation}
\mathbb{E}_d [\delta(S, A, S')\vecz(S)] = \sneg\Amat \vecw + \vecb
\label{eq_td_A}
\end{equation}
where $\mathbb{E}_{d} [\delta(S, A, S')\vecz(S)]=\zerovec$ at the fixed point solution for on-policy TD($\lambda$).
The linear $\mspbe$ can then be defined given the definition above:
\begin{equation}
\text{linear } \mspbe(\vecw) \defeq  (\sneg\Amat \vecw + \vecb)^\top \Cmat^\inv (\sneg\Amat \vecw + \vecb) = \|\sneg\Amat \vecw + \vecb \|_{C^{-1}}^2
.\label{eq:MSPBE}
\end{equation}
The only difference compared with the NEU is the weighted $\ell_2$ norm, weighted by $\Cmat^\inv$, instead of simply $\|\sneg\Amat \vecw + \vecb \|_2^2$.

\textbf{Notation Remark:} From the rest of the paper, we will not explicitly write the random variables as functions of $(S, A, S')$. For example, we will use $\mathbb{E}_{\pi} [\delta | S= s]$ to mean $\mathbb{E}_{\pi} [\delta(S, A, S') | S= s]$ and $\mathbb{E}_{\pi} [\gamma | S= s]$ to mean $\mathbb{E}_{\pi} [\gamma(S, A, S') | S= s]$.

\subsection{Emphatic Weightings as Prior Corrections}

Emphatic Temporal Difference learning, ETD($\lambda$), provides an alternative strategy for obtaining stability under off-policy sampling without computing the gradient of the linear $\mspbe$.  The key idea is to incorporate some prior corrections so that the weighting $d$ results in a positive definite matrix $\Amat$. Given such an $\Amat$, the TD($\lambda$) update---a semi-gradient algorithm---can be shown to converge. Importantly, this allows for a stable off-policy algorithm with only a single set of learned weights. Gradient-TD methods, on the other hand, use two stepsize parameters and two weight vectors to achieve stability.

Emphatic TD($\lambda$) or ETD($\lambda$) minimizes a variant of the linear $\mspbe$ defined in Equation~\ref{eq:MSPBE}, where the weighting $d$ is defined based on the \emph{followon} weighting. The followon reflects (discounted) state visitation under the target policy when doing excursions from the behavior: starting from states sampled according to $d_b$. The followon is defined as
\begin{equation}
f(s_t)\defeq d_b(s_t) + \sum_{s_{t-1}, a_{t-1}}d_b(s_{t-1})\pi(a_{t-1}|s_{t-1})P(s_t | s_{t-1}, a_{t-1})\gamma(s_{t-1},a_{t-1},s_t) + \hdots \label{eq:ETDF}
~.
\end{equation}
The emphatic weighting is $m(s_t) = d_b(s_t) \lambda + (1-\lambda) f(s_t)$. This is the state weighting ETD($\lambda$) uses in the linear $\mspbe$ in Equation~\ref{eq:MSPBE}, setting $d(s) = m(s)$.
ETD($\lambda$) uses updates
\begin{align*}
F_t \leftarrow& ~\rho_{t-1}\gamma_t F_{t-1} + 1 \quad \text{ with } M_{t}  = \lambda_t + (1 - \lambda_t)F_t\\
\vecz_t^{\rho} \leftarrow& ~\rho_t \left(\gamma_t \lambda \vecz_{t-1}^{\rho} + M_t \vecx_{t}\right)\\
\vecw_{t+1} \leftarrow& ~ \vecw_t +\alpha \delta_t \vecz_t^{\rho}
,
\end{align*}
with $F_0=1$ and $\vecz_0^{\rho}=\vec0$.

To gain some intuition for this weighting, consider the trace update with alternative-life TD(0) and ETD(0) with constant $\gamma$. For ETD(0), $M_t = F_t =  \sum_{j=0}^{t}{\gamma^j\prod_{i=1}^{j}\rho_{t-i}} $, giving $\vecz_t^{\rho} \leftarrow ~\rho_t \left(\gamma \lambda \vecz_{t-1}^{\rho} + \sum_{j=0}^{t}{\gamma^j\prod_{i=1}^{j}\rho_{t-i}} \vecx_{t}\right)$.
The trace for Alternative-life TD(0) (see Equation~\ref{eq:precup2001}) is $\vecz_t^{\rho} \leftarrow ~\rho_t \left(\gamma_t \lambda \vecz_{t-1}^{\rho} + \prod_{i=1}^{t}\rho_{i} \vecx_{t}\right)$. Both adjust the weighting on $\vecx_t$ to correct the state distributions. Alternative-Life TD more aggressively downweights states that would not have been visited under the target policy, because it only has a product, whereas ETD uses a sum over all $t$ up to that point.

Emphatic TD($\lambda$) has strong convergence guarantees in the case of linear function approximation. The ETD($\lambda$) under off-policy training has been shown to converge in expectation using the same expected update analysis used to show that TD($\lambda$) converges under on-policy training. Later, \citet{yu2015convergence} extended this result to show that ETD($\lambda$) converges with probability one.

This weighting also resolves the issues raised by Kolter's example \citep{kolter2011fixed}. Kolter's example demonstrated that for a particular choice of $\pi$ and $b$, the solution to the linear $\mspbe$ could result in arbitrarily bad value error compared with the best possible approximation in the function class. In other words, even if the true value function can be well approximated by the function class, the off-policy fixed point from the linear $\mspbe$ with weighting $d = d_b$ can result in an arbitrarily poor approximation to the values. In Section~\ref{sec_theory_ve}, we explain why the fixed points of the linear $\mspbe$ with the emphatic weighting do not suffer from this problem, expanding on \citep[Corollary 1]{hallak2016generalized} and \citep[Theorem 1]{white2017unifying}.

\subsection{Broadening the Scope of Weightings}

To determine which weightings to consider, we need to understand the role of the weighting. There are actually two possible roles. The first is to specify states of interest: determine the relative importance of a state for the accuracy of our value estimates compared to the true values. This provides the weighting in the value error: $\msve(\vecw) \defeq  \sum_{s\in\S} d(s) (\paramv[s] - v_\pi(s))^2$. The second is the choice of weighting in our objective, such as the linear $\mspbe$, which is a surrogate for the value error.

For the first question, we need to determine the relative importance of states. The choice depends on the purpose for the value estimation.
For example, if the policy is being evaluated for deployment in an episodic problem, a common choice is to put all weight on the set of start states \citep{sutton1999policy,bottou2013counterfactual} because it reflects the expected return in each episode.
On the other hand, if many value functions are learned in parallel---as in predictive representations \citep{sutton2011horde,white2015developing} or auxiliary tasks \citep{jaderberg2016reinforcement}---it is may be better to the predictions are accurate in many states. It is also possible that values from some states might be queried much more often, or that states might correspond to important catastrophic events from which it is important to have accurate predictions for accurate decision-making.

Overall, the choice of $d$ is subjective. Once we make this choice for our evaluation objective, we can ask the second question: which optimization objectives and what weightings are most effective for minimizing our evaluation objective? It is not obvious that the minimum of the linear $\mspbe$ with weighting $d_b$ provides the best solution to the $\msve$ with weighting $d_b$, for example. In fact, we know that the linear $\mspbe$ with $d = d_b$ suffers from a counterexample \citep{kolter2011fixed}, whereas using the emphatic weighting in the linear $\mspbe$ provides an upper bound on the $\msve$ under weighting $d_b$. We discuss the potential utility of using a different weighting for the objective than the desired weighting in the $\msve$ in Section~\ref{sec_theory_ve}.

\section{Broadening the Scope of Objectives}

In this section, we discuss how to generalize the linear $\mspbe$, to obtain the generalized $\mspbe$. This objective allows for nonlinear value estimation and unifies the $\msbe$ and linear $\mspbe$ under one objective.

\subsection{An Overview of Existing Objectives}

Let us start by discussing the standard evaluation objective used for policy evaluation: the mean squared value error ($\msve$)
\begin{equation}\label{eq:msve}
\msve(\vecw) \defeq  \sum_{s\in\S} d(s) (\paramv[s] - v_\pi(s))^2
.
\end{equation}
The approximation $\paramv[s]$ is penalized more heavily for inaccurate value estimates in highly weighted states $s$. This objective cannot be directly optimized because it requires access to $v_\pi(s)$.
One way to indirectly optimize the $\msve$ is to use the mean squared return error ($\msre$):
\begin{equation}\label{eq:msre}
\msre(\vecw) \defeq  \sum_{s\in\S} d(s) \CEpi{ \left(\paramv[s] - G_t\right)^2}{S_t=s}
.
\end{equation}
The minima of the $\msre(\vecw)$ and $\msve(\vecw)$ are the same because their gradients are equal
\begin{align*}
\nabla \msre(\vecw)
&=  \sum_{s\in\S} d(s) \CEpi{ \left(\paramv[s] - G_t\right) \nabla \paramv[s]}{S_t=s}\\
&=  2\sum_{s\in\S} d(s)  \left(\paramv[s] - \CEpi{G_t}{S_t=s} \right) \nabla \paramv[s]\\
&=  2\sum_{s\in\S} d(s)  \left(\paramv[s] - v_\pi(s) \right) \nabla \paramv[s] = \nabla \msve(\vecw)
.
\end{align*}

In practice, the $\msre$ is rarely used because it requires obtaining samples of entire returns (only updating at the end of episodes). Instead, bootstrapping is used and so forms of the Bellman error are used, as in the $\msbe$ and $\mspbe$.
The $\msbe$ reflects the goal of approximating the fixed-point formula given by the Bellman operator $\Bo$, defined as
\begin{equation}
\Bo\paramv[\cdot] (s) \defeq \CEpi{R + \gamma \paramv[S']}{S=s} \text{ for all $s$}
.
\end{equation}
When equality is not possible, the difference is minimized as in the $\msbe$
\begin{align}
\msbe(\vecw) &\defeq  \sum_{s\in\S} d(s) \left(\Bo\paramv[\cdot] (s) - \paramv[s]\right)^2\label{eq:msbe}
= \sum_{s\in\S} d(s) \CEpi{\delta(\vecw)}{S=s}^2
\end{align}
where we write $\delta(\vecw)$ to be explicit that this is the TD error for the parameters $\vecw$.

\begin{figure}[t]
\vspace{-0.2cm}
  \centering
  \includegraphics[width=0.6\linewidth]{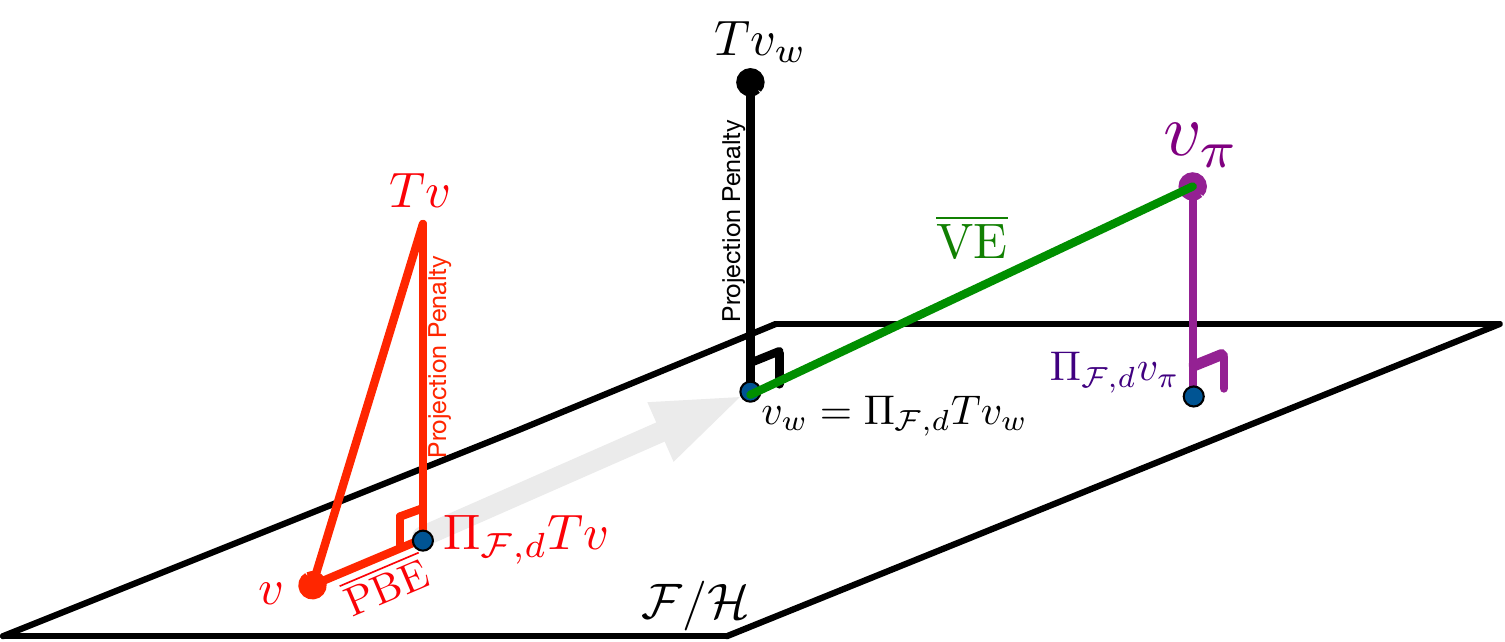}
  \caption{\label{fig:pbe-sol}
    The visualization above characterizes the true $v_\pi$, $\mspbe$ solution and how projections operate on successive approximations. Assume the estimate of $v_\pi$ starts from $\vecv$ in red. The Bellman operator pushes the value estimate out of the space of representable functions represented by the plane (Note this corresponds to $\mathcal{F} = \mathcal{H}$ introduced in Section~\ref{sec:identifiable-be}). The projection brings the approximation back down to the nearest representable function on the plane. This process is repeated over and over until the value estimates converge to the blue dot at the base of the black line. Subsequent updates push the approximation to $v_\pi$ out of the space of representable functions and the projection back onto the plane. The true value in this case is outside of $\mathcal{F}$, with the $\msve$ being the distance between the $\vecv$ at $\mspbe=0$ and $v_\pi$. Note the projection of $v_\pi$ onto $\mathcal{F}$ need not be equal to $\mspbe$ solution.
  }
\end{figure}



There has been much discussion, formal and informal, about using the $\msbe$ versus the $\mspbe$. The $\msbe$ can be decomposed into the $\mspbe$ and a projection penalty term \citep{scherrer2010should}. To understand why, recall the definition of a projection operator. For a vector space $\Vset$, the projection of a vector $v$ onto $\Vset$ is the closest point under a given (weighted) norm $\| \cdot \|_d$: $\min_{u \in \Vset } \| u - v \|_d$. This definition also applies to function spaces. Let $\Pi_{\Vset , d}$ be the weighted projection on the space of value functions, defined as $\Pi_{\Vset , d} u \defeq \argmin_{u \in \Vset } \| u - v \|_d$.
For a given vector $v \in \Re^{|\States|}$, composed of value estimates for each state, we get
\begin{equation}
\| v - \Bo v \|^2_d = \underbrace{\| v - \Pi_{\Vset , d}  \Bo v \|^2_d}_{\mspbe} + \underbrace{\| \Bo v - \Pi_{\Vset , d}  \Bo v \|^2_d}_{\text{Projection Penalty}}
\end{equation}
This penalty causes the $\msbe$ to prefer value estimates for which the projection does not have a large impact near the solution. The $\mspbe$ can find a fixed point where applying the Bellman operator $\Bo v$ moves far outside the space of representable functions, as long as the projection back into the space stays at $v$. The projection penalty is sensible, and in fact prevents some of the counterexamples on the solution quality for the $\mspbe$ discussed in Section~\ref{sec_theory_ve}. We visualize the projection penalty and a potential solution under the $\mspbe$ in Figure~\ref{fig:pbe-sol}, and contrast to a potential solution under the $\msbe$ in Figure~\ref{fig:be-sol}.

\begin{figure}[t]
\vspace{-0.2cm}
  \centering
  \includegraphics[width=0.5\linewidth]{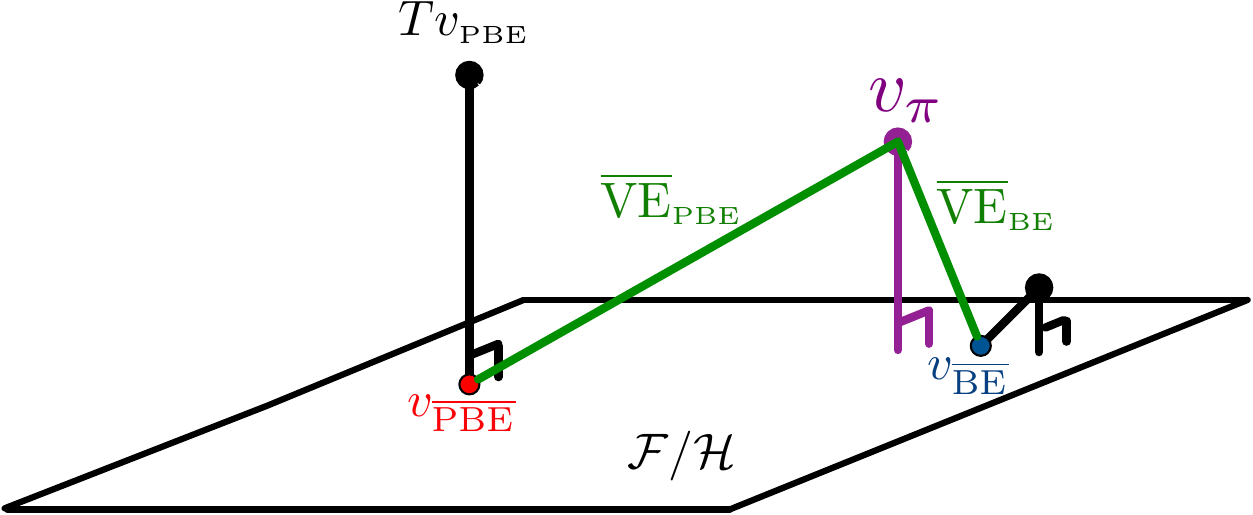}
  \caption{\label{fig:be-sol}
    A comparison of the $\msbe$ and $\mspbe$ solutions when the true value function is not representable. As before, we visualize how the approximation that minimizes the $\mspbe$ at convergence can be far from $v_\pi$ with a large projection penalty. The approximate value function that minimizes the $\msbe$ on the other hand is closer to $v_\pi$ and typically has a smaller projection penalty (note the Bellman operator would indeed push $\vecv_{{\text{\tiny BE}}}$ outside $\mathcal{F}$).
  }
\end{figure}

Despite the potential utility of the $\msbe$, it has not been widely used due to difficulties in optimizing this objective without a model. The $\msbe$ is difficult to optimize because of the well-known double sampling problem for the gradient. To see why, consider the gradient
\begin{align*}
\nabla_\vecw \msbe(\vecw)
&=  \sum_{s\in\S} d(s) \nabla_\vecw \CEpi{\delta(\vecw)}{S=s}^2\\
&=  2\sum_{s\in\S} d(s)  \CEpi{\delta(\vecw)}{S=s} \CEpi{\nabla_\vecw \delta(\vecw)}{S=s}\\
&=  2\sum_{s\in\S} d(s)  \CEpi{\delta(\vecw)}{S=s} \CEpi{\gamma \nabla_\vecw \paramv[S'] - \nabla_\vecw \paramv[s] }{S=s}
\end{align*}
To estimate this gradient for a given $S = s$, we need two independent samples of the next state and reward. We use the first to get a sample $\delta(\vecw)$ and the second to get a sample of $\gamma \nabla_\vecw \paramv[S'] - \nabla_\vecw \paramv[s]$. The product of these two samples gives an unbiased sample of the product of the expectations. If we instead only used one sample, we would erroneously obtain a sample of $\CEpi{\delta(\vecw) (\gamma \nabla_\vecw \paramv[S'] - \nabla_\vecw \paramv[s]) }{S=s}$.

One promising attempt to approximate the $\msbe$ used a non-parametric approach \citep{feng2019kernel}. The objective, called the $\mskbe$, takes pairs of samples from a buffer to overcome the double sampling problem. Unfortunately, this cannot overcome the issue of non-identifiability in the $\msbe$. There is a simple example where the same data is generated by two different MDPs, with different optima for the corresponding $\msbe$ \cite[Chapter 11.6]{sutton2018reinforcement}. The agent cannot hope to use the data to identify which of the two parameters is the optimal solution; that work used the term that the $\msbe$ is not \emph{learnable}.

The linear $\mspbe$, on the other hand, is practical to optimize under linear function approximation, as discussed above: the whole family of (gradient) TD algorithms is designed to optimize the linear $\mspbe$. Unfortunately, the $\mspbe$ is hard to optimize for the general nonlinear setting, because the projection is hard to compute. Prior attempts to extend GTD to the nonlinear $\mspbe$ \citep{maei2009convergent} resulted in an algorithm that requires computing Hessian-vector products.
In the next section, we discuss how to overcome these issues for the $\msbe$ and $\mspbe$, with a unified objective, that is a generalization of the $\mspbe$.

Finally, for completeness, we conclude with a description of the Mean-Squared TD error ($\mstde$), even though it is rarely used. The $\mstde$ was introduced to characterize the TD solution as a semi-gradient method. For the objective
\begin{equation}\label{eq:mstde}
\mstde(\vecw) \defeq  \sum_{s\in\S} d(s) \CEpi{ \left(R + \gamma \paramv[S'] - \paramv[s]\right)^2}{S=s}
\end{equation}
the gradient includes the gradient of $\paramv[S'] $. TD omits this term, and so is called a semi-gradient method. It is not typical to do gradient descent on the $\mstde$, due to commonly held views of poor quality and a counterexample for the residual gradient algorithm which uses the $\mstde$ \citep[]{sutton2018reinforcement}. We  highlight the significant bias when using the $\mstde$, in Appendix~\ref{sec_mstdebad}, providing further evidence that it is likely not a useful direction.

\subsection{An Identifiable $\msbe$}\label{sec:identifiable-be}

Before discussing the generalized $\mspbe$, we start by showing a conjugate form for the $\msbe$. This reformulation uses the strategy introduced by \citet{dai2017learning}, which more generally introduces this conjugate form for several objectives that use conditional expectations. They show how to use it for the $\msbe$ as an example, but defined it slightly differently because they condition on states and actions. For this reason, and because we will build further, we provide the explicit steps to derive the conjugate form for the $\msbe$.

Let $\Vset $ be the space of parameterized value functions and $\Fsetall$ the space of all functions.
The reformulation uses the fact that the biconjugate of the square function is $y^2 = \max_{h \in \mathbb{R}} 2yh - h^2$ and the fact that the maximum can be brought outside the sum (interchangeability), as long as a different scalar $h$ can be chosen for each state $s$, as it can be for $\Fsetall$ the space of all functions.
\begin{align*}
\msbe(\vecw)  &= \sum_{s\in\S} d(s) \CEpi{\delta(\vecw)}{S=s}^2 \\
&= \sum_{s\in\S} d(s) \max_{h \in \mathbb{R}} \left(2\CEpi{\delta(\vecw)}{S=s}h - h^2 \right) && \triangleright \text{ using the biconjugate function}\\
&=  \max_{h \in \Fsetall} \ \sum_{s\in\S} d(s) \left(2\CEpi{\delta(\vecw)}{S=s}h(s) - h(s)^2 \right) && \triangleright \text{ using interchangeability}
.
\end{align*}
The optimal $h^*(s) = \CEpi{\delta(\vecw)}{S=s}$, because
\begin{align*}
\argmax_{h \in \Fsetall} & \sum_{s\in\S} d(s) \left(2\CEpi{\delta(\vecw)}{S=s}h(s) - h(s)^2 \right)\\
&= \argmax_{h \in \Fsetall}  \sum_{s\in\S} d(s) \left( 2\CEpi{\delta(\vecw)}{S=s}h(s) - h(s)^2 - \CEpi{\delta(\vecw)}{S=s}^2 \right)\\
&= \argmax_{h \in \Fsetall}  -\sum_{s\in\S} d(s) \left(\CEpi{\delta(\vecw)}{S=s} - h(s)\right)^2\\
&= \argmin_{h \in \Fsetall}  \sum_{s\in\S} d(s) \left(\CEpi{\delta(\vecw)}{S=s} - h(s)\right)^2
.
\end{align*}
The function $h^*(s) = \CEpi{\delta(\vecw)}{S=s}$ provides the minimal error of zero. This optimal solution also makes it clear why the above is simply a rewriting of the $\msbe$ because
\begin{equation*}
2\CEpi{\delta(\vecw)}{S\!=\!s}h^*(s) - h^*(s)^2 = 2\CEpi{\delta(\vecw)}{S\!=\!s}^2 - \CEpi{\delta(\vecw)}{S\!=\!s}^2 = \CEpi{\delta(\vecw)}{S\!=\!s}^2
.
\end{equation*}

More generally, for the continuous state case, interchangeability also holds, as long as the function $h(s) = \CEpi{\delta(\vecw)}{S=s}$ satisfies $h \in \Fsetall$. We can more generically expressed the $\msbe$ using expectations over states: $\msbe(\vecw) = \mathbb{E}[\CEpi{\delta(\vecw)}{S}^2]$, where the outer expectation is over $S$ with distribution $d$. For $g(h, s) = \CEpi{\delta(\vecw)}{S=s}h - h^2$, the $\msbe$ is
 \begin{align}
 \mathbb{E}\left[\max_{h \in \mathbb{R}} g(h, S) \right] &= \int_{\S} d(s) \max_{h \in \mathbb{R}} g(h, s) ds
= \max_{h \in \Fsetall} \int_{\S} d(s) g(h(s), s) ds \label{eq_be_cont}
.
 \end{align}
Because $h(s) = \CEpi{\delta}{S=s}$ satisfies $h \in \Fsetall$, we know that a maximizer exists, as $h^* = h \in \Fsetall$. Then we can show that $\mathbb{E}[\max_{h \in \RR} g(h, S)] = \mathbb{E}[g(h^*(S), S)] = \max_{h \in \Fsetall} \mathbb{E}[g(h(S), S)]$.\footnote{This argument is similar to \citep[Lemma 1]{dai2017learning},
but we do not need to assume $g$ is upper semi-continuous and concave. They use this condition to ensure the existence of a maximum for $g(\cdot, s)$. We know the form of our $g$ and can directly determine the existence of a maximizer (expected TD error).}

As highlighted in \citep[Chapter 11.6]{sutton2018reinforcement}, the $\msbe$ is not identifiable. In that example, however, the inputs given to the value function learner are partially observable. In terms of the above formulation, this would mean the agent can only observe a part of the state for learning $\vecw$ but the whole state to learn $h$. Naturally, however, the input-space for $h$ should be similarly restricted to only observable information.
This leads us to a new set for $h$, which includes all functions on the same inputs $\phi(s)$ as given to $v$, rather than on state:
\begin{align*}
\Hsetall  \defeq \{h = f \circ \phi \ | \ \text{ where $f$ is any function on the space produced by $\phi$}\}
.
\end{align*}
The resulting $h$ is restricted to functions of the form $h(s) = f(\phi(s))$. We call the resulting $\msbe$ an \emph{Identifiable} $\msbe$, written as:
\begin{align*}
\ibe(\vecw)
&\defeq \max_{h \in \Hsetall } \ \mathbb{E}\left[2\CEpi{\delta(\vecw)}{S}h(S) - h(S)^2 \right] .
\end{align*}
Notice that $\Hsetall  \subseteq \Fsetall$, and so the solution to the $\ibe$ may be different from the solution to the $\msbe$. In particular, we know the $\ibe(\vecw) \le  \msbe(\vecw)$, because the inner maximization is more constrained. In fact, in many cases restricting $h$ can be seen as a projection on the errors in the objective, as we discuss next, making the $\ibe$ an instance of the generalized $\mspbe$.

\subsection{From the Identifiable Bellman Error back to a Projected Bellman Error}

The previous section discussed a conjugate form for the $\msbe$, which led to an identifiable $\msbe$. Even this $\ibe$, however, can be difficult to optimize, as we will not be able to perfectly represent any $h$ in $\Hsetall$. In this section, we discuss further approximations, with $h \in \Hset \subseteq \Hsetall$, leading to a new set of objectives based on projecting the Bellman error.

To practically use the minimax formulation for the $\msbe$, we need to approximate $h$ as an auxiliary estimator. This means $h$ must also be a parameterized function, and we will instead only obtain an approximation to the $\ibe$. Let $\Hset $ be a convex space of parameterized functions for this auxiliary function $h$. As we show below, this $\Hset$ defines the projection in the generalized $\mspbe$.

In order to define the $\mspbe$, we first define a projection operator $\Pi_{\Hset, d}$ which projects any vector $u \in \Re^{|\States|}$ onto convex subspace $\Hset \subseteq \Re^{|\States|}$ under state weighting $d$:
\begin{equation}
\Pi_{\Hset , d} u \defeq \argmin_{h \in \Hset } \| u - h \|_d.
\end{equation}
We define the generalized $\mspbe$ as
\begin{equation}
    \mspbe(\vecw) \defeq \| \Pi_{\Hset , d} (\Bo \paramv - \paramv) \|^2_d
\end{equation}
where each choice of $\Hset $ results in different projection operators. This view provides some intuition about the role of approximating $h$. Depending on how errors are projected, the value function approximation will focus more or less on the Bellman errors in particular states. If the Bellman error is high in a state, but those errors are projected to zero, then no further approximation resources will be used for that state. Under no projection---the set for $h$ being the set of all functions---no errors are projected and the values are learned to minimize the Bellman error. If $\Hset  = \Vset $, the same space is used to represent $h$ and $v$, then we obtain the projection originally used for the $\mspbe$.

We now show the connection between this $\mspbe$ and the $\msbe$.
In the finite state setting, we have a vector $u \in \Re^{|\States|}$ composed of entries $\CEpi{\delta(\vecw)}{S=s}$: $u = \Bo \paramv - \paramv$.
Given a Hilbert space equipped with inner-product $\langle\cdot, \cdot\rangle_d$, the projection onto this space using a diagonal weighting matrix $D \defeq \diag(d)$ is guaranteed to be an orthogonal projection. This property means $u = \Pi_{\Hset , d} u + \tilde{u} = h + \tilde{u}$, where $h = \Pi_{\Hset , d} u$ and $\tilde{u}$ is the component in $u$ that is orthogonal in the weighted space: $h^\top D \tilde{u} = 0$ for $D \defeq \diag(d)$.
Then we can write the conjugate form for the $\msbe$, now with restricted $\Hset \subset \Hsetall$
%
 \begin{align*}
 \max_{h \in \Hset }& \sum_{s\in\S} d(s) \left(  2\CEpi{\delta(\vecw)}{S=s}h(s) - h(s)^2 \right)\\
 &= \max_{h \in \Hset } \sum_{s\in\S} d(s) \left(  2u(s)h(s) - h(s)^2 \right) \hspace{2.0cm}\triangleright \text{ rewriting $u(s) = \CEpi{\delta(\vecw)}{S=s}$}\\
 &= \sum_{s\in\S} d(s) \left(2u(s) h(s) - h(s)^2 \right) \hspace{3.0cm}\triangleright \text{ where $h = \Pi_{\Hset , d} u$}\\
 &= \sum_{s\in\S} d(s) \left(2[h(s) + \tilde{u}(s)] h(s) - h(s)^2 \right) \hspace{1.3cm}\triangleright \text{ because $u(s) = h(s) + \tilde{u}(s)$}\\
 &= \sum_{s\in\S} d(s) \left(2 h(s)^2 - h(s)^2 \right) + 2 \sum_{s\in\S} d(s) \tilde{u}(s) h(s) \\
 &= \sum_{s\in\S} d(s) h(s)^2  + 2 \sum_{s\in\S} d(s) \tilde{u}(s) h(s) \\
 &= \sum_{s\in\S} d(s) h(s)^2 \hspace{2.0cm} \triangleright \text{ where $\sum_{s\in\S} d(s) \tilde{u}(s) h(s) = 0$ because }\\
 &=  \| \Pi_{\Hset , d} (\Bo \paramv - \paramv) \|^2_d  \hspace{2.0cm} \text{ $h$ is orthogonal to $\tilde{u}$, under weighting $d$ }\\
 &= \mspbe(\vecw)
 \end{align*}

The key assumption above is that $\Hset$ is a Hilert space in order to ensure that the projection operator $\Pi_{\Hset , d}$ is an orthogonal projection \citep{dudek1994nonlinear}. This assumption is easily satisfied by linear functions with a fixed basis $\phi(s)$, with bounded weights. For any two functions with weights $\vecw_1$ and $\vecw_2$, the function defined by weights $(1-c)\vecw_1 + c \vecw_2$ is also in the set. More generally, typical convex nonlinear function approximation sets used in machine learning are reproducing kernel Hilbert spaces. Many classes of neural networks have been shown to be expressible as RKHSs (see \citet{bietti2019group} for a nice overview), including neural networks with ReLU activations as are commonly used in RL. Therefore, this is not an overly restrictive assumption.

\subsection{The Connection to Previous $\mspbe$ Objectives}

In this section we show how the generalized $\mspbe$ lets us express the linear $\mspbe$, and even the nonlinear $\mspbe$, by selecting different sets $\Hset$.
First let us consider the linear $\mspbe$. The easiest way to see this is to use the saddlepoint formulations developed for the linear $\mspbe$ \citep{mahadevan2014proximal,liu2016proximal,touati2018convergent}. The goal there was to re-express the existing linear $\mspbe$ using a saddlepoint form, rather than to re-express the $\msbe$ or find connections between them. The linear $\mspbe$ = $\| \vecb - \Amat \vecw \|^2_{\Cmat^\inv}$ can be rewritten using the conjugate for the two norm. The conjugate for the two-norm is $\tfrac{1}{2} \| \vecy \|_{\Cmat^\inv} = \max_{\vech} \vecy^\top \vech - \tfrac{1}{2} \| \vech \|_{\Cmat}^2$, with optimal $\vech = \Cmat^\inv \vecy$. Correspondingly, we get
\begin{align*}
\tfrac{1}{2} \| \vecb - \Amat \vecw \|^2_{\Cmat^\inv}
&=  \max_{h \in \RR^{\xdim}}  (\vecb - \Amat \vecw)^\top \vech - \tfrac{1}{2} \| \vech \|_{\Cmat}^2
\end{align*}
where the solution for $\vech = \Cmat^\inv (\vecb - \Amat \vecw)$. This solution makes the first term equal to $\| \vecb - \Amat \vecw\|^2_{\Cmat^\inv}$ and the second term equal to $-\tfrac{1}{2} \| \vecb - \Amat \vecw \|^2_{\Cmat^\inv}$; adding them gives $\tfrac{1}{2} \| \vecb - \Amat \vecw \|^2_{\Cmat^\inv}$.

\begin{figure*}
\vspace{-0.3cm}
    \centering
    \includegraphics[width=0.7\linewidth]{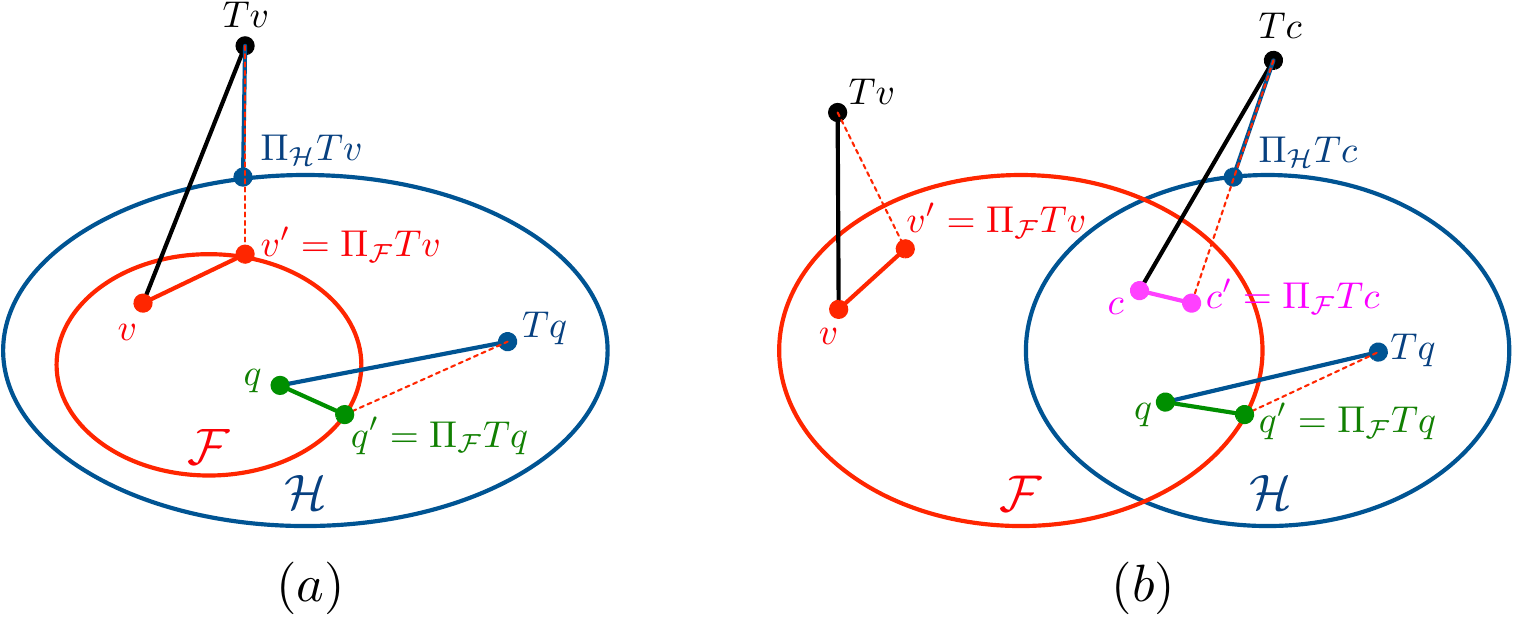}
    \caption{\label{fig:pbe-ven}
        A visual interpretation of how the Bellman operator can push the value estimates outside the space of representable functions and the role of the projection operator. The set $\Vset$ corresponds to the (parameterized) space of value functions and $\Hset$ is the set of functions that approximate (project) the Bellman error $\Bo v - v$. Potential settings include $\Vset = \Hset$ (visualized in Figure~\ref{fig:pbe-sol}), $\Vset \subset \Hset$ visualized in (a) and $\Vset \neq \Hset$ visualized in (b).
        In (a), we highlight two cases: $\Bo v$ is not representable by any function in $\mathcal{F}$ or $\mathcal{H}$, or $\Bo v$ is representable by functions in $\mathcal{H}$ but not $\mathcal{F}$. In (b) we see examples of projections when $\mathcal{F}$ intersects $\mathcal{H}$.
    }
\end{figure*}

We can obtain the same formulation under the $\mspbe$, by restricting $\Vset$ and $\Hset$ to be the same set of linear functions. Let $\mathcal{L} = \{f: \States \rightarrow \RR : f(s) = \vecx(s)^\top \vecw, \vecw \in \RR^\xdim \}$. For $\Vset = \Hset = \mathcal{L}$, we have that
$h^* = \argmin_{h \in \mathcal{L}}  \sum_{s\in\S} d(s) \left(\CEpi{\delta(\vecw)}{S=s} - h(s)\right)^2$ satisfies
$h^*(s) = \vecx(s)^\top \vech^*$.
This $h^*$ is the linear regression solution for targets $\delta(\vecw)$, so $\vech^* = \E{\vecx \vecx^\top}^{\inv} \E{\vecx \delta(\vecw)}$ which equals $\Cmat^\inv (\vecb - \Amat \vecw)$.
We can further verify that the resulting $\mspbe$ matches the linear $\mspbe$ (see Appendix~\ref{app_saddlepoint}).
This result is alluded to in the connection between the NEU and the $\mskbe$, in \citep[Corollary 3.5]{feng2019kernel}, but not explicitly shown.

This connection also exists with the nonlinear $\mspbe$, but with a surprising choice for the parameterization of $h$: using the gradient of the value estimate as the features. The nonlinear $\mspbe$ is defined as \citep{maei2009convergent}
\begin{align*}
\text{nonlinear } \mspbe(\vecw) &= \E{\delta(\vecw) \nabla_\vecw \paramv[s]}^\top \E{\nabla_\vecw \paramv[s] \nabla_\vecw \paramv[s]^\top}^{\inv} \E{\delta(\vecw) \nabla_\vecw \paramv[s]}
.
\end{align*}
\newcommand{\hnonlinear}{h^*_{\text{nl}}}
\newcommand{\vechnonlinear}{\vech^*_{\text{nl}}}
This corresponds to the linear $\mspbe$ when $\Vset = \mathcal{L}$, because $\nabla_\vecw \paramv[s] = \vecx(s)$. Define set $\mathcal{G}_\vecw = \{f: \States \rightarrow \RR : f(s) = \vecy(s)^\top \vech, \vech \in \RR^\xdim \text{ and } \vecy(s) = \nabla_\vecw \paramv[s]\}$. Notice that this function set for $h$ changes as $\vecw$ changes. Then we get that
\begin{equation*}
\hnonlinear = \argmin_{h \in \mathcal{G}_\vecw}  \sum_{s\in\S} d(s) \left(\CEpi{\delta(\vecw)}{S=s} - h(s)\right)^2
\end{equation*}
satisfies $\hnonlinear(s) = \nabla_\vecw \paramv[s]^\top \vechnonlinear$ where $\vechnonlinear = \E{\nabla_\vecw \paramv[s] \nabla_\vecw \paramv[s]^\top}^{\inv} \E{\delta(\vecw) \nabla_\vecw \paramv[s]}$.

Plugging this optimal $h$ back into the formula,
we get that
\begin{align*}
& \max_{h \in \mathcal{G}_\vecw}  \sum_{s\in\S} d(s) \left(2\CEpi{\delta(\vecw)}{S=s}h(s) - h(s)^2\right)
= \sum_{s\in\S} d(s) \left(2\CEpi{\delta(\vecw)}{S=s}\hnonlinear(s) - \hnonlinear(s)^2\right) \\
&= \Big(\sum_{s\in\S} d(s) 2\CEpi{\delta(\vecw)}{S=s}\nabla_\vecw \paramv[s]^\top\Big) \vechnonlinear - \sum_{s\in\S} d(s) (\vechnonlinear)^\top\nabla_\vecw \paramv[s]\nabla_\vecw \paramv[s]^\top \vechnonlinear \\
&= 2 \mathbb{E}[\delta(\vecw)\nabla_\vecw \paramv[s]]^\top \vechnonlinear - (\vechnonlinear)^\top \mathbb{E}[\nabla_\vecw \paramv[s]\nabla_\vecw \paramv[s]^\top] \vechnonlinear \\
&= 2\text{nonlinear } \mspbe(\vecw)  - \text{nonlinear } \mspbe(\vecw) \\
&= \text{nonlinear } \mspbe(\vecw)
\end{align*}
This nonlinear $\mspbe$ is not an instance of the generalized $\mspbe$, as we have currently defined it, because the $\Hset$ changes with $\vecw$. It is possible that such a generalization is worthwhile, as using the gradient of the values as features is intuitively useful. Further, interchangeability should still hold, as the exchange of the maximum was done for a fixed $\vecw$. Therefore, it is appropriate to explore an $\Hset$ that changes with $\vecw$, and in our experiments we test $\Hset = \mathcal{G}_\vecw$.

In summary, in this section we introduced the generalized $\mspbe$ and highlighted connections to the linear $\mspbe$ and $\msbe$. The generalized $\mspbe$ provides a clear path to develop value estimation under nonlinear function approximation, providing a strict generalization of the linear $\mspbe$. Two secondary benefits are that the generalized $\mspbe$ provides a clear connection between the $\msbe$ and $\mspbe$, based on a difference in the choice of projection ($\Hset$), and resolves the identifiability issue in the $\msbe$.

\section{Understanding the Impact of Choices in the Generalized $\mspbe$}

The two key choices in the Generalized $\mspbe$ is the state-weighting and the (projection) set $\Hset$. There are at least three clear criteria for selecting $\Hset$ and the state-weighting: (1) the quality of the solution, (2) the feasibility of implementation and (3) the estimation error during learning and the impact on learning the primary weights.
In this section, we provide some conceptual and empirical insight into how to choose $\Hset$, and empirically show that the choice of weighting can significantly change the quality of the solution. After first showing that these choices clearly matter, and some intuition for why, we then provide theory characterizing the quality of the solution in the following section.

\subsection{The Projection Set and the Quality of the Solution}

The first criteria parallels the long-standing question about the quality of the solution under the linear $\mspbe$ versus the $\msbe$. The examples developed for that comparison provide insights on $\Hset$. In this section, we revisit these examples, now in context of the generalized $\mspbe$.

Objectives based on Bellman errors perform \emph{backwards bootstrapping}, where the value estimates in a state $s$ are adjusted both toward the value of the next state and the value of the previous state.
In the case of the $\msbe$, backwards bootstrapping can become an issue when two or more states are heavily aliased and these aliased states lead to successor states with highly different values.
Because the aliased states look no different to the function approximator, they must be assigned the same estimated value.
For each of these aliased states, backwards bootstrapping forces the function approximator to balance between accurately predicting the successor values for all aliased states, as well as adjusting the successor values to be similar to those of the aliased states.

The $\mspbe$, on the other hand, projects the error for the aliased states,
 ignoring the portion of the Bellman error that forces the function approximator to balance the similarity between the aliased state value and the successor state value.
This allows the function approximator the freedom to accurately estimate the values of the successor states without trading off error in states which it cannot distinguish.

To make this concrete, consider the following 4-state MDP from \citet{sutton2009fast}.
State $A_1$ and $A_2$ are aliased under the features for $\Vset$. For the linear $\mspbe$, $\Hset = \Vset$, and so the states are also aliased when approximating $h$. For the $\msbe$, they are not aliased for $h$.
$A_1$ transitions to $B$ and terminates with reward 1. $A_2$ transition to $C$ and terminates with reward 0. The linear $\mspbe$ results in the correct values for $B$ and $C$---1 and 0 respectively---because it does not suffer from backwards bootstrapping. The $\msbe$, on the other hand, assigns them values $\tfrac{3}{4}$ and $\tfrac{1}{4}$, to reduce Bellman errors at $A_1$ and $A_2$. A generalized $\mspbe$ with other $\Hset \neq \Vset$ would suffer the same issue as the $\msbe$ in this example, unless the projection $\Pi_\Hset$ mapped errors in the aliased states to zero.

\begin{figure*}
    \centering
    \includegraphics[width=\linewidth]{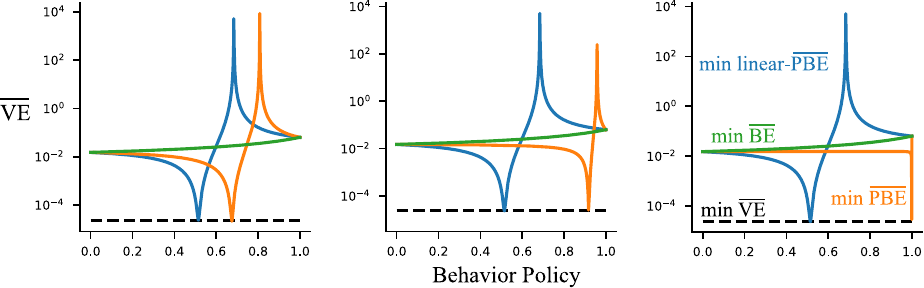}
    \caption{\label{fig:kolterexample}
      The visualization above shows how the $\mspbe$ solution can result in arbitrarily bad value error under some behaviours. The blue line above is the same as the visualization used in prior work to demonstrate issues with minimizing $\mspbe$ (see \citet{kolter2011fixed} for a description of the counterexample). The  vertical axis measures $\msve$ and the horizontal axis different behavior policies. This figure differs from \citet{kolter2011fixed}; we show that the $\msbe$ solution exhibits low error and highlight the impact of changing $\Hset$.
      The size of the set $\Hset$ increases from the left subplot to the right.
      More behavior policies result in low generalized $\mspbe$ as the set $\Hset$ increases.
    }
\end{figure*}

On the other hand, the linear $\mspbe$ can find solutions where the Bellman error is very high, even though the projected Bellman error is zero.
Consider the plane of value functions that can be represented with a linear function approximator.
The Bellman operator can take the values far off of this surface, only to be projected back to this surface through the projection operator.
At the fixed-point, this projection brings the value estimate back to the original values and the distance that the value estimate moved on the plane is zero, thus the $\mspbe$ is zero. The $\mspbe$ can be zero even when the $\msbe$ is large. \citet{kolter2011fixed} provides an example where the solution under the $\mspbe$ can be made arbitrarily far from the true value function.
We expand on this example in Figure~\ref{fig:kolterexample}, and show that the solution under the linear $\mspbe$ can be arbitrarily poor, even though the features allow for an $\epsilon$ accurate value estimate and the solution under the $\msbe$ is very good.

If we use an emphasis weighting for the states, rather than behavior visitation, then the solution under the $\mspbe$ becomes reasonable. Further, even just a small change to $\Hset$ so that $\Hset \neq \Vset$ resolves this counterexample. Ideally, we would use an $\Hset$ similar to $\Vset$, to avoid backwards bootstrapping. Conceptually, a potentially reasonable choice for $\Hset$ is therefore either (1) $\Hset = \Vset$ with some consideration on adjusting the state weighting and (2) an $\Hset$ that is only slightly bigger than $\Vset$, potentially with the inclusion of an additional feature.

\paragraph{A Simple Experiment for Solution Quality Under Different Weightings and $\Hset$}
We empirically investigate the quality of the solution under the $\mspbe$ and $\msbe$ with three different weightings: $d_b$, $d_\pi$ and $\mweight$. The solution quality is measured by the $\msve$ under $d_b$ and $d_\pi$.
We compute the fixed-point of each objective on a 19-state random walk with randomly chosen target and behavior policies.
To isolate the impact of representation on the fixed-points, we investigate several forms of state representation where $v_\pi$ is outside the representable function class.
We include the \emph{Dependent} features from \citet{sutton2009fast}, randomly initialized sparse ReLu networks, tile-coded features, and state aggregation.

The random-walk has 19 states with the left-most and right-most state being terminal.
The reward function is zero everywhere except on transitioning into the right-most terminal state where the agent receives +1 reward, and on the left-most terminal state where the agent receives -1 reward.
The discount factor is set to $\gamma=0.99$.

\begin{figure*}
  \centering
  \includegraphics[width=0.75\linewidth]{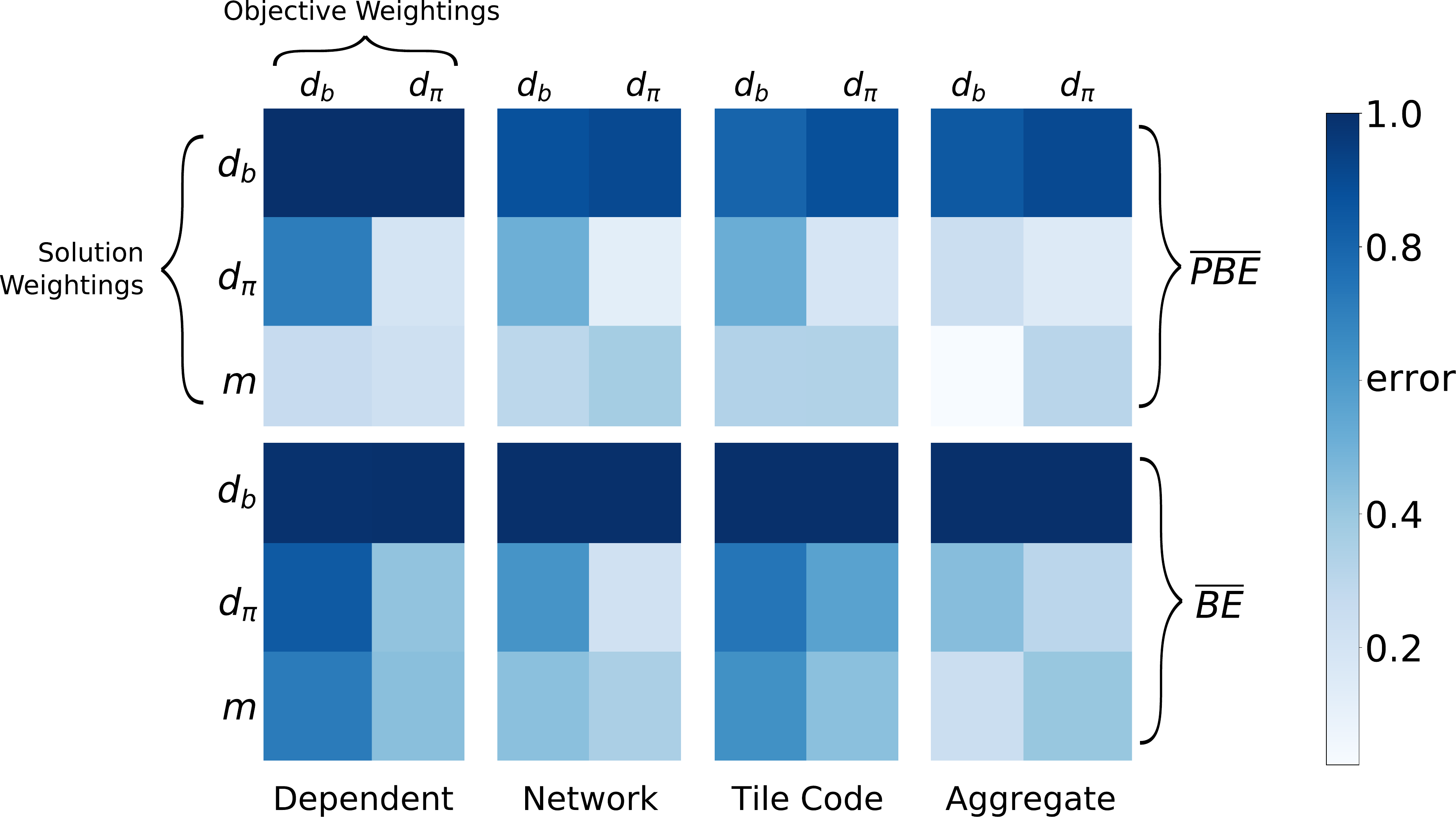}
  \caption{\label{fig:rw-fixed-points}
    Investigating the $\msve$ of the fixed-points of $\mspbe$ and $\msbe$ under $d_b$, $d_\pi$, and $\mweight$ on a 19-state random walk. All errors are computed closed form given access to the reward and transition dynamics. The fixed-point of the $\mspbe$ with emphatic weighting consistently has the lowest error across several different state representations (light color); while the fixed-point of the $\mspbe$ under $d_b$ has the highest error (dark blue).
    Results are averaged over one million randomly generated policies and state representations.
  }
\end{figure*}

We run each experimental setting one million times with a different randomly initialized neural network, random offset between tilings in the tile-coder, and randomly sampled target and behavior policy.
The policies are chosen uniformly randomly on the standard simplex.
The neural network is initialized with a Xavier initialization \citep{glorot2010understanding}, using 76 nodes in the first hidden layer and 9 nodes in the final feature layer.
Then 25\% of the neural network weights are randomly set to zero to encourage sparsity between connections and to increase variance between different randomly generated representations.
The tile-coder uses 4 tilings each offset randomly and each containing 4 tiles.
The state aggregator aggressively groups the left-most states into one bin and the right-most states into another, creating only two features.

Figure~\ref{fig:rw-fixed-points} shows the normalized log-error of the fixed-points of the $\mspbe$ and $\msbe$ under each weighting. A normalized error between $[0,1]$, for each representation, is obtained by (1) computing the best value function representable by those features, $\min_{v \in \Vset} \msve(v)$ under $d_b$ or $d_\pi$ and (2) subtracting this minimal $\msve$, and normalizing by the maximum $\msve$ for each column (across objectives and weightings for a fixed representation).
The fixed-points are computed using their least-squares closed form solutions given knowledge of the MDP dynamics.
Plotted is the mean error across the one million randomly initialized experimental settings.
The standard error between settings is negligibly small.

Interestingly, the fixed-points corresponding to weighting $d_b$ consistently have the highest error across feature representations, even on the excursion $\msve$ error metric with weighting $d_b$.
The $\mspbe$ under emphatic weighting, $\mweight$, consistently has the lowest error across all feature representations, though is slightly outperformed by $\mspbe$ with weighting $d_\pi$ for the $\msve$ with weighting $d_\pi$.
In these experiments, the $\msbe$ appears to have no advantages over the $\mspbe$, meaning that the more restricted $\Hset$ for $\mspbe$ produces sufficiently high quality solutions.

\subsection{Feasibility of the Implementation}

There are many feasible choices for estimating $h$.
Likely the simplest is to use the same approximator for $h$ as for the values. For example, this might mean that $h$ and $v$ use the same features, or that we have two heads on a shared neural network. However, we could feasibly consider a much bigger class for $h$, because $h$ is only used during training, not prediction. For example, we might want $v$ to be efficient to query, and so use a compact parametric function approximator. But $h$ could use a more computationally costly function approximator, updated with replay between agent-environment interaction steps.

To expand the space $\Hset$, one feasible approach is to use a separate set of features for $\Hset$ or learn a separate neural network. The separate neural network implicitly can learn a different set of features, and so allows $h$ to use different features than $v$. If we allow this second neural network to be much bigger, then we expand the space $\Hset$ and make the generalized $\mspbe$ closer to the $\msbe$.

We can take this expansion further by using nonparametric function approximators for $h$. For example, a reservoir of transitions can be stored, where $\CEpi{\delta_t}{S_t=s}$ is approximated using a weighted average over $\delta_t$ in the buffer, where the weighting is proportional to similarity between that state and $s$.
This is the strategy taken by the Kernel $\msbe$ \citep{feng2019kernel}, precisely to reduce bias in $h$ and so better approximate the $\msbe$.

When learning online, this non-parametric approach is less practical. Either a large buffer needs to be maintained, or a sufficient set of representative transitions identified and stored. Further, it is not clear that estimating the $\msbe$ more closely is actually desirable, as discussed in the previous section. In this work, where we learn online, we advocate for the simplest approach: a shared network, with two heads (see Section~\ref{sec_alg_control}). We also show that learning two separate neural networks performs comparably, in Section~\ref{sec_h_bases}.

\subsection{Estimation Error and the Impact on Primary Weight Updates}

The generalized $\mspbe$ presents an additional trade-off between approximation error and estimation error in $h$. A rich $\Hset$ may reduce the approximation error (and projection penalty) at the expense of higher error in estimating $\delta_t$ via $h$. A more restricted $\Hset$ may yield lower estimation error, because less data is needed to estimate $h$. Note that this trade-off is for approximating/estimating the objective itself. It is different from---and secondary to---the approximation-estimation trade-off for the value function with set $\Vset$.

One strategy to restrict $\Hset$ is to add regularization on $h$. For example, an $\ell_2$ regularizer constrains $h$ to be closer to zero---reducing variance---and improves convergence rates. This strategy was introduced in an algorithm called TD with Regularized Corrections (TDRC) and the control variant, Q-learning with Regularized Corrections (QRC) \citep{ghiassian2020gradient}. Empirically, these algorithms performed comparably to their TD counterparts, in some cases performing significantly better. This particular constraint on $\Hset$ was particularly appropriate, because the bias from regularization asymptotically disappears: at the TD fixed-point, the true parameters $\hparams$ are zero and regularization biases $\hparams$ towards zero.

More generally, the criteria for selecting $\Hset$ is about improving the primary update, rather than necessarily reducing approximation error or estimation error for $\Hset$. Characterizing how $\Hset$ improves updates for the primary weights remains an open question. One could imagine algorithmic strategies to identify such an $\Hset$ using meta-learning, with the objective to optimize features for $\Hset$ to make the primary weights learn more quickly. This question is particularly difficult to answer, as $h$ can be used in two ways: within the standard saddlepoint update or for gradient corrections, as we discuss further in Section~\ref{sec_algs}. With gradient corrections, the interim bias in $h$ is less problematic than in the saddlepoint update.

In this work, we again advocate for a simple choice: reducing estimation error for $\Hset$ using $\ell_2$ regularization, within a gradient corrections update as in TDRC \citep{ghiassian2020gradient}. We find this choice to be generally effective, and easy to use. Nonetheless, there is clearly much more investigation that can be done to better understand the choice of $\Hset$.

\section{Bounding Value Error \& the Impact of Weighting on Solution Quality}\label{sec_theory_ve}

The desired objective to minimize is the value error with weighting $\dobj$. We, however, optimize a surrogate objective, like the $\mspbe$, with a potentially different weighting $\dsol$. In fact, in the last section in Figure \ref{fig:rw-fixed-points}, we saw the it can be better to pick $\dsol \neq \dobj$, where optimizing $\mspbe$ with $\dsol = m$ produce better solutions in terms of $\msve$ with $\dobj = d_b$ than $\mspbe$ with $\dsol = d_b$.
In this section, we characterize the solution quality under the generalized $\mspbe$, which depends both on $\Hset$ and $\dsol$.

Let $v_{\vecw}$ be the vector consisting of value function estimates $v(s, \vecw)$. Further, let $\vsol$ be the solution to the generalized $\mspbe$. Similarly to prior theoretical work \citep[Equation 5]{yu2010error},
our goal is to find bounds of the form
\begin{equation}
    \underbrace{\| \vsol- v_\pi \|_{\dobj}}_{\text{Value error}} \leq C(\dobj, \dsol,\Hset) \underbrace{\| \Pi_{\Vset,\dsol}v_\pi - v_\pi \|_{\dsol}}_{\text{Approximation error}} \label{eq_generic_bound}
\vspace{-0.1cm}
\end{equation}
where the constant $C(\dobj, \dsol,\Hset)$ in the bound depends on the two weightings and the projection set $\Hset$. The term $\| \Pi_{\Vset,\dsol}v_\pi - v_\pi \|_{\dsol} = \min_{v \in \Vset} \| v - v_\pi \|_\dsol$ represents the approximation error: the lowest error under function class $\Vset$ if we could directly minimize $\msve$ under our weighting $\dsol$. Compared to prior theoretical work \citep[Equation 5]{yu2010error}, here we generalize to the nonlinear setting and where $\dsol$ may not equal $\dobj$.
We start in the case where $\Hset  = \Vset $, and then generalize to $\Hset  \supset \Vset $---where $\Hset$ is a superset of $\Vset $---in the following subsection.
Many of these results build on existing work, which we reference throughout; we also provide a summary table of existing results in the appendix, in Table~\ref{tbl_known}.

\subsection{Upper Bound on $\msve$ when $\Hset  = \Vset $}\label{sec:upper-bound-msve-equal}

Throughout this section we will assume that $\Hset  = \Vset $, so that the projection operator for both the objective and value function space is the same. This matches the setting analyzed for the linear $\mspbe$, though here we allow for nonlinear functions.

Our goal is to characterize the solution to the $\mspbe$, the fixed point $\vsoleq = \Pi_{\Vset,\dsol} \Bo \vsoleq$. The typical approach is to understand the properties of $\Bo$ under norm $\| \cdot \|_\dsol$, as in \citet[Lemma 6.9]{bertsekas1996neurodynamic} or \citet[Theorem 1]{white2017unifying}. However, we can actually obtain more general results, by directly characterizing $\Pi_{\Vset,\dsol}\Bo$ and making assumptions about the norm only for a subset of value functions. This approach builds on the strategy taken by \citet[Theorem 2]{kolter2011fixed}, where the choice of $\dsol$ was constrained to ensure a contraction, and on the strategy taken by \citet[Theorem 4.3]{ghosh2020representations}, where the set of value functions is constrained to ensure a contraction. We combine the two ideas and get a more general condition, as well as an extension to nonlinear function approximation.

\begin{assumption}[Convex Function Space]
The set $\Vset$ is convex.
\end{assumption}
This convexity assumption is needed to ensure the projection operator has the typical properties, particularly that $\| \Pi_{\Vset,\dsol}(\vgen_1 - \vgen_2) \|_{\dsol} \le \| \vgen_1 -\vgen_2 \|_{\dsol}$ for all $\vgen_1 ,\vgen_2 \in \Vset$.

To characterize the Bellman operator, it will be useful to directly define the discounted transition operator (matrix) under $\pi$, $\Ppigamma \in \mathbb{R}^{|S| \times |S|}$, where
\begin{equation}
\Ppigamma(s,s') \defeq \sum_a \pi(a | s) P(s' | s, a) \gamma(s,a,s')
\end{equation}
For a constant discount of $\gamma_c < 1$ in the continuing setting, this simplifies to $\Ppigamma(s,s') = \gamma_c P_\pi$ for $P_\pi(s,s') \defeq \sum_a \pi(a | s) P(s' | s, a)$. We can characterize when the projected Bellman operator is a contraction, by using either the norm of this discounted transition operator or the norm of the projected discounted transition operator for a restricted set of value functions.
\begin{definition}[Discounted Transition Constant]
Define the \emph{discounted transition constant} $\ctransition \defeq \| \Ppigamma \|_\dsol$, the weighted spectral norm of the discounted transition operator $\Ppigamma$.
\end{definition}
\begin{definition}[Operator Constant]
Define the \emph{projected Bellman operator constant} $\contset{\Vset} > 0$ for the set of value functions $\Vsetsub \subseteq \Vset$ as the constant that satisfies
\begin{equation}
\| \Pi_{\Vset,\dsol}\Ppigamma (\vgen_1 - \vgen_2) \|_{\dsol} \le \contset{\Vset} \| \vgen_1 -\vgen_2 \|_{\dsol} \ \ \ \ \text{for any $\vgen_1,\vgen_2 \in \Vsetsub$.}
\end{equation}
Notice that $\contset{\Vset} \le \ctransition$ because $\| \Pi_{\Vset,\dsol}\Ppigamma \vgen \|_{\dsol} \le \| \Ppigamma \vgen \|_{\dsol}$ for any $\vgen$.
\end{definition}
We show next that the $\msve$ of the solution $\vsoleq$ to the generalized $\mspbe$ is upper bounded by the approximation error times a constant. This constant depends only on the discounted transition constant, if it is less than 1, and otherwise depends also on the operator constant.
\begin{theorem} Assume $\contset{\Vset} < 1$. Let
\vspace{-0.3cm}
\begin{equation}
  C(\dsol, \Vset)  \defeq \left \{ \begin{array}{ll}
         \frac{1+\ctransition}{1-\contset{\Vset}} & \mbox{if $\ctransition \geq 1$};\\
        \frac{1}{1- \ctransition} & \mbox{if $\ctransition < 1$}.\end{array} \right.
\end{equation}
Then
\vspace{-0.3cm}
\begin{equation}
    \| \vsoleq- v_\pi \|_{\dsol} \leq C(\dsol, \Vset) \| \Pi_{\Vset,\dsol}v_\pi - v_\pi \|_{\dsol}
    \label{eq_msve_bound}
    .
\end{equation}
\end{theorem}
\begin{proof}
\textbf{Case 1}: $\ctransition < 1$.
This follows using the standard strategy in \citet[Lemma 6.9]{bertsekas1996neurodynamic} or \citet[Theorem 1]{white2017unifying}. Notice first that $\vsoleq = \Pi_{\Vset,\dsol} \Bo \vsoleq$, since it is a solution to the $\mspbe$. Also note that $v_\pi = \Bo v_\pi$.
\begin{align*}
    \| \vsoleq- v_\pi \|_{\dsol}
    &\leq \| \vsoleq- \Pi_{\Vset,\dsol} v_\pi \|_{\dsol} + \| \Pi_{\Vset,\dsol} v_\pi - v_\pi \|_{\dsol}  \\
     &= \| \Pi_{\Vset,\dsol} \Bo \vsoleq - \Pi_{\Vset,\dsol}\Bo v_\pi \|_{\dsol} + \| \Pi_{\Vset,\dsol} v_\pi - v_\pi \|_{\dsol}  \\
          &= \| \Pi_{\Vset,\dsol} \Bo (\vsoleq -  v_\pi) \|_{\dsol} + \| \Pi_{\Vset,\dsol} v_\pi - v_\pi \|_{\dsol}  \\
             &\leq \| \Bo (\vsoleq  -  v_\pi) \|_{\dsol} + \| \Pi_{\Vset,\dsol} v_\pi - v_\pi \|_{\dsol}  \\
             &= \| \Ppigamma (\vsoleq  -  v_\pi) \|_{\dsol} + \| \Pi_{\Vset,\dsol} v_\pi - v_\pi \|_{\dsol}  \\
             &\leq \| \Ppigamma \|_{\dsol} \| \vsoleq  -  v_\pi \|_{\dsol} + \| \Pi_{\Vset,\dsol} v_\pi - v_\pi \|_{\dsol}  \\
             &\leq \ctransition \| \vsoleq  -  v_\pi \|_{\dsol} + \| \Pi_{\Vset,\dsol} v_\pi - v_\pi \|_{\dsol}  \\
 \implies  (1 - \ctransition)  &\| \vsoleq- v_\pi \|_{\dsol} \leq \| \Pi_{\Vset,\dsol}v_\pi - v_\pi \|_{\dsol}.
\end{align*}
Note that in the above $\Bo (\vsoleq  -  v_\pi) = \Bo \vsoleq  -  \Bo v_\pi = \Ppigamma \vsoleq  -  \Ppigamma v_\pi$ because the reward term in the Bellman operator cancels in the subtraction.

\textbf{Case 2}: $\ctransition \ge 1$.
For this case, we use the approach in \citet[Theorem 2]{kolter2011fixed}. We cannot use the above approach, since $(1 - \ctransition)$ is negative. We start again by adding and subtracting $\Pi_{\Vset,\dsol} v_\pi$, but bound the first term differently.
\begin{align*}
    \| \vsoleq- \Pi_{\Vset,\dsol} v_\pi \|_{\dsol}
     &= \| \Pi_{\Vset,\dsol} \Bo \vsoleq - \Pi_{\Vset,\dsol}\Bo v_\pi \|_{\dsol} \\
          &= \| \Pi_{\Vset,\dsol} \Ppigamma \vsoleq -   \Pi_{\Vset,\dsol} \Ppigamma  v_\pi \|_{\dsol} \\
             \leq &\| \Pi_{\Vset,\dsol} \Ppigamma \vsoleq - \Pi_{\Vset,\dsol} \Ppigamma \Pi_{\Vset,\dsol}v_\pi \|_{\dsol} + \| \Pi_{\Vset,\dsol} \Ppigamma \Pi_{\Vset,\dsol}v_\pi - \Pi_{\Vset,\dsol}\Ppigamma v_\pi \|_{\dsol}
             .
\end{align*}
By assumption $\contset{\Vset} < 1$ and both $\vsoleq \in \Vset$ and $\Pi_{\Vset,\dsol}v_\pi  \in \Vset$. Therefore, for the first term
\begin{align*}
 \| \Pi_{\Vset,\dsol} \Ppigamma \vsoleq - \Pi_{\Vset,\dsol} \Ppigamma \Pi_{\Vset,\dsol}v_\pi \|_{\dsol}
 &\le  \contset{\Vset} \| \vsoleq - \Pi_{\Vset,\dsol}v_\pi \|_{\dsol}
 \le  \contset{\Vset} \| \vsoleq - v_\pi \|_{\dsol}
 .
\end{align*}
For the second term, we have that
\begin{align*}
\| \Pi_{\Vset,\dsol} \Ppigamma \Pi_{\Vset,\dsol}v_\pi - \Pi_{\Vset,\dsol} \Ppigamma v_\pi \|_\dsol
&\le \|  \Ppigamma \Pi_{\Vset,\dsol}v_\pi - \Ppigamma v_\pi \|_\dsol
\le \ctransition \|\Pi_{\Vset,\dsol}v_\pi - v_\pi \|_\dsol
.
\end{align*}
Putting this all together, we have
\begin{align*}
    \| \vsoleq- v_\pi \|_{\dsol}
    &\leq \| \vsoleq- \Pi_{\Vset,\dsol} v_\pi \|_{\dsol} + \| \Pi_{\Vset,\dsol} v_\pi - v_\pi \|_{\dsol}  \\
    &\leq \contset{\Vset} \| \vsoleq - v_\pi \|_{\dsol}  + (1+ \ctransition) \| \Pi_{\Vset,\dsol} v_\pi - v_\pi \|_{\dsol} \\
    \implies  (1- \contset{\Vset}) &\| \vsoleq- v_\pi \|_{\dsol}
    \leq  (1+ \ctransition) \| \Pi_{\Vset,\dsol} v_\pi - v_\pi \|_{\dsol}
    .
\end{align*}
\proofspace
\end{proof}
Finally, we can get the desired result in Equation~\eqref{eq_generic_bound}, by considering a different state weighting $\dobj$ for the evaluation versus the state weighting $\dsol$ we use in the $\mspbe$.
For example, in the off-policy setting, $\dobj$ could correspond to $d_\pi$ but we learn under $\dsol = d_b$ or the emphatic weighting $\dsol = m$.
It is appropriate to separate the evaluation and solution weightings, as it may not be feasible to use $\dsol = \dobj$ and further it is even possible we can improve our solution by carefully selecting $\dsol$ different from $\dobj$. Some insights about when this choice can be beneficial has been provided in seminal work on policy gradient methods \citep{kakade2002approximately}. We do not explore when we can obtain such improvements in this work. Our bounds assume worst case differences, as per the next definition, and so the bound on solution quality is bigger when $\dsol \neq \dobj$.
\begin{definition}[State Weighting Mismatch]
Define the state weighting mismatch between the desired weighting $\dobj$ and the weighting used in the solution $\dsol$ as
\begin{equation}
\kappa(\dobj, \dsol) \defeq \max_{s \in \States} \frac{\dobj(s)}{\dsol(s)}
.
\end{equation}
\end{definition}

\begin{corollary} Again assuming $\contset{\Vset} < 1$,
we can further bound the error under a different weighting $\dobj$
\begin{align*}
    \| \vsoleq- v_\pi \|_{\dobj} &\le \sqrt{\kappa(\dobj, \dsol)} C(\dsol, \Vset)\| \Pi_{\Vset, \dsol} v_\pi - v_\pi \|_{\dsol}
    .
\end{align*}
\end{corollary}
\begin{proof}
We use the state mismatch
$\kappa(\dobj, \dsol) = \max_{s \in \States} \frac{\dobj(s)}{\dsol(s)}$
and get
\begin{align*}
    \| \vsoleq- v_\pi \|^2_{\dobj}
    &= \sum_{s \in \States} \dobj(s) (\vsoleq (s) - v_\pi(s) )^2
    = \sum_{s \in \States} \dobj(s) \frac{\dsol(s)}{\dsol(s)} (\vsoleq- v_\pi )^2\\
    &\le \kappa(\dobj, \dsol) \sum_{s \in \States} \dsol(s) (\vsoleq- v_\pi )^2
    = \kappa(\dobj, \dsol) \| \vsoleq- v_\pi \|^2_{\dsol}
    .
\end{align*}
Combined with the inequality in Equation~\eqref{eq_msve_bound},
we obtain
\begin{align*}
    \| \vsoleq- v_\pi \|_{\dobj} &\le \sqrt{\kappa(\dobj, \dsol)} C(\dsol, \Vset) \| \Pi_{\Vset, \dsol} v_\pi - v_\pi \|_{\dsol}
\end{align*}
\proofspace
\end{proof}
This corollary gives us the desired result, specifically with $C(\dobj, \dsol, \Vset) = \sqrt{\kappa(\dobj, \dsol)} C(\dsol, \Vset)$.

The next question is when we can expect $\contset{\Vset} < 1$ or $\ctransition < 1$. More is known about when $\ctransition < 1$, because this is the condition used to prove that on-policy TD and emphatic TD are convergent:
under $d = d_\pi$ and $d = m$, we know that $s_d < 1$ \citep[Theorem 1]{white2017unifying}. It is clear that other weights $d$ will give this result, but we also know certain off-policy settings where $s_d > 1$, such as in Baird's counterexample. Even if $s_d \ge 1$, we can still ensure the projected Bellman operator is a contraction by either limiting the set of weightings $\dsol$ or $\Vset$ so that $\contset{\Vset} < 1$. Of course, if $\dsol = d_\pi$ or $m$, then automatically $\contset{\Vset} < 1$ because for all $\dsol$ and $\Vset$, $\contset{\Vset} \le \ctransition$. We provide some characterization of these constants in the next proposition, including both known results and a few novel ones.
\begin{proposition}\label{prop_statements}
The following statements are true about $\ctransition$.
\begin{enumerate}
\item If $\dsol = d_\pi$ or $\dsol = m$, then $\ctransition < 1$.
\item If $\dsol = d_\pi$ and the discount is a constant $\gamma_c < 1$ (a continuing problem), then $\ctransition  = \gamma_c$.
\item If $\dsol = d_\pi$ and for some constant $\gamma_c$, $\gamma(s,a,s') \le \gamma_c$ for all $(s,a,s')$, then $\ctransition  \le \gamma_c$.
\item If $d_\pi(s) > 0$ iff $\dsol(s) > 0$, then $\ctransition \le s_{d_\pi} \sqrt{\kappa(\dsol,d_\pi)\kappa(d_\pi,\dsol)}$.
\end{enumerate}
\end{proposition}

The proof for this result, as well as the remaining proofs, are in Appendix~\ref{app_proofs}. The previous proofs were included in the main body, as they provide useful intuition.

There is an important relationship between quality of the solution and convergence of TD, given by $\contset{\Vset}$. If $\contset{\Vset} < 1$, then the projected Bellman operator is a contraction. Correspondingly, the update underlying TD converges.
Off-policy TD, therefore, can converge, as long as $\dsol$ or $\Vset$ are chosen such that $\contset{\Vset} < 1$. Additionally, notice that we only have bounds on the solution quality under this same criteria. This suggests the following pessimistic conclusion: gradient TD methods, even if they ensure convergence, may not ensure convergence to a quality solution. Therefore, we should eschew them altogether and instead should focus on controlling the weighting $\dsol$ or the function space $\Vset$.

However, this conclusion only arises due to the looseness of the bound. The theorem statement requires $\contset{\Vset} < 1$ for all of $\Vset$. The proof itself only requires this contraction property for $\vsoleq$ and $\Pi_{\Vset,\dsol}v_\pi$. The projected operator could be an expansion during the optimization and the final quality of the solution under gradient descent will still be good as long as $\contset{\Vsetsubopt} < 1$ for $\Vsetsubopt \subset \Vset$ a small subset of $\Vset$ that contains $\vsoleq$ and $\Pi_{\Vset,\dsol}v_\pi$.

\begin{corollary} Assume there is a convex subset $\Vsetsubopt \subset \Vset$ where  $\vsoleq, \Pi_{\Vset,\dsol}v_\pi \in \Vsetsubopt$ and $\contset{\Vsetsubopt} < 1$. Then
\begin{equation*}
    \| \vsoleq- v_\pi \|_{\dsol} \leq C(\dsol, \Vsetsubopt) \| \Pi_{\Vset,\dsol}v_\pi - v_\pi \|_{\dsol}
    .
\end{equation*}
\end{corollary}

An example of a problem that satisfies $\contset{\Vsetsubopt} < 1$ but not $\contset{\Vset} < 1$ is Baird's counterexample \citep{baird1995residual}. In fact, $v_\pi \in \Vset$, and $\| \vsoleq- v_\pi \|_{\dsol} = 0$. Rather, the initialization is started in a part of the space where the projected Bellman operator is an expansion, causing off-policy TD to diverge.

On the other hand, Kolter's counterexample is one where no such convex subset exists.
The distribution $\dsol$ can be chosen to make $\Amat$ more and more singular,
correspondingly causing the TD solution $\vecw = \Amat^\inv \vecb$ to grow.
The distance between the solution of the $\mspbe$ and the optimal value function can be arbitrarily large and no such convex set $\Vsetsubopt$ can exist.

Overall, these results suggest that we should combine all our of techniques to improve convergence and solution quality. The choice of $\Vset$ and $\dsol$ should be considered, particularly to improve solution quality. If our function class contains the true value function $v_\pi$, then the approximation error $\| \Pi_{\dsol}v_\pi - v_\pi \|_{\dsol} = 0$ and the $\msve$ is zero regardless of the operator constant. Even in this ideal scenario, TD methods can diverge. Gradient methods should be used to avoid convergence issues, because the projected Bellman operator may not be a contraction in all parts of the space. Gradient methods can help us reach a part of the space $\Vsetsubopt$ where $\contset{\Vsetsubopt} < 1$, and so provide an upper bound on the quality of the solution. An exciting open question remains as how to select $\Vset$ and $\dsol$, to obtain $\contset{\Vsetsubopt} < 1$.

\newcommand{\ObMat}{H}

\subsection{Upper Bound on $\msve$ when $\Hset  \supseteq \Vset $}\label{sec:upper-bound-msve-supset}

We next consider more general projections, i.e., for any function space $\Hset  \supseteq \Vset $, that includes the $\msbe$ and $\mspbe$ as special cases. We show two different approaches to upper bounding the $\msve$: re-expressing the objective using oblique projections and characterizing the difference to the $\msbe$, for which there is a straightforward bound on the $\msve$.

\subsubsection{The $\mspbe$ as an Oblique Projection}

The goal of this section is to rewrite the $\mspbe$ as an oblique projection, which then gives a generic bound on the value error. The bound relies on the norm of the oblique projection, which is not easily computable. But, it provides a potential direction for special cases where the norm of this projection might be able to be simplified.

We start by re-expressing the generalized $\mspbe$ as a weighted $\msve$, using the same approach as \citet{schoknecht2003optimality} and \citet{scherrer2010should}. Notice first that $v_\pi = (\eye - \Ppigamma)^\inv r_\pi$. Then the generalized $\mspbe$ for any $v_\vecw \in \Vset$, written in projection form, is
\begin{align*}
    &\| \Pi_{\Hset , d} (T v_{\vecw} - v_{\vecw}) \|^2_{\dsol}
    = \| \Pi_{\Hset , d} (r_\pi + \Ppigamma v_{\vecw} - v_{\vecw}) \|^2_{\dsol} \\
&= \| \Pi_{\Hset , d} (r_\pi - (\eye - \Ppigamma) v_{\vecw}) \|^2_{\dsol} \\
&= \| \Pi_{\Hset , d} [(\eye - \Ppigamma) v_\pi - (\eye - \Ppigamma) v_{\vecw}] \|^2_{\dsol} \hspace{3.0cm} \triangleright \ r_\pi = (\eye - \Ppigamma) v_\pi \\
&= \| \Pi_{\Hset , d} (\eye - \Ppigamma) (v_\pi - v_{\vecw}) \|^2_{\dsol} \\
&= \| v_\pi - v_{\vecw} \|^2_{\ObMat}  \hspace{5.0cm} \triangleright \ObMat \defeq (\eye - \Ppigamma)^\top \Pi_{\Hset , d}^\top D  \Pi_{\Hset , d} (\eye - \Ppigamma)
\end{align*}
Minimizing the generalized $\mspbe$ therefore corresponds to minimizing the $\msve$ with a reweighting over states that may no longer be diagonal, as $\ObMat$ is not a diagonal matrix. In fact, we can see that the solution to the generalized $\mspbe$ is a projection of $v_\pi$ onto set $\Vset$ under weighting $\ObMat$, namely $v = \Pi_{\Vset,\ObMat} v_\pi$. A projection under such a non-diagonal weighting is called an oblique projection.
Using this form, we can obtain an upper bound using a similar approach to \citep[Proposition 3]{scherrer2010should}, with proof in Appendix~\ref{app_proofs}.
\begin{theorem}\label{thm_oblique}
If $\Hset \supseteq \Vset$, then the solution $\vsol$ to the generalized $\mspbe$ satisfies
\begin{equation}
\| v_\pi - \vsol \|_{\dsol} \le \| \Pi_{\Vset,\ObMat}  \|_\dsol \| v_\pi -  \Pi_{\Vset, \dsol} v_\pi \|_{\dsol}
.
\end{equation}
\end{theorem}
We can next extend a previous result, which expressed this projection under linear function approximation for both the TD fixed point and the solution to the $\msbe$. We can now more generally express the oblique projection for  interim $\Hset$, with proof in Appendix \ref{app_proofs}.
\begin{corollary}\label{cor_oblique_sol}
Assume $\Vset$ is the space of linear functions with features $\vecx$ and $\Hset$ the space of linear functions with features $\phivec$, with $\Vset \subseteq \Hset$. Then the solution to the $\mspbe$ is
\begin{equation}
\vsol = \Xmat \wsol \ \ \ \text{ for } \wsol = (M^\top  (\eye - \Ppigamma) \Xmat)^\inv M^\top r_\pi
\end{equation}
for $M \defeq \Pi_{\Hset, \dsol}^\top \Dsmat (\eye - \Ppigamma) \Xmat$.
Further, $\vsol = \Pi_{\Vset,\ObMat} v_\pi$ for
\begin{equation}
\Pi_{\Vset,\ObMat} = \Xmat (M^\top (\eye - \Ppigamma) \Xmat)^\inv M^\top (\eye - \Ppigamma)
.
\end{equation}
\end{corollary}

\newcommand{\approxerr}{\mathrm{ApproxError}}
\subsubsection{The Distance between the $\mspbe$ and the $\msbe$}

There is a well known upper bound on the $\msve$, in terms of the $\msbe$. If we can characterize the distance of the $\mspbe$ to the $\msbe$, we could potentially exploit this known result to upper bound the $\msve$ in terms of the $\mspbe$. Such a characterization should be possible because, for a sufficiently large $\Hset$, the $\mspbe$ is equivalent to the $\msbe$. For smaller $\Hset$, we can view the $\mspbe$ as an approximation to the $\msbe$, and so can characterize that approximation error.

For a given $\vgen \in \Vset$, let $h_\vgen^*(s) = \mathbb{E}[\delta(\vgen) | S = s]$ and define
\begin{equation*}
\approxerr(\Hset,\vgen) \defeq \min_{h \in \Hset} \| h_\vgen^* - h \|_{\dsol}
\end{equation*}
with worst-case approximation error
\begin{equation}
\approxerr(\Hset) \defeq \max_{\vgen \in \Vset} \approxerr(\Hset,\vgen)
\end{equation}
where we overload notation because the distinction is clear from the arguments given. Then we can obtain the following result, with proof in Appendix~\ref{app_proofs}.
\begin{theorem}\label{thm_approx_h}
For any $\vgen \in \Vset$,
\begin{align*}
\| v_\pi - \vgen \|_{\dsol}
&\le \| (\eye - \Ppigamma)^\inv \|_{\dsol} \underbrace{\|\Bo v - v \|_{\dsol}}_{\msbe}\\
&\le \| (\eye - \Ppigamma)^\inv \|_{\dsol} \Big( \underbrace{\|\Pi_{\Hset, \dsol} \Bo v - v \|_{\dsol}}_{\mspbe} + \approxerr(\Hset,v) \Big)\\
&\le \| (\eye - \Ppigamma)^\inv \|_{\dsol} \Big(\|\Pi_{\Hset, \dsol} \Bo v - v \|_{\dsol}+ \approxerr(\Hset) \Big)
.
\end{align*}
Note that if $\ctransition < 1$, then $\| (\eye - \Ppigamma)^\inv \|_{\dsol} \le (1-\ctransition)^\inv$.
\end{theorem}
The upper bound depends on the value of the $\mspbe$ and the approximation error of $\Hset$.
For $\Hset = \Vset$, the $\mspbe$ is zero but the approximation error is likely higher. As $\Hset$ gets bigger, the approximation error gets smaller, but the $\mspbe$ is also larger. Once $\Hset$ is big enough to include $h^*$ for a given $v$, then the approximation error is zero but the $\mspbe$ is at its maximum, which is to say it is equal to the $\msbe$. This upper bound is likely minimized for an interim value of $\Hset$, that balances between the error from the $\mspbe$ and the approximation error.

It is important to recognize that the approximation error for $\Vset$ and $\Hset$ can be quite different. For example, if $v_\pi \in \Vset$, then $\approxerr(\Hset, v_\pi) = 0$ if $0 \in \Hset$. The function that returns zero for every state should be in $\Hset$ to ensure zero approximation error, but need not be in $\Vset$. For many convex function spaces we consider, all convex combinations of functions are in $\Vset$, and it is likely that we have 0 in $\Vset$. But nonetheless, there are likely instances where $\Vset$ contains near optimal value functions, but that same set produces high approximation error for $\Hset$. An important future direction is to better understand the differences in function spaces needed for $\Vset$ and $\Hset$.

\section{Algorithms for the Generalized $\mspbe$}\label{sec_algs}

The linear $\mspbe$ is often optimized using gradient correction algorithms as opposed to saddlepoint methods. The canonical methods are TDC (gradient corrections) and GTD2 (saddlepoint), where TDC has been consistently shown to perform better than GTD2 \citep{white2016investigating,ghiassian2020gradient}.
We show that similar gradient correction algorithms arise for the generalized $\mspbe$, and discuss such an algorithm called TDRC.

\subsection{Estimating the Gradient of the Generalized $\mspbe$}\label{sec:estimating_gradient_gpbe}

To see why (at least) two classes of algorithms arise, consider the gradient for the generalized $\mspbe$, for a given $h(s) \approx \CEpi{\delta(\vecw)}{S=s}$ with a stochastic sample $\delta(\vecw)$ from $S = s$:
\begin{align*}
-\nabla_\vecw \delta(\vecw) h(s) &= h(s) [\nabla_\vecw \paramv[s] - \gamma \nabla_\vecw \paramv[S']]
.
\end{align*}
This is the standard \emph{saddlepoint update}.
The key issue with this form is that any inaccuracy in $h$ has a big impact on the update of $\vecw$, and $h$ can be highly inaccurate during learning. Typically, it is initialized to zero, and so it multiples the update to the primary weights by a number near zero, making learning slow.

The \emph{gradient correction update}
is preferable because it relies less on the accuracy of $h$. The first term uses only the sampled TD-error.
\begin{align*}
\Delta \vecw &\gets \delta(\vecw) \nabla_\vecw \paramv[s] - \paramh[s]\gamma \nabla_\vecw \paramv[S']
\end{align*}
where $\Delta \hparams \gets (\delta(\vecw) - \paramh[s])\nabla_\hparams \paramh[s]$ just like the saddlepoint update. But, the update is biased because it assumes it has optimal $h^* \in \Hset$ for part of the gradient.
To see why, we extend the derivation for the linear setting.
\begin{align*}
-\nabla_\vecw \delta(\vecw) h(s) &= h(s) [\nabla_\vecw \paramv[s]- \gamma \nabla_\vecw \paramv[S']]\\
&= h(s) \nabla_\vecw \paramv[s] - h(s)\gamma \nabla_\vecw \paramv[S']\\
&= (h(s) - \delta(\vecw) + \delta(\vecw)) \nabla_\vecw \paramv[s] - h(s)\gamma \nabla_\vecw \paramv[S']\\
&= \delta(\vecw) \nabla_\vecw \paramv[s] + (h(s) - \delta(\vecw)) \nabla_\vecw \paramv[s] - h(s)\gamma \nabla_\vecw \paramv[S']
\end{align*}
This resembles the gradient correction update, except it has extra term $(h(s) - \delta(\vecw)) \nabla_\vecw \paramv[s]$. In the linear setting, if we have the linear regression solution for $h^*$ with parameters $\hparams$,
then this second term is zero in expectation. This is because $ \nabla_\vecw \paramv[s] = \vecx(s)$, giving
\begin{equation*}
\CEpi{(h(s) - \delta(\vecw)) \nabla_\vecw v(s, \vecw)}{S=s}
=  \vecx(s) \vecx(s)^\top \hparams - \vecx(s)\CEpi{\delta(\vecw)}{S=s}
\end{equation*}
and so in expectation across all states, because $\hparams = \mathbb{E}[\vecx(S) \vecx(S)^\top]^\inv \mathbb{E}[\vecx(S) \delta]$, we get that
\begin{align*}
\mathbb{E}[(h(S) - \delta(\vecw)) \nabla_\vecw v(S, \vecw)]
&=  \mathbb{E}[\vecx(S) \vecx(S)^\top] \mathbb{E}[\vecx(S) \vecx(S)^\top]^\inv \mathbb{E}[\vecx(S) \delta(\vecw)] - \mathbb{E}[\vecx(S)\delta(\vecw)]\\
&=  \mathbb{E}[\vecx(S) \delta(\vecw)] - \mathbb{E}[\vecx(S)\delta(\vecw)] = 0
.
\end{align*}
Therefore, given the optimal $h \in \Hset $ for $\Hset$ the set of linear functions, this term can be omitted in the stochastic gradient and still be an unbiased estimate of the full gradient.

More generally, the same reasoning applies if $h(s)$ can be re-expressed as a linear function of $\nabla_\vecw \paramv[S]$. This provides further motivation for using features produced by the gradient of the values, as in the nonlinear $\mspbe$, to estimate $h$. Another choice is to use the features in the last layer of the neural network used for $\paramv[S]$. Because the output is a linear weighting of features from the last layer, $\nabla_\vecw \paramv[S]$ includes this last layer as one part of the larger vector.
A head for $h$ can be added to the neural network, where $h$ is learned as a linear function of this layer. Its updates do not influence the neural network itself---gradients are not passed backwards through the network---to ensure it is a linear function of the last layer.

Unlike the saddlepoint update, however, the gradient correction update is no longer a straightforward gradient update, complicating analysis. It is possible, however, to analyze the dynamical system underlying these updates.

The asymptotic solution does not require the omitted term, under certain conditions on $h$, as discussed above. If the dynamical system moves towards this stable asymptotic solution, then convergence can be shown.
The TDC update relies on just such a strategy: the joint update is rewritten as a linear system, that is then shown to be a contraction that iterates towards a stable solution \citep{maei2011gradient}. The extension to the nonlinear setting is an important open problem.

Theoretical work for these updates under nonlinear function approximation has been completed, for an algorithm called SBEED \citep{dai2018sbeed}. SBEED uses a gradient correction update, but the theory is for the saddlepoint version. This work investigates a slightly different saddlepoint update, learning $h(s)  \approx \CE{R + \gamma \paramv[S']}{S=s}$, instead of estimating the entire TD error. A TD-error estimate can be obtained using $h(s) - \paramv[s]$. Even under this alternate route, it was reconfirmed that the gradient correction update was preferable---since the final proposed algorithm used this form---but the theory non-trivial.

\textbf{Remark:} Another way to interpret gradient correction algorithms is as approximations of the gradient of the $\msbe$. We can consider two forms for the negative gradient of the $\msbe$:
\begin{align*}
&\CEpi{\delta(\vecw)}{S=s} \CEpi{\nabla_\vecw \paramv[s]  - \gamma \nabla_\vecw \paramv[S']}{S=s}\\
 \text{or} \hspace{2.0cm} &\CEpi{\delta(\vecw)}{S=s}\nabla_\vecw \paramv[s] - \CEpi{\delta(\vecw)}{S=s} \CEpi{\gamma \nabla_\vecw \paramv[S']}{S=s}
\end{align*}
because $\paramv[s]$ is not random. We can estimate the first form of the gradient using an estimate $h(s) \approx \CEpi{\delta(\vecw)}{S=s}$:
$h(s) (\nabla_\vecw \paramv[s] - \gamma \nabla_\vecw \paramv[S'])$.
This corresponds to a saddlepoint update. To estimate the second form of the gradient, notice that we do not have a double sampling problem for the first term. This means we can use $\delta$ to compute an unbiased sample for the first term:
$\delta \nabla_\vecw \paramv[s] - h(s) \gamma \nabla_\vecw \paramv[S']$.
 This strategy corresponds to the gradient correction update.

\subsection{TDRC:\ A Practical Gradient-Based Prediction Algorithm}\label{sec:alg_pred}
In this section we incorporate the above insights into an easy-to-use gradient-based algorithm called TD with Regularized Corrections (TDRC). Empirical work with TDC suggests that it performs well because the secondary stepsize is set small, and thus it is effectively using conventional off-policy TD \citep{ghiassian2020gradient}.
TDRC builds on this insight, providing a method that performs similarly to TD when TD performs well, while maintaining the theoretical soundness of TDC.\@ The key idea is to regularize $h$ to be near zero, using an $\ell_2$ regularizer, where if $h(s) = 0$ then the update corresponds to the TD update.

The motivation for this approach is that it can help learn $h$ more quickly, without incurring much bias.
In the linear setting with $\Vset = \Hset$, the parameters $\hparams = 0$ for $h^*$ at the TD fixed point. Applying $\ell_2$ regularization to the secondary weights, therefore, incurs little bias asymptotically.
But, adding a strongly convex regularizer can improve the convergence rate of the secondary variable, and reduce variance of the estimator.

In fact, this perspective on $h(s)$ helps to explain the unreasonably good performance of TD in many practical settings.
TD is equivalent to TDC with $h(s) = 0$, which is a zero variance but highly biased estimate of the expected TD error.
Asymptotically as our estimates for the primary variable improve, the bias of this heuristic choice of $h$ decreases until we converge to the same fixed-point as TDC (in the cases where TD converges).

The TDRC update equations with importance sampling ratios $\rho(s,a) = \tfrac{\pi(a | s)}{b(a | s)}$ are
\begin{align*}
    \Delta \vecw &\leftarrow \rho(s,a)\delta(\vecw) \nabla_\vecw \paramv[s] - \rho(s,a) \paramh[s] \gamma \nabla_\vecw \paramv[S'] \\
    \Delta \hparams &\leftarrow \rho(s,a)(\delta(\vecw) - \paramh[s]) \nabla_\hparams \paramh[s] - \beta \hparams
\end{align*}
where $\vecw$ and $\hparams$ are parameters of $v : \States \to \Re$ and $h : \States \to \Re$ respectively.
Notice that the only difference between TDRC and TDC is the inclusion of the $-\beta \hparams$ term in the update for $\hparams$.
It can be shown in the linear setting with $\Vset = \Hset$ that the fixed-points of the TDRC dynamical system are the same as the TDC system and that, under some conditions on the learning rates, TDRC converges to its fixed-point on any MDP \citep{ghiassian2020gradient}.

By adding an $\ell_2$ regularizer with regularization parameter $\beta$ to TDC, we can interpolate between TD and TDC.\@
With $\beta$ very large the solution set of $h(s)$ becomes diminishingly small until it contains only the point $h(s) = 0$ as $\beta \to \infty$.
On the other hand, as $\beta$ approaches zero, then the solution set of $h(s)$ becomes that of the $\mspbe$ and TDRC approaches TDC.\@
Because $\beta$ scales between two known good algorithms, this suggests that the impetus to highly tune $\beta$ is small: many choices of $\beta$ should yield a reasonably performing algorithm.

There are many options to approximate both $v(s)$ and $h(s)$ using neural networks.
In our experiments, we chose to use a single network with two sets of outputs, one to predict $v(s)$ and the other to predict $h(s)$.
To avoid difficulties with balancing between two loss functions, we only use the gradients from the primary update to adjust the weights of the network and use the gradients from the secondary update only to adjust the weights for that head.
The resulting architecture for estimating action values is shown in Figure~\ref{fig:network_architecture}.

\section{Extending to Control}

The previous development was strictly for policy evaluation. The formulation of a sensible generalized $\mspbe$ for control, however, can be obtained using a similar route. The conjugate form has already been used to develop a novel control algorithm for nonlinear function approximation, called SBEED \citep{dai2018sbeed}. The SBEED algorithm explicitly maintains a value function and policy, to incorporate entropy regularization, and uses the gradient correction update. We develop an alternative control algorithm that learns only action-values and uses the gradient correction update.

\subsection{The Control Objective}

Our goal is to approximate $q^*: \States \times \Actions \rightarrow \mathbb{R}$, the action-values for the optimal (soft) policy. Instead of the Bellman operator, we use the Bellman optimality operator or generalizations that use other forms of the max but are still guaranteed to be contractions, like the mellow-max operator \citep{asadi2017alternative}. Let $m$ be the given max operator, that takes action-values and returns a (soft) greedy value. In Q-learning, we use a hard-max $m(q(s, \cdot)) = \max_a q(s,a)$ and in mellow-max, $m(q(s, \cdot)) = \tau^\inv \log \left(\frac{1}{|\Actions|} \sum_{a \in \Actions} \exp(\tau q(s,a)) \right)$.
 As $\tau \rightarrow \infty$, the mellow-max operator approaches the hard-max operator.
The Bellman optimality operator $\Bo_m$ corresponds to
\begin{equation}
(\Bo_m q)(s,a) \defeq \CE{R + \gamma(S,A,S') m(q(S', \cdot))}{S=s, A=a}
\end{equation}
where the expectation is over next state and reward. The fixed point of this operator is $q^*$.

To learn $\paramq[s,a]$ approximating $q^*(s,a)$, we define the $\msbe$ for control
\begin{equation}
\msbe(\vecw) \defeq \sum_{s,a} d(s,a) \CE{\delta(\vecw)}{S=s, A=a}^2
\end{equation}
where $\delta(\vecw) \defeq R + \gamma(S,A,S') m(\paramq[S',\cdot]) - \paramq[S, A]$ and
$d: \S \times \A \rightarrow [0, \infty)$ is some weighting. We overload the notation for the weighting $d$, to make the connection to the previous objectives clear.
We can rewrite this objective using conjugates, to get
\begin{align*}
\msbe(\vecw)
&= \sum_{s\in\S, a\in\A} d(s,a) \max_{h \in \mathbb{R}} \left(2\CE{\delta(\vecw)}{S=s, A=a}h - h^2 \right) \\
&=  \max_{h \in \Fsetall} \ \sum_{s\in\S, a\in\A} d(s,a) \left(2\CE{\delta(\vecw)}{S=s, A=a}h(s,a) - h(s,a)^2 \right)
.
\end{align*}
As before, this maximization can be rewritten as a minimization,

where the optimal $h^*(s,a) = \CE{\delta(\vecw)}{S=s, A=a}$. This equivalence is true for the hard-max operator or the mellow-max, even though the operator is no longer smooth.
Finally, in practice, we will learn an approximate $h$, from the set $\Hset$, resulting in a $\mspbe$ for control:
\begin{align*}
\mspbe(\vecw)
&\defeq  \max_{h \in \Hset} \ \sum_{s\in\S, a\in\A} d(s,a) \left(2\CE{\delta(\vecw)}{S=s, A=a}h(s,a) - h(s,a)^2 \right)
.
\end{align*}
This objective is the first generalized $\mspbe$ for learning action-values for control.

The algorithm is a simple modification of the policy evaluation algorithms above, described briefly here and expanded upon more in the next section. The update to $h(s,a)$ is still a gradient of a squared error to the TD error. The saddlepoint gradient update, for a given $h(s,a)$ with a stochastic sample $\delta(\vecw)$ from $S = s, A = a$:
\begin{align*}
-\nabla_\vecw \delta(\vecw) h(s,a) &= h(s,a) [\nabla_\vecw \paramq[s,a] - \gamma \nabla_\vecw m(\paramq[S', \cdot])]
\end{align*}
with gradient correction form
\begin{align*}
\delta(\vecw) \nabla_\vecw \paramq[s,a] - \gamma h(s,a) \nabla_\vecw m(\paramq[S',\cdot])
.
\end{align*}
Both involve taking a gradient through the max operator $m$. For the hard-max operator, this results in a subgradient. The mellow-max operator, on the other hand, is differentiable with derivative
\begin{equation*}
\frac{\partial}{\partial w_i} m(\paramq[s,\cdot])=  \frac{1}{\sum_{a \in \Actions} \exp(\tau \paramq[s, a])} \sum_{a \in \Actions}  \exp(\tau \paramq[s, a]) \frac{\partial}{\partial w_i}\paramq[s, a]
\end{equation*}

We can reason similarly about the validity of the gradient correction update.
\begin{align*}
-\nabla_\vecw &\delta(\vecw) h(s,a)
= (h(s,a) - \delta(\vecw) + \delta(\vecw)) \nabla_\vecw \paramq[s,a] - h(s,a)\gamma \nabla_\vecw m(\paramq[S',\cdot])\\
&= \delta(\vecw) \nabla_\vecw \paramq[s,a] + (h(s,a) - \delta(\vecw)) \nabla_\vecw \paramq[s,a] - h(s,a)\gamma \nabla_\vecw m(\paramq[S',\cdot])
.
\end{align*}
As before, we can conclude that we can drop this second term, as long as the optimal $h \in \mathcal{H}$ is representable as a linear function of $\nabla_\vecw \paramq[s,a]$. The fixed point for the gradient correction updates that drop the term $(h(s,a) - \delta(\vecw)) \nabla_\vecw \paramq[s,a]$ will still converge to the same fixed point, if they converge. The key question that remains is, if the dynamical system produced by these equations does in fact converge.

\begin{figure}[t!]
    \centering
    \includegraphics[width=0.5\textwidth]{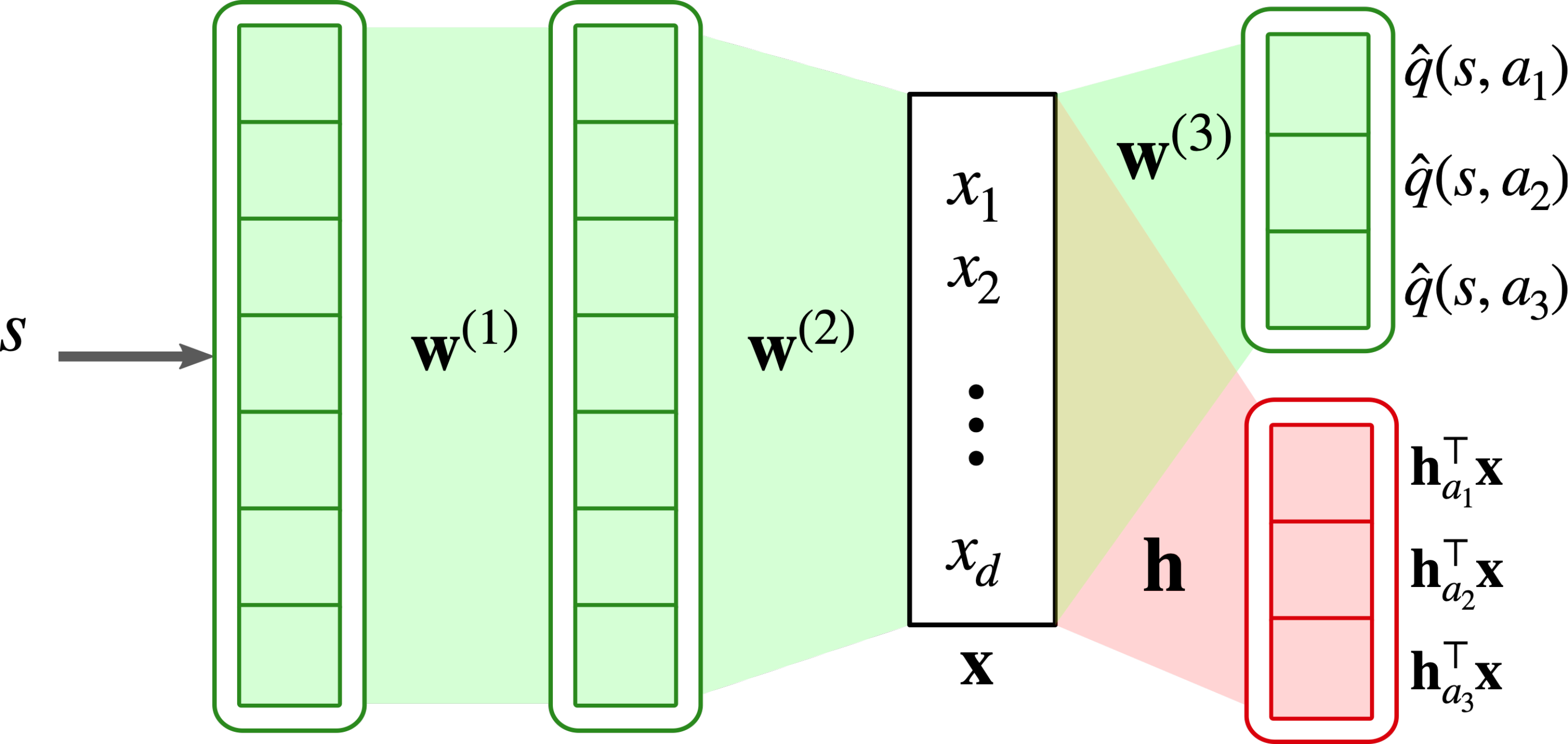}
    \caption{\label{fig:network_architecture}
        Visualization of the neural network architecture used to approximate $\hat q(s, a)$ and $h(s, a)$ for the QRC algorithm.
        TDRC uses the same network architecture, except with only one output for each head of the network predicting $\hat v(s)$ and $h(s)$ respectively.
        The green shaded regions represent gradient information from the update to $\hat q(s, a)$ and the red shaded region represents the gradient information from the update $h(s, a)$.
        Notice only $\hat q(s, a)$ modifies the weights of the earlier layers of the network.
    }
\end{figure}

\subsection{QRC:\ A Practical Gradient-Based Control Algorithm}\label{sec_alg_control}
In this section, we modify the TDRC algorithm specified in Section~\ref{sec:alg_pred} for control. Let $\hparams_{t, A_t}$ be the weights for only the secondary head predicting $h(S_t, A_t)$, $\vecx_t$ the last layer of the network---giving the features for $h$---and $\vecw_t$ all of the remaining weights in the network. We use the mellowmax operator, to get the following updates for the QRC algorithm
\begin{align*}
    \delta_t = R_{t+1}& + \gamma m(\paramq[S_{t+1}, \cdot]) -  \paramq[S_t, A_t], \text{ and } h_t =  \hparams_{t, A_t}\tr \vecx_t\\
    \vecw_{{t+1}} &\leftarrow \vecw_t + \alpha \delta_t \nabla_w \paramq[S, A] - \alpha \gamma h_t  \nabla_\vecw m(\paramq[S',\cdot])\\
    \hparams_{t+1, A_t} &\leftarrow \hparams_{t, A_t} + \alpha \left[ \delta_t - h_t \right] \vecx_t - \alpha\beta \hparams_{t, A_t}
\end{align*}

QRC is similar to SBEED \citep{dai2018sbeed}, but has two key differences. The first is that SBEED learns a  state-value function and an explicit policy. We learn action-values and use a direct mellowmax on the action-values to compute the policy.
The other key difference is that the SBEED update uses a different form of gradient correction, which estimates only the Bellman step instead of the TD error in order to correct the gradient.
Recall that QRC interpolates between a gradient correction update at one extreme of its hyperparameter ($\beta = 0$) and Q-learning at the other extreme (large $\beta$).
SBEED, on the other hand, interpolates between a gradient correction update and a residual-gradient method minimizing the mean squared TD error.
When QRC converges, it converges to the same fixed-point for all values of its hyperparameter.
SBEED, however, interpolates between the fixed-point for the $\mspbe$ and the mean squared TD error.

\section{Empirical Investigation of the Generalized $\mspbe$ for Control}

In this section, we empirically investigate QRC.\@
We first provide a comparison of QRC with Q-learning and SBEED across four benchmark domains. Then we delve more deeply into the design choices in QRC.
We compare using gradient corrections within QRC to the saddlepoint form in Section~\ref{sec_gc_vs_sp}.
We then investigate how QRC performs with separate bases for $h$, including a more powerful basis for $h$, in Section~\ref{sec_h_bases}.

\blfootnote{\noindent Code for all experiments is available at https://github.com/rlai-lab/Generalized-Projected-Bellman-Errors}

\subsection{Benchmark Environments and Experimental Design}\label{sec:benchmark-envs}

We used four simulation domains with neural network function approximation.
We chose simulation domains with a sufficiently small state dimension to efficiently compare algorithms across many different random initializations of neural network, while still computing the required number of experiment repetitions required for statistical significance.
On the other hand, to tease out differences between algorithms, we require domains with sufficiently complex learning dynamics. We chose three classic control domains known to be challenging when approximation resources and agent-environment interactions are limited: Acrobot \citep{sutton1996generalization}, Cart Pole \citep{barto1983neuronlike}, and Mountain Car \citep{moore1990efficient}. We also used Lunar Lander \citep{brockman2016openai} to investigate performance in a domain with a dense reward function and moderately higher-dimensional state.

The network architectures were as follows.
For Acrobot and Mountain Car, we used two layer fully-connected neural networks with 32 units in each layer and a ReLU transfer.
The output layer has an output for each action-value and uses a linear transfer.
For the Cart Pole and Lunar Lander domains, we used the same architecture except with 64 units in each hidden layer.
We use a shared network with multiple heads for all algorithms unless otherwise specified.
In experiments with policy-gradient-based methods the parameterized policy uses an independent neural network.
We do not use target networks.

We swept consistent values of the hyperparameters for every experiment.
We swept the stepsize parameter over a wide range $\alpha \in \{2^{-12}, 2^{-11}, \ldots, 2^{-7} \}$ for every algorithm.
For algorithms which chose a stepsize on the boundary of this range---for instance, GQ often chose the smallest stepsize---we performed a one-off test to ensure that the range was still representative of the algorithm's performance.
All algorithms used mellowmax, with $\tau$ swept in the range $\tau \in \{0, 10^{-4}, 10^{-3}, \ldots, 10^0 \}$, including 0 to allow algorithms to choose to use a hard-max.
Algorithms based on the SBEED update have an additional hyperparameter $\eta$ which interpolates between the gradient correction update and a residual gradient update.
For all experiments we swept values of $\eta \in \{10^{-3}, 10^{-2}, 10^{-1}, 10^0 \}$ and the ratio between the actor and critic stepsizes $\nu \in \{2^{-4}, 2^{-3}, \ldots, 2^1\}$, often giving SBEED algorithms twenty-four times as many parameter permutations to optimize over compared to other algorithms.
Likewise, we allowed saddlepoint methods (GQ) to optimize over the regularization parameter $\beta \in \{0, 0.5, 1, 1.5\}$, to give them an opportunity to perform well.

The remaining hyper-parameters were not swept, but instead set to reasonable defaults.
We used a replay buffer to store the last 4000 transitions, then sampled 32 independent transitions without replacement to compute mini-batch averaged updates.
We used the ADAM optimizer \citep{kingma2015adam} for all experiments with the default hyperparameters, a momentum term of $\beta_1 = 0.9$ and a squared-gradient term of $\beta_2 = 0.999$.
We additionally tested Stochastic Gradient Descent and RMSProp and found that most conclusions remain the same, so choose not to include these results to focus the presentation of results.
For each of the four domains we use a discount factor of $\gamma = 0.99$ and cutoff long-running episodes at 500 steps for Acrobot and Cart Pole and 1000 steps for Mountain Car.
On episode cutoff events, we do not make an update to the algorithm weights to avoid bootstrapping over this imaginary transition and on true episode termination steps we update with $\gamma = 0$.

We use a non-conventional performance measure to more fairly report algorithm performance.
A common performance metric is to report the cumulative reward at the end of each episode, running each algorithm for a consistent number of episodes. This choices causes algorithms to have different amounts of experience and updates.
Some algorithms use more learning steps in the first several episodes and achieve higher asymptotic performance because they effectively learned for more steps.
We instead report the cumulative reward from the current episode on each step of the current episode.
For example in Mountain Car, if the $kth$ episode takes 120 steps, then we would record -120 for each step of the episode.
We then run each algorithm for a fixed number of steps instead of a fixed number of episodes, so that each algorithm gets the same number of learning steps and a consistent amount of data from the environment. We record performance over 100,000 steps, recorded every 500 steps---rather than every step---to reduce storage costs.

To avoid tuning the hyperparameters for each algorithm for every problem, we start by investigating a single set of hyperparameters for each algorithm across all four benchmark domains.
We evaluate the hyperparameters according to mean performance over runs, for each domain.
We then use a Condorcet voting procedure to find the single hyperparameter setting that performs best across all domains.

\begin{figure}
    \centering
	\includegraphics[width=\linewidth]{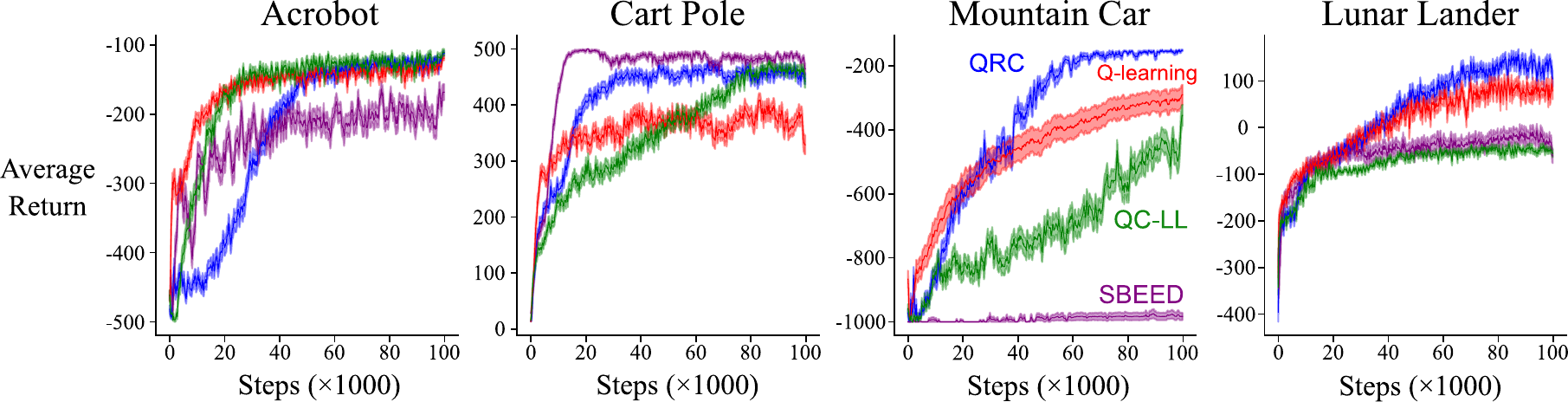}
    \caption{\label{fig:bakeoff}
        {\em Learning curves using the best performing hyperparameters across 4 domain.}
        The learning curves above are averaged over 100 independent runs and shaded regions correspond to one standard error.
    }
\end{figure}

\subsection{Overall Results in the Benchmark Environments}\label{sec:benchmark-experiments}
Figure~\ref{fig:bakeoff} shows the learning curves for each algorithm with the single best performing hyperparameter setting across domains.
QRC was the only algorithm to consistently be among the best performing algorithms on every domain and was the only algorithm with a single hyperparameter setting that could solve all four domains.
Although SBEED was given twenty-four times as many hyperparameter combinations to optimize over, its performance was consistently worse than all other benchmark algorithms.
This suggests that the voting procedure was unable to identify a single hyperparameter setting that was consistently good across domains.
We additionally include the nonlinear control variant of TDC with a locally linear projection which we call QC-LL \citep{maei2009convergent}.
QC-LL performed well on the two simpler domains, Acrobot and Cart Pole, but exhibited poor performance in the two more challenging domains.
We report other voting procedures as well as the performance of the best hyperparameters tuned for each domain independently in Appendix~\ref{app:additional-results}.

\begin{figure}
    \centering

    \includegraphics[width=\linewidth]{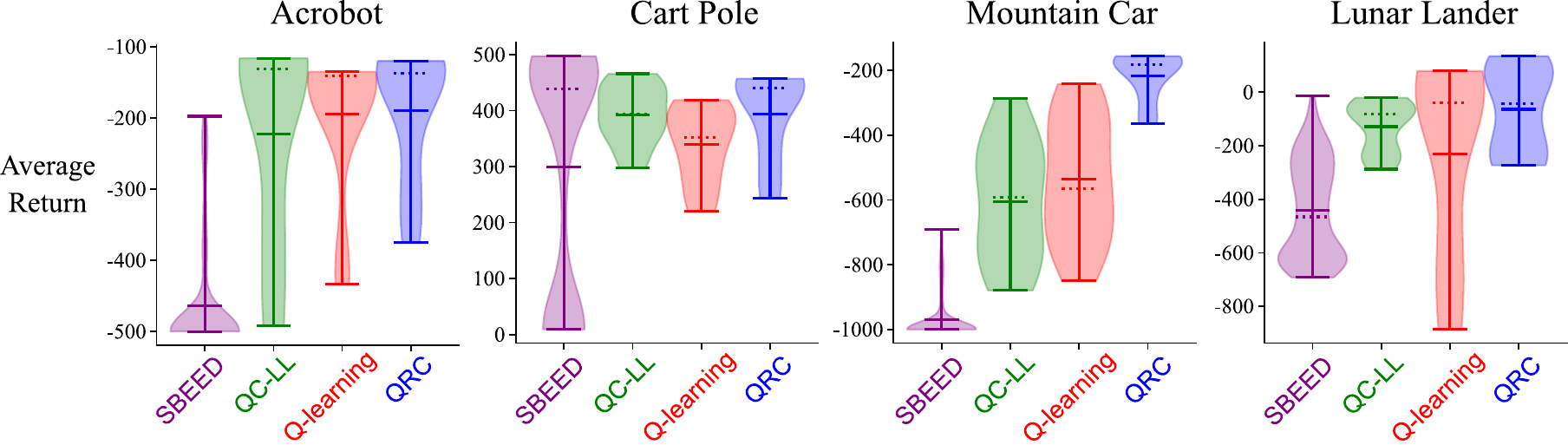}
    \caption{\label{fig:bakeoff-param-dist}
        {\em Distribution of average returns over hyperparameter settings for each benchmark domain.} The vertical axis represents the average performance of each hyperparameter setting (higher is better) and the width of each curve represents the proportion of hyperparameters which achieve that performance level, using a fitted kernel density estimator.
        The solid horizontal bars show the maximum, mean, and minimum performance respectively and the dashed horizontal bar represents the median performance over hyperparameters.
        QRC in blue generally performs best and exhibits less variability across hyperparameter settings.
    }
\end{figure}

To understand how hyperparameter selection impacts the performance of each algorithm, we visualize the performance distribution over hyperparameters in Figure~\ref{fig:bakeoff-param-dist}.
Ideally, we prefer the variability in performance across hyperparameters to be small and the distribution to be concentrated around a single performance level near the top of the plot.
Plots with an hourglass shape represent bimodal distributions where several hyperparameter values perform well and several perform poorly.
Plots where the mean and median horizontal markers are quite separated indicate highly skewed distributions.

All algorithms exhibit different hyperparameter performance distributions across the four domains. The SBEED algorithm often has one (or a few) hyperparameter setting(s) which perform well, especially in the Cart Pole domain.
This suggests that SBEED is highly sensitive to its hyperparameters.
and even small deviations from the ideal hyperparameters can lead to very low performance.
Q-Learning generally exhibits wide spread of performance over hyperparameters, with a high skew in two of the problem settings.
QRC has much lower spread of performance over hyperparameters often with the bulk of the distribution located near the highest performing hyperparameter setting.

\subsection{Investigating Variability Across Runs}

In the previous section we investigated performance averaged across runs; in this section, we investigate how much the algorithms vary across runs. Two algorithms could have similar average returns across runs, even if one has many poor runs and many high performing runs and the other has most runs that have reasonable performance.

To observe these differences, we visualize the distribution of returns across 100 runs, as in Figure~\ref{fig:example-performance-dist}. The horizontal axis shows the return achieved for a given run, averaged over the last 25\% of steps.
The vertical axis is the proportion of independent runs that achieve that return, using both a histogram with 30 bins and a fitted Gaussian kernel density estimator. In this example plot, we can see that QRC and Q-learning have similar distributions over returns across runs, QC-LL is peaked at a lower return and SBEED exhibits some bimodality, with several reasonable runs and a few runs with low return.
\begin{figure}[h]
    \centering
    \includegraphics[width=0.35\linewidth]{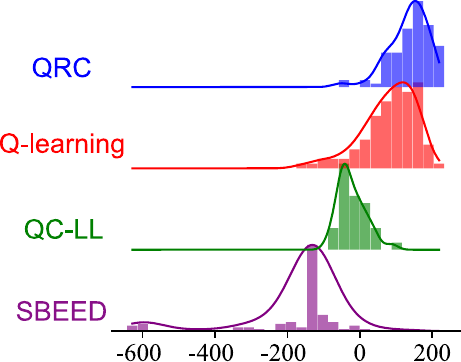}
    \caption{\label{fig:example-performance-dist}
        {\em The performance distribution over runs} for the best performing hyperparameter settings for each algorithm on Lunar Lander.
        The horizontal axis represents the average episodic return over the last 25\% of steps.
        The vertical axis for each subplot represents the proportion of trials that obtained a given level of performance.
        The plot shows the empirical histogram and kernel density estimator for the performance distribution over 100 independent trials. Mass concentrated to the right indicates better performance.
    }
\end{figure}

Because each algorithm prefers different stepsizes, we provide these distribution plots across stepsizes, optimizing the remaining hyperparameters for each stepsize. Figure~\ref{fig:bakeoff-run-dist} shows the distribution for every algorithm, across every domain and swept stepsize value.
For every algorithm and domain, the distribution with the highest mean is highlighted with a bold color; distributions for all other stepsizes are shown with a faded color.
QRC consistently has narrow performance distributions over runs.
The worse performing hyperparameter settings for QRC often exhibit narrow distributions as well.
In contrast, the other three algorithms demonstrate skewed performance distributions for their best hyperparameters, pulling the mean towards lower performance levels.
SBEED often exhibits bimodality in its performance distributions, especially on the Mountain Car domain.\footnote{Our results with SBEED may appear pessimistic compared to results in the literature. It is possible previously published results with SBEED were achieved through domain specific hyper-parameter tuning. Further, our results were averaged over 100 runs, whereas prior work used five. We do not have access to the code used in prior work, and so we can only speculate.}

\begin{figure}[t]
    \centering
    \includegraphics[width=\linewidth]{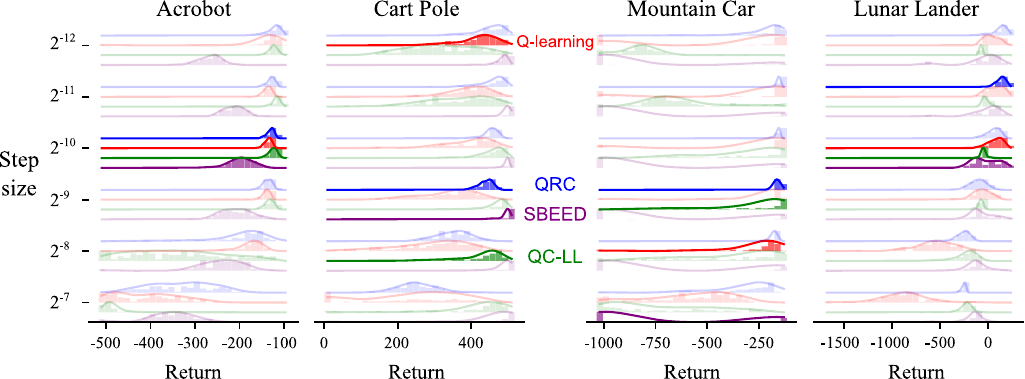}
    \caption{\label{fig:bakeoff-run-dist}
        {\em Sensitivity to stepsize parameter.} Distribution of the return per episode for the last 25\% of episodes across choice of stepsize. Each row of this figure correponds to the performance on each algorithm across domains for one value of the stepsize parameter. Each subplot is exactly like Figure \ref{fig:example-performance-dist}: the distribution of performance for all four algorithms using a particular stepsize parameter value on a single domain. The highlighted plots in each column represent the best performing stepsize parameter value.
        QRC consistently exhibits a narrow distribution of performance where the bulk of the distribution is on the upper end of the performance metric (towards the right is better).
        Q-learning and Nonlinear QC both have wide performance distributions on all domains and exhibit bimodal distributions on Mountain Car.
        SBEED tends to exhibit bimodal performance often, with a non-trivial proportion of runs which fail to learn beyond random performance.
    }
\end{figure}

\subsection{Gradient Correction Updates versus Saddlepoint Updates}\label{sec_gc_vs_sp}

\begin{figure}[t]
    \centering
    \includegraphics[width=\linewidth]{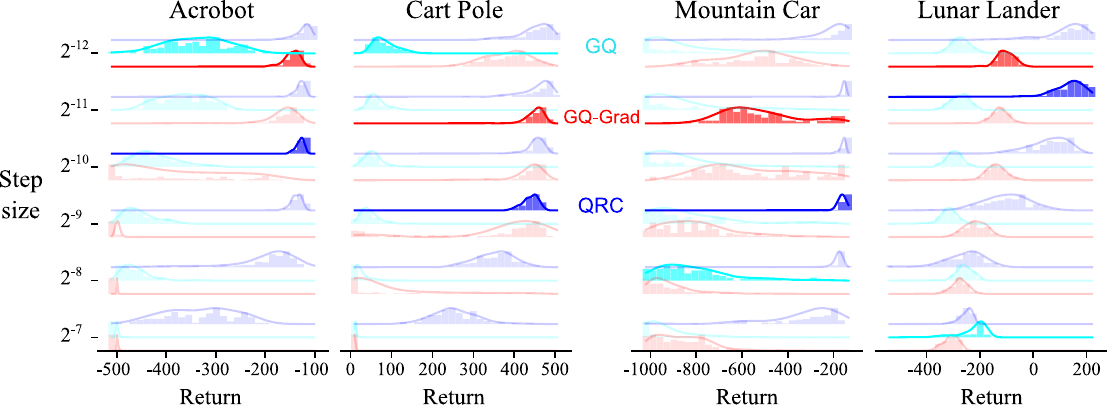}
    \caption{\label{fig:gc_vs_sp}
        {\em Comparing gradient correction-based updates (QRC) and saddlepoint methods (GQ, GQ-Grad).}
        GQ-Grad utilizes the gradient of $v$ as features for the secondary variable, $h$.
        Allowing the saddlepoint methods to estimate $h(s)$ by using a linear function of the gradients of the primary variable yields slightly higher performance. Nonetheless, saddlepoint methods suffer from wide performance distributions with the bulk of the distribution being further left than the gradient correction-based updates.
    }
\end{figure}

In this section we compare the saddlepoint and gradient correction forms of the gradient.
In Section~\ref{sec:estimating_gradient_gpbe}, we discussed these two strategies for estimating the gradient of the generalize $\mspbe$, and advocated for using the gradient correction form in QRC.\@
To ablate this choice of update strategy, we fix all other design decisions.
Figure~\ref{fig:gc_vs_sp} shows the saddlepoint method compared to the gradient correction method across our four domains.

The distribution plots suggest that the saddlepoint method \emph{can} learn reasonable policies in some runs, however, the performance distributions tend to be skewed right with a bulk of the distribution around poor performance. QRC is restricted to regularization parameter $\beta=1$;
to give the saddlepoint method a greater chance of success, we additionally sweep $\beta$ for it, giving it four times as many hyperparameter permutations.
Even with the increased hyperparameter search space, it is clear that the saddlepoint method under-performs the gradient correction-based methods.

\subsection{Using a Shared or Separate Basis for $h$}\label{sec_h_bases}

In this section we investigate the impact of using a shared network in QRC.\@
By sharing a network, we effectively restrict $\Hset = \Vset$; by using a different network, $\Hset$ is less restricted.
A shared network is simpler and, in section~\ref{sec:estimating_gradient_gpbe}, we show that defining $h$ to be a function of the gradients for $\paramv$ reduces the bias of gradient correction methods, such as QRC. However, in Section~\ref{sec:upper-bound-msve-supset} we show a less restricted $\Hset$ can improve performance. So, which is better?
We compare QRC with a shared network, QRC with a separate network for $h$ (QRC-Sep), and QRC with gradient features for $h$ (QRC-Grad) to understand the impact of these choices.

To help remove the confounding variable of number of learnable parameters, we swept over network architectures for the separate network version of QRC.\@
We kept the same general structure of the neural network by restricting the number of hidden units in both hidden layers to be the same, then swept over the number of units per layer: $n \in \{8, 16, 32, 64\}$ for the network for $h$ in Acrobot and Mountain Car and $n \in \{16, 32, 64, 128\}$ for Cart Pole and Lunar Lander.
At the smallest end of the range, QRC with separate networks has fewer learnable parameters than when using shared heads, and at the upper end of the range, it has considerably more learnable parameters.

Figure~\ref{fig:dual_basis} summarizes the results using separate networks for the dual.
Performance is comparable for each of the three method, with QRC-Sep performing marginally worse overall and QRC-Grad performing marginally better.
Notably the performance ordering appears to follow the bias of each method, with QRC-Sep having more bias due to needing to learn two separate networks, and QRC-Grad having slightly less bias due to $h$ being a function of the gradients of $\paramv$.
In general, the simpler choice of using a shared network appears to be as effective in these benchmark environments.

\begin{figure}[t]
    \centering
    \includegraphics[width=\linewidth]{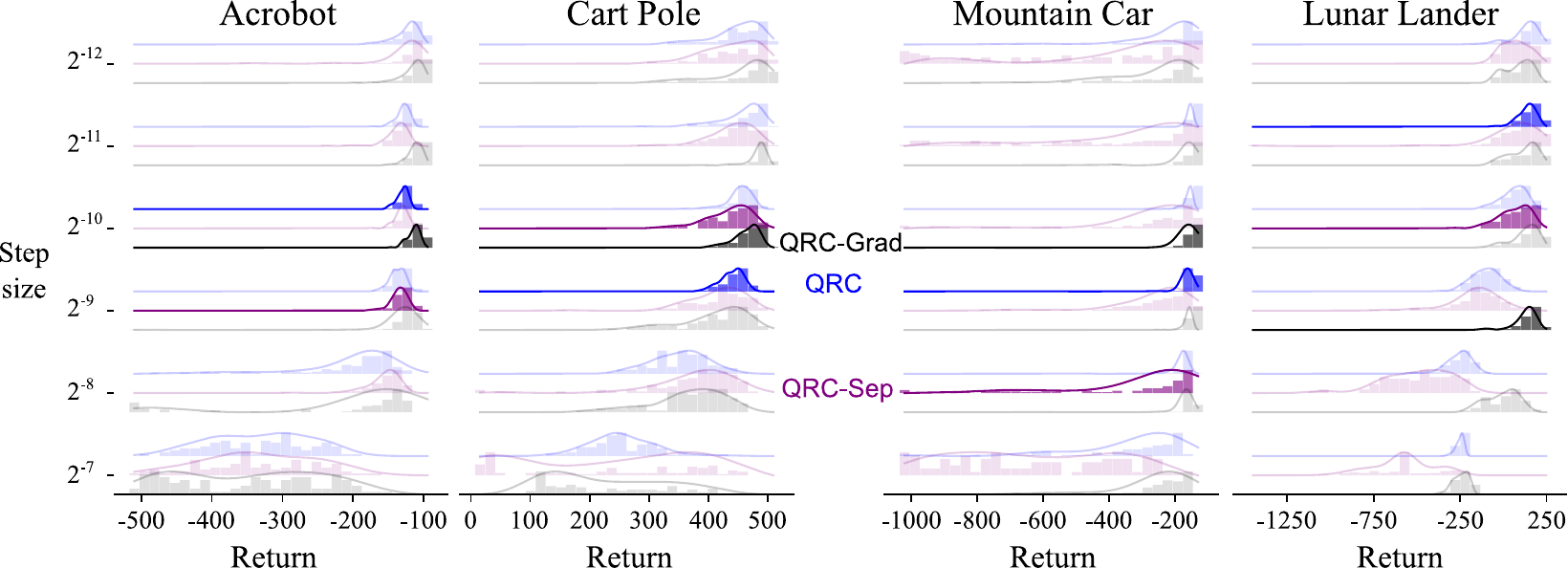}
    \caption{\label{fig:dual_basis}
        {\em How to represent $h$}: ablating the choice of basis function for the secondary variable, by comparing a shared network with two heads (QRC), two separate networks (QRC-Sep), and one network for the primary variable and a linear function of the primary variable's gradients for the secondary (QRC-Grad).
    }
\end{figure}

\subsection{Using Parameterized Policies}

In this section we attempt to better understand the poor performance of SBEED by investigating both action-value and actor-critic variants of the algorithms.
Parameterized policies add a significant source of complexity both to the agent design and to the learning dynamics of the problem, making it challenging to directly compare SBEED to QRC.\@
In this section we complete the square by introducing a parameterized policy version of TDRC, which can be viewed as an Actor-Critic method with a TDRC critic, and introduce an action-value version of SBEED.\@
We show that the action-value version of SBEED significantly outperforms its parameterized policy counterpart, but still under-performs QRC in all domains and TDRC-PG in most domains.

The actor-critic version of TDRC uses the following update
\begin{align*}
    \vecw_{t+1} &\leftarrow \vecw_t + \alpha \delta_t \nabla_\vecw v(S_t, \vecw_t) - \alpha \gamma \hparams_t\tr \vecx_t \nabla_\vecw v(S_{t+1}, \vecw_t) \\
    \hparams_{t+1} &\leftarrow \hparams_t + \alpha \left[ \delta_t - \hparams_t\tr \vecx_t \right] x_t - \alpha \beta \hparams_t \\
    \theta_{\pi,t+1} &\leftarrow \theta_{\pi,t} + \alpha \delta_t \nabla_{\theta_\pi} \ln \pi_{\theta_\pi} (A_t | S_t)
    ,
\end{align*}
which can be viewed as a standard actor-critic update with TDRC as the critic.
Our design choices for TDRC-PG mirror those of SBEED, for instance using a two-headed neural network for the value function estimator and a separate neural network for the policy gradient estimator to avoid the need to balance between loss functions.

\begin{figure}[t]
    \centering
    \includegraphics[width=\linewidth]{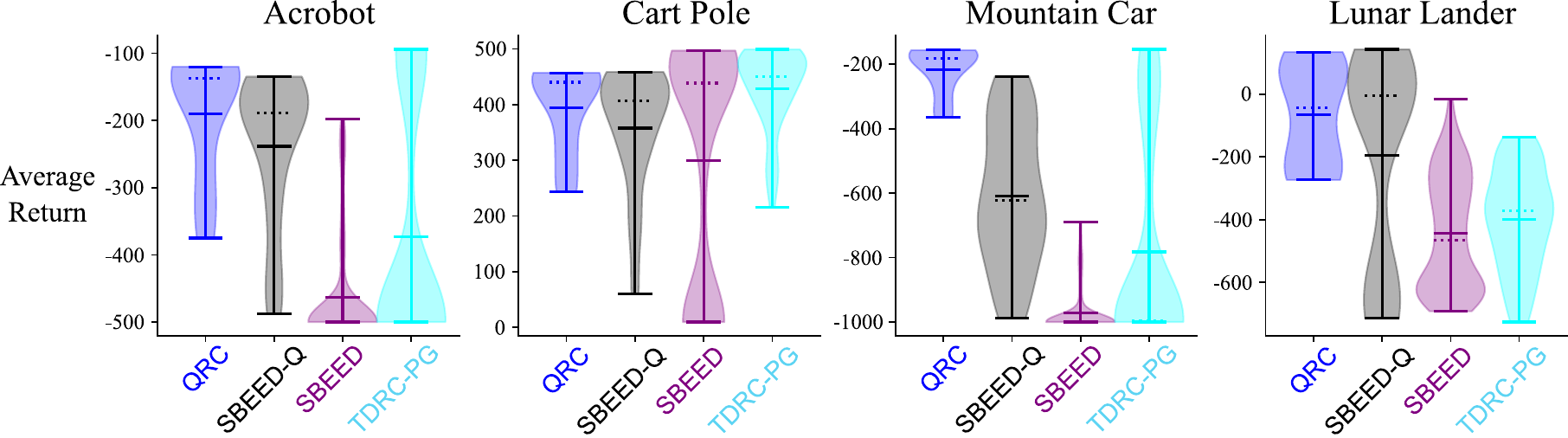}

    \caption{\label{fig:parameterized_policies}
        {\em Impact of policy parameterization: action values verses parameterized policies.} The plots above compare the QRC and SBEED update functions both for learning state-value functions with parameterized policies and action-value functions with mellowmax.
    }
\end{figure}

In Figure~\ref{fig:parameterized_policies} we investigate the distribution of performance for each choice of update cross-producted with each choice of action-value or parameterized policy.
The width of each violin plot represents the proportion of parameter settings whose area under the learning curve achieved the given performance level indicated on the vertical axis.
In all cases, QRC had the most narrow performance distribution centered around the highest or near-highest reward.
The parameterized-policy variants generally exhibit wider distributions and lower average AUC than their action-value counterparts, but occasionally achieve higher maximum performance for a single hyperparameter setting.
The SBEED update generally performed worse than the QRC update across all four domains.

The empirical superiority of QRC over SBEED in our experiments is partially, but not fully, explained by learning parameterized policies.
Because SBEED-Q outperforms SBEED overall, some of SBEED's poor performance is due to the parameterized policy.
SBEED-Q and QRC still perform differently in every domain, with SBEED-Q generally having much higher sensitivity to hyperparameters and occasionally achieving lower maximum performance than QRC, even after tuning over many more hyperparameter settings.

\section{Conclusion}
In this paper, we introduced the generalized $\mspbe$ for off-policy value estimation, bringing together insights developed for both the linear $\mspbe$ and the $\msbe$.
We investigated the quality of the $\mspbe$ solution, both theoretically by bounding the $\msve$ and empirically under different state weightings and for different projection operators.
We introduced an algorithm for minimizing the $\mspbe$ both in prediction and control, incorporating empirical insights from optimizing the linear $\mspbe$.
Finally, we showed empirically that this control algorithm is less sensitive to its hyperparameters and exhibits more stable performance compared to several baselines on four benchmark domains.

This work highlights several open questions for off-policy value estimation.
We discuss three key questions that follow-up work can focus on.

\textbf{Can the empirical benefits of gradient correction methods be combined with the theoretical convenience of saddlepoint methods?}
Our proposed algorithm uses gradient corrections because these methods typically outperform their saddlepoint counterparts in practice \citep{white2016investigating,ghiassian2018online,ghiassian2020gradient}.
In Section~\ref{sec:estimating_gradient_gpbe} we discuss how saddlepoint methods use a standard gradient update for min-max problems, and so import the resulting theoretical guarantees. The gradient correction update, on the other hand, is only asymptotically unbiased under certain conditions on the auxiliary variable $h$. Convergence is only guaranteed under linear function approximation; more needs to be understood about the theoretical properties of this update in the nonlinear setting.

\textbf{How do we use emphatic weightings for control?}
In policy evaluation our experiments show that emphatic weights consistently produce better solutions than either the alternative-life solution $d_\pi$ or the excursion solution $d_\mu$ (see Figure~\ref{fig:rw-fixed-points}).
In practice, incorporating emphatic weightings can yield high-variance \citep{sutton2018reinforcement}.
An interesting direction is to investigate the use of emphatic weightings in QRC.\@

\textbf{What function class should we use for $\Hset$?}
In this work, we concluded that using a shared network for the primary and auxiliary variables was reasonable effective for optimizing the $\mspbe$. This effectively restricts $\Hset = \Vset$, which mimics the choice for the linear $\mspbe$, even when we use neural network function approximation. However, restricting $\Hset$ can result in poor quality solutions, namely those where solutions under the $\msbe$ are high quality but not under the linear $\mspbe$. Much more work needs to be done to understand the impacts of different choices for $\Hset$.

\subsubsection*{Acknowledgments}
The authors gratefully acknowledge funding from
the Natural Sciences and Engineering Research Council of Canada, CIFAR and Google DeepMind. Thanks to Sina Ghiassian and Richard Sutton for helpful discussions on an earlier version of this work.

\bibliography{paper}

\begin{thebibliography}{61}
\providecommand{\natexlab}[1]{#1}
\providecommand{\url}[1]{\texttt{#1}}
\expandafter\ifx\csname urlstyle\endcsname\relax
  \providecommand{\doi}[1]{doi: #1}\else
  \providecommand{\doi}{doi: \begingroup \urlstyle{rm}\Url}\fi

\bibitem[Antos et~al.(2008)Antos, Szepesv{\'a}ri, and Munos]{antos2008learning}
Andr{\'a}s Antos, Csaba Szepesv{\'a}ri, and R{\'e}mi Munos.
\newblock Learning near-optimal policies with {{Bellman-residual}} minimization
  based fitted policy iteration and a single sample path.
\newblock \emph{Machine Learning}, 71\penalty0 (1):\penalty0 89--129, 2008.
\newblock \doi{10.1007/s10994-007-5038-2}.

\bibitem[Asadi and Littman(2017)]{asadi2017alternative}
Kavosh Asadi and Michael~L. Littman.
\newblock An {{Alternative Softmax Operator}} for {{Reinforcement Learning}}.
\newblock \emph{International Conference on Machine Learning}, 2017.

\bibitem[Baird(1995)]{baird1995residual}
Leemon Baird.
\newblock Residual {{Algorithms}}: {{Reinforcement Learning}} with {{Function
  Approximation}}.
\newblock \emph{Machine Learning Proceedings}, pages 30--37, 1995.
\newblock \doi{10.1016/B978-1-55860-377-6.50013-X}.

\bibitem[Barto et~al.(1983)Barto, Sutton, and Anderson]{barto1983neuronlike}
Andrew~G. Barto, Richard~S. Sutton, and Charles~W. Anderson.
\newblock Neuronlike adaptive elements that can solve difficult learning
  control problems.
\newblock \emph{IEEE Transactions on Systems, Man, and Cybernetics},
  SMC-13\penalty0 (5):\penalty0 834--846, 1983.
\newblock \doi{10.1109/TSMC.1983.6313077}.

\bibitem[Bertsekas and Tsitsiklis(1996)]{bertsekas1996neurodynamic}
Dimitri~P Bertsekas and J.N. Tsitsiklis.
\newblock \emph{Neuro-Dynamic {{Programming}}}.
\newblock {Athena Scientific}, 1996.

\bibitem[Bietti and Mairal(2019)]{bietti2019group}
Alberto Bietti and Julien Mairal.
\newblock Group invariance, stability to deformations, and complexity of deep
  convolutional representations.
\newblock \emph{The Journal of Machine Learning Research}, 20\penalty0
  (1):\penalty0 876--924, 2019.

\bibitem[Bottou et~al.(2013)Bottou, Peters, {Qui{\~n}onero-Candela}, Charles,
  Chickering, Portugaly, Ray, Simard, and Snelson]{bottou2013counterfactual}
L{\'e}on Bottou, Jonas Peters, Joaquin {Qui{\~n}onero-Candela}, Denis Charles,
  Max Chickering, Elon Portugaly, Dipankar Ray, Patrice Simard, and Ed~Snelson.
\newblock Counterfactual reasoning and learning systems: {{The}} example of
  computational advertising.
\newblock \emph{Journal of Machine Learning Research}, 2013.

\bibitem[Brockman et~al.(2016)Brockman, Cheung, Pettersson, Schneider,
  Schulman, Tang, and Zaremba]{brockman2016openai}
Greg Brockman, Vicki Cheung, Ludwig Pettersson, Jonas Schneider, John Schulman,
  Jie Tang, and Wojciech Zaremba.
\newblock Openai gym.
\newblock \emph{arXiv preprint arXiv:1606.01540}, 2016.

\bibitem[Dai et~al.(2017)Dai, He, Pan, Boots, and Song]{dai2017learning}
Bo~Dai, Niao He, Yunpeng Pan, Byron Boots, and Le~Song.
\newblock Learning from {{Conditional Distributions}} via {{Dual Embeddings}}.
\newblock \emph{International Conference on Artificial Intelligence and
  Statistics}, page~10, 2017.

\bibitem[Dai et~al.(2018)Dai, Shaw, Li, Xiao, He, Liu, Chen, and
  Song]{dai2018sbeed}
Bo~Dai, Albert Shaw, Lihong Li, Lin Xiao, Niao He, Zhen Liu, Jianshu Chen, and
  Le~Song.
\newblock {{SBEED}}: {{Convergent Reinforcement Learning}} with {{Nonlinear
  Function Approximation}}.
\newblock \emph{International Conference on Machine Learning}, page~10, 2018.

\bibitem[Dann et~al.(2014)Dann, Neumann, and Peters]{dann2014policy}
Christoph Dann, Gerhard Neumann, and Jan Peters.
\newblock Policy {{Evaluation}} with {{Temporal Differences}}: {{A Survey}} and
  {{Comparison}}.
\newblock \emph{Journal of Machine Learning Research}, page~75, 2014.

\bibitem[Du et~al.(2017)Du, Chen, Li, Xiao, and Zhou]{du2017stochastic}
Simon~S. Du, Jianshu Chen, Lihong Li, Lin Xiao, and Dengyong Zhou.
\newblock Stochastic {{Variance Reduction Methods}} for {{Policy Evaluation}}.
\newblock \emph{International Conference on Machine Learning}, 2017.

\bibitem[Dudek and Holly(1994)]{dudek1994nonlinear}
Ewa Dudek and Konstanty Holly.
\newblock Nonlinear orthogonal projection.
\newblock \emph{Annales Polonici Mathematici}, 59:\penalty0 1--31, 1994.

\bibitem[Espeholt et~al.(2018)Espeholt, Soyer, Munos, Simonyan, Mnih, Ward,
  Doron, Firoiu, Harley, Dunning, Legg, and Kavukcuoglu]{espeholt2018impala}
Lasse Espeholt, Hubert Soyer, Remi Munos, Karen Simonyan, Volodymir Mnih, Tom
  Ward, Yotam Doron, Vlad Firoiu, Tim Harley, Iain Dunning, Shane Legg, and
  Koray Kavukcuoglu.
\newblock {{IMPALA}}: {{Scalable Distributed Deep-RL}} with {{Importance
  Weighted Actor-Learner Architectures}}.
\newblock \emph{International Conference on Machine Learning}, 2018.

\bibitem[Feng et~al.(2019)Feng, Li, and Liu]{feng2019kernel}
Yihao Feng, Lihong Li, and Qiang Liu.
\newblock A {{Kernel Loss}} for {{Solving}} the {{Bellman Equation}}.
\newblock \emph{Advances in Neural Information Processing Systems 32}, pages
  15456--15467, 2019.

\bibitem[Ghiassian et~al.(2018)Ghiassian, Patterson, White, Sutton, and
  White]{ghiassian2018online}
Sina Ghiassian, Andrew Patterson, Martha White, Richard~S. Sutton, and Adam
  White.
\newblock Online {{Off-policy Prediction}}.
\newblock \emph{arXiv:1811.02597}, 2018.

\bibitem[Ghiassian et~al.(2020)Ghiassian, Patterson, Garg, Gupta, White, and
  White]{ghiassian2020gradient}
Sina Ghiassian, Andrew Patterson, Shivam Garg, Dhawal Gupta, Adam White, and
  Martha White.
\newblock Gradient {{Temporal-Difference Learning}} with {{Regularized
  Corrections}}.
\newblock \emph{International Conference on Machine Learning}, 2020.

\bibitem[Ghosh and Bellemare(2020)]{ghosh2020representations}
Dibya Ghosh and Marc~G. Bellemare.
\newblock Representations for {{Stable Off-Policy Reinforcement Learning}}.
\newblock \emph{International Conference on Machine Learning}, 2020.

\bibitem[Glorot and Bengio(2010)]{glorot2010understanding}
Xavier Glorot and Yoshua Bengio.
\newblock Understanding the difficulty of training deep feedforward neural
  networks.
\newblock \emph{International Conference on Artificial Intelligence and
  Statistics}, pages 249--256, 2010.

\bibitem[Hackman(2013)]{hackman2013faster}
Leah Hackman.
\newblock \emph{Faster {{Gradient-TD Algorithms}}}.
\newblock PhD thesis, University of Alberta, 2013.

\bibitem[Hallak and Mannor(2017)]{hallak2017consistent}
Assaf Hallak and Shie Mannor.
\newblock Consistent on-line off-policy evaluation.
\newblock \emph{International Conference on Machine Learning}, pages
  1372--1383, 2017.

\bibitem[Hallak et~al.(2016)Hallak, Tamar, Munos, and
  Mannor]{hallak2016generalized}
Assaf Hallak, Aviv Tamar, Remi Munos, and Shie Mannor.
\newblock Generalized {{Emphatic Temporal Difference Learning}}:
  {{Bias-Variance Analysis}}.
\newblock \emph{AAAI}, page~7, 2016.

\bibitem[Jaderberg et~al.(2016)Jaderberg, Mnih, Czarnecki, Schaul, Leibo,
  Silver, and Kavukcuoglu]{jaderberg2016reinforcement}
Max Jaderberg, Volodymyr Mnih, Wojciech~Marian Czarnecki, Tom Schaul, Joel~Z.
  Leibo, David Silver, and Koray Kavukcuoglu.
\newblock Reinforcement {{Learning}} with {{Unsupervised Auxiliary Tasks}}.
\newblock \emph{International Conference on Learning Representations}, 2016.

\bibitem[Kakade and Langford(2002)]{kakade2002approximately}
Sham Kakade and John Langford.
\newblock Approximately optimal approximate reinforcement learning.
\newblock \emph{International Conference on Machine Learning}, 2002.

\bibitem[Kingma and Ba(2015)]{kingma2015adam}
Diederik~P Kingma and Jimmy Ba.
\newblock Adam: {{A Method}} for {{Stochastic Optimization}}.
\newblock \emph{International Conference on Learning Representations}, page~15,
  2015.

\bibitem[Kolter(2011)]{kolter2011fixed}
J~Z Kolter.
\newblock The {{Fixed Points}} of {{Off-Policy TD}}.
\newblock \emph{Advances in Neural Information Processing Systems}, page~9,
  2011.

\bibitem[Littman and Sutton(2002)]{littman2002predictive}
Michael~L. Littman and Richard~S Sutton.
\newblock Predictive {{Representations}} of {{State}}.
\newblock \emph{Advances in Neural Information Processing Systems}, pages
  1555--1561, 2002.

\bibitem[Liu et~al.(2015)Liu, Liu, Ghavamzadeh, Mahadevan, and
  Petrik]{liu2015finitesample}
Bo~Liu, Ji~Liu, Mohammad Ghavamzadeh, Sridhar Mahadevan, and Marek Petrik.
\newblock Finite-{{Sample Analysis}} of {{Proximal Gradient TD Algorithms}}.
\newblock \emph{International Conference on Uncertainty in Artificial
  Intelligence}, 2015.

\bibitem[Liu et~al.(2016)Liu, Liu, Ghavamzadeh, Mahadevan, and
  Petrik]{liu2016proximal}
Bo~Liu, Ji~Liu, Mohammad Ghavamzadeh, Sridhar Mahadevan, and Marek Petrik.
\newblock Proximal {{Gradient Temporal Difference Learning Algorithms}}.
\newblock \emph{International Joint Conference on Artificial Intelligence},
  page~5, 2016.

\bibitem[Liu et~al.(2020)Liu, Swaminathan, Agarwal, and
  Brunskill]{liu2020offpolicy}
Yao Liu, Adith Swaminathan, Alekh Agarwal, and Emma Brunskill.
\newblock Off-{{Policy Policy Gradient}} with {{Stationary Distribution
  Correction}}.
\newblock \emph{Uncertainty in Artificial Intelligence}, pages 1180--1190,
  2020.

\bibitem[Maei(2011)]{maei2011gradient}
Hamid~Reza Maei.
\newblock \emph{Gradient {{Temporal-Difference Learning Algorithms}}}.
\newblock PhD thesis, University of Alberta, {Edmonton}, 2011.

\bibitem[Maei and Sutton(2010)]{maei2010gq}
Hamid~Reza Maei and Richard~S. Sutton.
\newblock {{GQ}} (lambda): {{A}} general gradient algorithm for
  temporal-difference prediction learning with eligibility traces.
\newblock \emph{Conference on Artificial General Intelligence}, 2010.

\bibitem[Maei et~al.(2009)Maei, Szepesv{\'a}ri, Bhatnagar, Precup, Silver, and
  Sutton]{maei2009convergent}
Hamid~Reza Maei, Csaba Szepesv{\'a}ri, Shalabh Bhatnagar, Doina Precup, David
  Silver, and Richard~S Sutton.
\newblock Convergent {{Temporal-Difference Learning}} with {{Arbitrary Smooth
  Function Approximation}}.
\newblock \emph{Advances in Neural Information Processing Systems}, pages
  1204--1212, 2009.

\bibitem[Mahadevan et~al.(2014)Mahadevan, Liu, Thomas, Dabney, Giguere, Jacek,
  Gemp, and Liu]{mahadevan2014proximal}
Sridhar Mahadevan, Bo~Liu, Philip Thomas, Will Dabney, Steve Giguere, Nicholas
  Jacek, Ian Gemp, and Ji~Liu.
\newblock Proximal {{Reinforcement Learning}}: {{A New Theory}} of {{Sequential
  Decision Making}} in {{Primal-Dual Spaces}}.
\newblock \emph{arXiv:1405.6757}, 2014.

\bibitem[Mahmood et~al.(2017)Mahmood, Yu, and Sutton]{mahmood2017multistep}
Ashique~Rupam Mahmood, Huizhen Yu, and Richard~S. Sutton.
\newblock Multi-step {{Off-policy Learning Without Importance Sampling
  Ratios}}.
\newblock \emph{arXiv:1702.03006}, 2017.

\bibitem[Maillard et~al.(2010)Maillard, Munos, Lazaric, and
  Ghavamzadeh]{maillard2010finitesample}
Odalric-Ambrym Maillard, Remi Munos, Alessandro Lazaric, and Mohammad
  Ghavamzadeh.
\newblock Finite-{{Sample Analysis}} of {{Bellman Residual Minimization}}.
\newblock \emph{Asian Conference on Machine Learning}, page~16, 2010.

\bibitem[Moore(1990)]{moore1990efficient}
Andrew~William Moore.
\newblock \emph{Efficient Memory-Based Learning for Robot Control}.
\newblock PhD thesis, University of Cambridge, 1990.

\bibitem[Munos et~al.(2016)Munos, Stepleton, Harutyunyan, and
  Bellemare]{munos2016safe}
Remi Munos, Tom Stepleton, Anna Harutyunyan, and Marc Bellemare.
\newblock Safe and {{Efficient Off-Policy Reinforcement Learning}}.
\newblock \emph{Advances in Neural Information Processing Systems}, page~9,
  2016.

\bibitem[Precup et~al.(2000)Precup, Sutton, and Singh]{precup2000eligibility}
Doina Precup, Richard~S Sutton, and Satinder Singh.
\newblock Eligibility {{Traces}} for {{Off-Policy Policy Evaluation}}.
\newblock \emph{International Conference on Machine Learning}, page~9, 2000.

\bibitem[Precup et~al.(2001)Precup, Sutton, and Dasgupta]{precup2001offpolicy}
Doina Precup, Richard~S Sutton, and Sanjoy Dasgupta.
\newblock Off-{{Policy Temporal-Difference Learning}} with {{Function
  Approximation}}.
\newblock \emph{International Conference on Machine Learning}, page~9, 2001.

\bibitem[Scherrer(2010)]{scherrer2010should}
Bruno Scherrer.
\newblock Should one compute the {{Temporal Difference}} fix point or minimize
  the {{Bellman Residual}}? {{The}} unified oblique projection view.
\newblock \emph{International Conference on Machine Learning}, 2010.

\bibitem[Schoknecht(2003)]{schoknecht2003optimality}
Ralf Schoknecht.
\newblock Optimality of {{Reinforcement Learning Algorithms}} with {{Linear
  Function Approximation}}.
\newblock \emph{Advances in Neural Information Processing Systems}, page~8,
  2003.

\bibitem[Sutton(1996)]{sutton1996generalization}
Richard~S. Sutton.
\newblock Generalization in reinforcement learning: {{Successful}} examples
  using sparse coarse coding.
\newblock \emph{Advances in Neural Information Processing Systems}, pages
  1038--1044, 1996.

\bibitem[Sutton and Barto(2018)]{sutton2018reinforcement}
Richard~S. Sutton and Andrew~G. Barto.
\newblock \emph{Reinforcement {{Learning}}: {{An Introduction}}}.
\newblock {MIT Press}, November 2018.
\newblock ISBN 978-0-262-35270-3.

\bibitem[Sutton et~al.(1999{\natexlab{a}})Sutton, McAllester, Singh, and
  Mansour]{sutton1999policy}
Richard~S. Sutton, David~A. McAllester, Satinder~P. Singh, and Yishay Mansour.
\newblock Policy gradient methods for reinforcement learning with function
  approximation.
\newblock \emph{Advances in Neural Information Processing Systems},
  99:\penalty0 1057--1063, 1999{\natexlab{a}}.

\bibitem[Sutton et~al.(1999{\natexlab{b}})Sutton, Precup, and
  Singh]{sutton1999mdps}
Richard~S. Sutton, Doina Precup, and Satinder Singh.
\newblock Between {{MDPs}} and semi-{{MDPs}}: {{A}} framework for temporal
  abstraction in reinforcement learning.
\newblock \emph{Artificial Intelligence}, 112\penalty0 (1):\penalty0 181--211,
  1999{\natexlab{b}}.
\newblock \doi{10.1016/S0004-3702(99)00052-1}.

\bibitem[Sutton et~al.(2009)Sutton, Maei, Precup, Bhatnagar, Silver,
  Szepesv{\'a}ri, and Wiewiora]{sutton2009fast}
Richard~S. Sutton, Hamid~Reza Maei, Doina Precup, Shalabh Bhatnagar, David
  Silver, Csaba Szepesv{\'a}ri, and Eric Wiewiora.
\newblock Fast gradient-descent methods for temporal-difference learning with
  linear function approximation.
\newblock \emph{International Conference on Machine Learning}, pages 1--8,
  2009.
\newblock \doi{10.1145/1553374.1553501}.

\bibitem[Sutton et~al.(2011)Sutton, Modayil, Delp, Degris, Pilarski, White, and
  Precup]{sutton2011horde}
Richard~S Sutton, Joseph Modayil, Michael Delp, Thomas Degris, Patrick~M
  Pilarski, Adam White, and Doina Precup.
\newblock Horde: {{A Scalable Real-time Architecture}} for {{Learning
  Knowledge}} from {{Unsupervised Sensorimotor Interaction}}.
\newblock \emph{International Conference on Autonomous Agents and Multi-Agent
  Systems}, page~8, 2011.

\bibitem[Sutton et~al.(2016)Sutton, Mahmood, and White]{sutton2016emphatic}
Richard~S Sutton, A~Rupam Mahmood, and Martha White.
\newblock An {{Emphatic Approach}} to the {{Problem}} of {{Off-policy
  Temporal-Difference Learning}}.
\newblock \emph{Journal of Machine Learning Research}, page~29, 2016.

\bibitem[Szyld(2006)]{szyld2006many}
Daniel~B. Szyld.
\newblock The many proofs of an identity on the norm of oblique projections.
\newblock \emph{Numerical Algorithms}, 42\penalty0 (3-4):\penalty0 309--323,
  2006.

\bibitem[Tanner and Sutton(2005)]{tanner2005td}
Brian Tanner and Richard~S. Sutton.
\newblock {{TD}}({$\lambda$}) networks: Temporal-difference networks with
  eligibility traces.
\newblock \emph{International Conference on Machine Learning}, pages 888--895,
  2005.
\newblock \doi{10.1145/1102351.1102463}.

\bibitem[Touati et~al.(2018)Touati, Bacon, Precup, and
  Vincent]{touati2018convergent}
Ahmed Touati, Pierre-Luc Bacon, Doina Precup, and Pascal Vincent.
\newblock Convergent {{Tree Backup}} and {{Retrace}} with {{Function
  Approximation}}.
\newblock \emph{International Conference on Machine Learning}, page~10, 2018.

\bibitem[Tsitsiklis and Van~Roy(1997)]{tsitsiklis1997analysis}
J.N. Tsitsiklis and B.~Van~Roy.
\newblock An analysis of temporal-difference learning with function
  approximation.
\newblock \emph{IEEE Transactions on Automatic Control}, 42\penalty0
  (5):\penalty0 674--690, 1997.
\newblock \doi{10.1109/9.580874}.

\bibitem[Wang et~al.(2016)Wang, Schaul, Hessel, {van Hasselt}, Lanctot, and {de
  Freitas}]{wang2016dueling}
Ziyu Wang, Tom Schaul, Matteo Hessel, Hado {van Hasselt}, Marc Lanctot, and
  Nando {de Freitas}.
\newblock Dueling {{Network Architectures}} for {{Deep Reinforcement
  Learning}}.
\newblock \emph{International Conference on Machine Learning}, page~9, 2016.

\bibitem[Watkins(1989)]{watkins1989learning}
Christopher John Cornish~Hellaby Watkins.
\newblock \emph{Learning from Delayed Rewards}.
\newblock PhD thesis, King's College, Cambridge, 1989.

\bibitem[White(2015)]{white2015developing}
Adam White.
\newblock \emph{Developing a {{Predictive Approach}} to {{Knowledge}}}.
\newblock PhD thesis, University of Alberta, 2015.

\bibitem[White and White(2016)]{white2016investigating}
Adam White and Martha White.
\newblock Investigating practical linear temporal difference learning.
\newblock \emph{International Conference on Autonomous Agents and Multi-Agent
  Systems}, 2016.

\bibitem[White(2017)]{white2017unifying}
Martha White.
\newblock Unifying {{Task Specification}} in {{Reinforcement Learning}}.
\newblock \emph{International Conference on Machine Learning}, page~9, 2017.

\bibitem[Williams and Baird(1993)]{williams1993tighta}
Ronald~J. Williams and Leemon~C. Baird.
\newblock Tight performance bounds on greedy policies based on imperfect value
  functions.
\newblock Technical report, {Citeseer}, 1993.

\bibitem[Yu(2015)]{yu2015convergence}
Huizhen Yu.
\newblock On {{Convergence}} of {{Emphatic Temporal-Difference Learning}}.
\newblock \emph{Conference on Learning Theory}, page~28, 2015.

\bibitem[Yu and Bertsekas(2010)]{yu2010error}
Huizhen Yu and Dimitri~P. Bertsekas.
\newblock Error bounds for approximations from projected linear equations.
\newblock \emph{Mathematics of Operations Research}, 35\penalty0 (2):\penalty0
  306--329, 2010.

\end{thebibliography}

\newpage
\appendix

\section{An Example of Posterior and Prior Corrections}\label{app_example_prior}

In this section we go through an explicit example, to highlight the role of prior and posterior corrections. For ease of reference, recall that Off-policy TD($\lambda$) that only has posterior corrections uses update
\begin{align}
\vecw_{t+1} \leftarrow& ~\vecw_t + \alpha \delta_t \vecz^\rho_t \nonumber \\
\vecz^\rho_t \leftarrow& ~\rho_t(\gamma_t \lambda \vecz^\rho_{t-1}  +\vecx_t) \tag{\ref{eq:offTD}}
,
\end{align}
and Alternative-life TD($\lambda$) incorporates both prior and posterior corrections
\begin{align}
\vecw_{t+1} \leftarrow& ~\vecw_t + \alpha \delta_t \vecz^\rho_t \nonumber \\
\vecz^\rho_t \leftarrow& ~\rho_t \left(\gamma_t \lambda \vecz_{t-1}  +\prod_{k=1}^{t-1} \rho_k \vecx_t \right) \tag{\ref{eq:precup2001}}
.
\end{align}

\begin{figure}[H]
    \centering
    \includegraphics[width=0.4\linewidth]{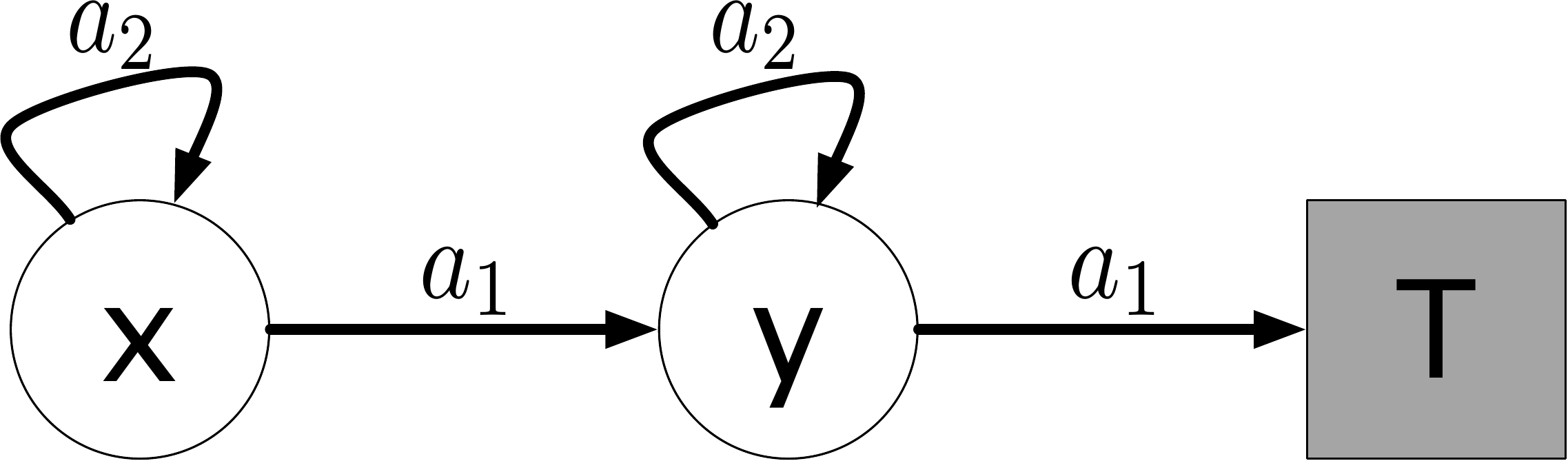}
    \caption{\label{fig:SimpleMDP}
        A simple MDP to understand the differences between prior corrections and posterior corrections in off-policy TD algorithms with importance sampling.
    }
\end{figure}

Consider the MDP depicted in Figure~\ref{fig:SimpleMDP}. Each episode starts in the leftmost state denoted `x' and terminates on transition into the terminal state denoted with `T', and each state is  represented with a unique tabular state encoding: x$: [1,0]$, y$: [0,1]$.
In each state there are two possible actions and the behavior policy chooses each action in each state with 0.5 probability. The target policy chooses action $a_1$ in all states. A posterior correction method like Off-policy TD($\lambda$), will always update the value of a state if action $a_1$ is taken. For example if the agent experiences the transition $y\rightarrow T$, Off-policy TD($\lambda$) will update the value of state $y$; no matter the history of interaction before entering state $y$.

Although the importance sampling corrections product in the eligibility trace update, Off-policy TD($\lambda$) does not use importance sampling corrections computed from prior time-steps to update the value of the current state. This is easy to see with an example. For simplicity we assume $\gamma_t$ is 1 everywhere, except upon termination where it is zero.
Let's examine the updates and trace contents for a trajectory where $b$'s action choices perfectly agree with $\pi$:
\[
    x\rightarrow y\rightarrow T.
\]
After the transition from $x\rightarrow y$, Off-policy TD($\lambda$) will update the value estimate corresponding to $x$:
\begin{align*}
\begin{bmatrix}
           \vhatt{x}{1} \\
           \vhatt{y}{1} \\
         \end{bmatrix} \leftarrow
\begin{bmatrix}
           0 \\
           0 \\
         \end{bmatrix}
+ \alpha\delta_1 \vecz^\rho_1 = \alpha \delta_1
\begin{bmatrix}
           \frac{\pi(a_1|x)}{b(a_1|x)}\lambda \\
           0 \\
         \end{bmatrix}
,
\end{align*}
where $\hat v_1(x)$ denotes the estimated value of state $x$ on time step $t=1$ (after the first transition), and as usual $\vecz^\rho_0$ and $\hat v$ are initialized to zero. After the second transition, $y\rightarrow T$, the importance sampling corrections will product in the trace, and the value estimates corresponding to both $x$ and $y$ are updated:
\begin{align*}
\begin{bmatrix}
           \vhatt{x}{2} \\
           \vhatt{y}{2} \\
         \end{bmatrix}
\leftarrow
\begin{bmatrix}
           \vhatt{x}{1} \\
           \vhatt{y}{1} \\
         \end{bmatrix}
+ \alpha\delta_2
\begin{bmatrix}
           \frac{\pi(a_1|y)}{b(a_1|y)}\frac{\pi(a_1|x)}{b(a_1|x)}\lambda^2 \\
            \frac{\pi(a_1|y)}{b(a_1|y)}\lambda \\
         \end{bmatrix}
.
\end{align*}
The estimated value of state $y$ is only updated with importance sampling corrections computed from state transitions that occur after the visit to $y$: using $\frac{\pi(a_1|y)}{b(a_1|y)}$, but not $\frac{\pi(a_1|x)}{b(a_1|x)}$.

Finally, consider another trajectory that deviates from the target policy's choice on the second step of the trajectory:  \[
    x\rightarrow y \rightarrow y\rightarrow T.
\]
On the first transition the value of $x$ is updated as expected, and no update occurs as a result of the second transition. On the third, transition the estimated value of state $x$ is \emph{not} updated; which is easy to see from inspecting the eligibility trace on each time-step:
\begin{align*}
\vecz^\rho_1 = \begin{bmatrix}
            \frac{\pi(a_1|x)}{b(a_1|x)}\lambda  \\
           0
         \end{bmatrix};~
\vecz^\rho_2 = \vec 0;~
\vecz^\rho_3 = \begin{bmatrix}
            0  \\
           \frac{\pi(a_1|y)}{b(a_1|y)}\lambda
         \end{bmatrix}.
\end{align*}
The eligibility trace is set to zero on time step two, because the target policy never chooses action $a_2$ in state $y$ and thus $\frac{\pi(a_2|y)}{b(a_2|y)} = 0$.
The value of state $S_t$ is never updated using importance sampling corrections computed on time steps prior to $t$.

Now consider the updates performed by Alternative-life TD($\lambda$) using different trajectories from our simple MDP (Figure~\ref{fig:SimpleMDP}). If the agent ever selects action $a_2$, then none of the following transitions will result in further updates to the value function. For example, the trajectory $x\rightarrow y\rightarrow y \rightarrow y \cdots y \rightarrow T$ will update $\paramv[s]$ corresponding to the first $x \rightarrow y$ transition, but $\paramv[y]$ would never be updated due to the product in Equation~\ref{eq:precup2001}. In contrast, the Off-policy TD($\lambda$) algorithm described in Equation~\ref{eq:offTD} would update $\paramv[s]$ on the first transition, and also update $\paramv[y]$ on the last transition of the trajectory.

\section{The Bias of the $\mstde$}\label{sec_mstdebad}

In this section, we highlight that the $\mstde$ can impose significant bias on the value function solution, to reduce variance of the targets in the TD-error. This could have utility for reducing variance of updates in practice. But, when considering the optimal solution for the $\mstde$---as opposed to the optimization path to get there---it suggests that the $\mstde$ is a poor choice of the objective.

We can similarly write the $\mstde$ as a decomposition, in terms of both the $\msbe$ and the $\mspbe$.
The $\mstde$ decomposes into the $\msbe$ and a bias term due to correlation between samples
\begin{align*}
&\CEpi{ \left(R + \gamma v(S')- V(s)\right)^2}{S=s} \\
&= \CEpi{ \left(R + \gamma v(S') - \CEpi{ R + \gamma v(S')}{S=s}  +  \CEpi{ R + \gamma v(S')}{S=s} - V(s)\right)^2}{S=s} \\
&= \CEpi{ \left(R + \gamma v(S') - \CEpi{ R + \gamma v(S')}{S=s}\right)^2}{S=s}  +  \left(\CEpi{ R + \gamma v(S')}{S=s} - V(s)\right)^2
\end{align*}
giving
\begin{align}
    \mstde = \| v - \Bo v \|^2_D + \E{\Var[R + \gamma v(S') | S=s]}
\end{align}
This further yields an equality between the $\mspbe$ and the $\mstde$
\begin{align}
    \mstde = \| v - \Pi_D \Bo v \|^2_D + \| \Bo v - \Pi_D \Bo v \|^2_D + \E{\Var[R + \gamma v(S') | S=s]}
\end{align}

This variance penalty encourages finding value functions that minimize the variance of the target. Notice that this connection exists in supervised learning as well, by simply considering the case where $\gamma = 0$. The $\msbe$ include terms $\left(\CEpi{R}{S=s} - v(s)\right)^2$ and the $\mstde$ is the standard squared error $\CEpi{ \left(R - v(s)\right)^2}{S=s}$.  In regression, this additional variance penalty has no impact on the optimal solution, because it does not include $v$:  $\E{\Var[R + 0\cdot v(S') | S=s]} = \E{\Var[R | S=s]}$.

Now the question is if it is useful to learn $v$ that results in a lower-variance target. One benefit is that the estimator itself could have lower variance, and so for a smaller number of samples this could warrant the additional bias in $v$. When comparing these objectives in their ideal forms, over all states under function approximation, it is an undesirable penalty on the value estimates. For instance, even in a 5-state random walk with tabular features, the fixed-point of $\mstde$ has non-negligible value error ($\msve \approx 0.26$) while the fixed-point of any of the other objectives considered in this work would have zero error.

\section{The Connection to the Saddlepoint Form of the Linear $\mspbe$}\label{app_saddlepoint}

In this section we explicitly show the steps that the generalized $\mspbe$ with $\Hset = \Vset$ under linear function approximation produces the linear $\mspbe$.
\begin{align*}
&\max_{h \in \mathcal{L}}  \sum_{s\in\S} d(s) \left(2\CEpi{\delta(\vecw)}{S=s}h(s) - h(s)^2 \right)\\
&= \sum_{s\in\S} d(s) 2\CEpi{\delta(\vecw)}{S=s}\vecx(s)^\top \hparams^* - \sum_{s\in\S} d(s) (\vecx(s)^\top \hparams^*)^2\\
&= \left(\sum_{s\in\S} d(s) 2\CEpi{\delta(\vecw)}{S=s}\vecx(s)^\top\right) \hparams^* - \sum_{s\in\S} d(s) (\hparams^*)^\top \vecx(s) \vecx(s)^\top \hparams^*\\
&= 2\E{\delta(\vecw)\vecx(s)}^\top \hparams^* - (\hparams^*)^\top \left(\sum_{s\in\S} d(s) \vecx(s) \vecx(s)^\top\right) \hparams^*\\
&= 2(\vecb - \Amat \vecw) \hparams^* - (\hparams^*)^\top \Cmat \hparams^*\\
&= 2\| \vecb - \Amat \vecw \|_{\Cmat^\inv}^2 - (\hparams^*)^\top (\vecb - \Amat \vecw) \\
&= 2\| \vecb - \Amat \vecw \|_{\Cmat^\inv}^2 - \| \vecb - \Amat \vecw \|_{\Cmat^\inv}^2 \\
&= \| \vecb - \Amat \vecw \|_{\Cmat^\inv}^2
\end{align*}

\section{Proofs for Section \ref{sec_theory_ve}}\label{app_proofs}

This section includes proofs for the theorems in the main body.

\noindent
\textbf{Proposition \ref{prop_statements}}
The following statements are true about $\ctransition$.
\begin{enumerate}
\item If $\dsol = d_\pi$ or $\dsol = m$, then $\ctransition < 1$.
\item If $\dsol = d_\pi$ and the discount is a constant $\gamma_c < 1$ (a continuing problem), then $\ctransition  = \gamma_c$.
\item If $\dsol = d_\pi$ and for some constant $\gamma_c$, $\gamma(s,a,s') \le \gamma_c$ for all $(s,a,s')$, then $\ctransition  \le \gamma_c$.
\item If $d_\pi(s) > 0$ iff $\dsol(s) > 0$, then $\ctransition \le s_{d_\pi} \sqrt{\kappa(\dsol,d_\pi)\kappa(d_\pi,\dsol)}$.
\end{enumerate}
\begin{proof}
\textbf{Case 1} is proven in \citep[Theorem 1]{white2017unifying}.

\textbf{Case 2} is standard result, because $\Ppigamma = \gamma_c P_\pi$. Therefore
 $\| \Ppigamma \|_{d_\pi} = \gamma_c \| P_\pi \|_{d_\pi} = \gamma_c$
because $\| P_\pi \|_{d_\pi} = 1$.

\textbf{Case 3} is a simple extension of the above.
\begin{align*}
    \Ppigamma(s, s') = \sum_a \pi(a | s) \gamma(s, a, s') P(s' | s, a)
    = \gamma_c \sum_a \pi(a | s) \gamma(s, a, s')/\gamma_c P(s' | s, a)
    = \gamma_c \Ppigammatilde
\end{align*}
with $\Ppigammatilde \defeq \sum_a \pi(a | s) \gamma(s, a, s')/\gamma_c P(s' | s, a)$. Because $\gamma(s, a, s')/\gamma_c \le 1$, we know that $\Ppigammatilde$ is a substochastic matrix and so that $\| \Ppigammatilde \|_{d_\pi} \le 1$. Therefore,  $\| \Ppigamma \|_{d_\pi} = \| \gamma_c \Ppigammatilde \|_{d_\pi} \le \gamma_c$.

\textbf{Case 4} is an extension of the strategy used in \citet[Theorem 2]{kolter2011fixed}.

\begin{align*}
    \| \Ppigamma \|_{\dsol}
    &= \| D^\powhalf \Ppigamma D^\invhalf \|_{2}\\
    &= \| D^\powhalf D_\pi^\invhalf D_\pi^\powhalf \Ppigamma D_\pi^\invhalf D_\pi^\powhalf D^\invhalf \|_{2}\\
    &\le \| D^\powhalf D_\pi^\invhalf \|_2 \| D_\pi^\powhalf \Ppigamma D_\pi^\invhalf \|_2 \| D_\pi^\powhalf D^\invhalf \|_{2}\\
    &= \max_{s \in \States} \sqrt{\dsol(s)/d_\pi(s)} \| \Ppigamma \|_{d_\pi} \max_{s \in \States} \sqrt{d_\pi(s)/\dsol(s)}\\
    &\le \| \Ppigamma \|_{d_\pi} \sqrt{\max_{s \in \States} \dsol(s)/d_\pi(s)} \sqrt{\max_{s \in \States} d_\pi(s)/\dsol(s)}\\
    &\le s_{d_\pi} \sqrt{\kappa(\dsol,d_\pi) \kappa(d_\pi,\dsol) }
    .
\end{align*}
\proofspace
\end{proof}

\noindent
\textbf{Theorem~\ref{thm_oblique}}
If $\Hset \supseteq \Vset$, then the solution $\vsol$ to the generalized $\mspbe$ satisfies
\begin{equation}
\| v_\pi - \vsol \|_{\dsol} \le \| \Pi_{\Vset,\ObMat}  \|_\dsol \| v_\pi -  \Pi_{\Vset, \dsol} v_\pi \|_{\dsol}
.
\end{equation}
\begin{proof}
If $\Pi_{\Vset,\ObMat}$ is the identity---a trivial projection---then the result immediately follows. This is because this implies $v_\pi\in\Vset$, and so both sides of the equation are zero.

Otherwise, assume $\Pi_{\Vset,\ObMat} $ is a non-trivial projection. Notice that $\Pi_{\Vset,\ObMat} \Pi_{\Vset, \dsol} = \Pi_{\Vset, \dsol}$, because the first projection already projects to the set $\Vset$ and so applying $\Pi_{\Vset,\ObMat}$ has no effect.
Now for any $v$,
\begin{align*}
(\eye - \Pi_{\Vset,\ObMat})(\eye -  \Pi_{\Vset, \dsol}) v
&= (\eye - \Pi_{\Vset,\ObMat} -  \Pi_{\Vset, \dsol} + \Pi_{\Vset,\ObMat} \Pi_{\Vset, \dsol}) v\\
&= (\eye - \Pi_{\Vset,\ObMat} -  \Pi_{\Vset, \dsol} + \Pi_{\Vset, \dsol}) v\\
&= (\eye - \Pi_{\Vset,\ObMat}) v
\end{align*}
  because by assumption $\Pi_{\Vset,\ObMat} \Pi_{\Vset, \dsol} = \Pi_{\Vset, \dsol}$.
  Therefore we get that
\begin{align*}
\| v_\pi - v \|_{\dsol}
&= \| v_\pi - \Pi_{\Vset,\ObMat}v_\pi \|_{\dsol} \\
&= \| (\eye - \Pi_{\Vset,\ObMat}) v_\pi \|_{\dsol} \\
&= \| (\eye - \Pi_{\Vset,\ObMat})(\eye -  \Pi_{\Vset, \dsol}) v_\pi \|_{\dsol} \\
&\le \| \eye - \Pi_{\Vset,\ObMat} \|_\dsol \| (\eye -  \Pi_{\Vset, \dsol}) v_\pi \|_{\dsol} \\
&= \| \eye - \Pi_{\Vset,\ObMat} \|_\dsol \| v_\pi -  \Pi_{\Vset, \dsol} v_\pi \|_{\dsol} \\
&= \| \Pi_{\Vset,\ObMat} \|_\dsol \| v_\pi -  \Pi_{\Vset, \dsol} v_\pi \|_{\dsol}
\end{align*}
where the last step follows from the fact that for a non-trivial projection operator, $\| \Pi_{\Vset,\ObMat} \|_\dsol = \|\eye - \Pi_{\Vset,\ObMat} \|_\dsol$ \citep[Theorem 2.3]{szyld2006many}.
\proofspace
\end{proof}

\noindent
\textbf{Corollary~\ref{cor_oblique_sol}}
Assume $\Vset$ is the space of linear functions with features $\vecx$ and $\Hset$ the space of linear functions with features $\phivec$, with $\Vset \subseteq \Hset$. Then the solution to the $\mspbe$ is
\begin{equation}
\vsol = \Xmat \wsol \ \ \ \text{ for } \wsol = (M^\top  (\eye - \Ppigamma) \Xmat)^\inv M^\top r_\pi
\end{equation}
for $M \defeq \Pi_{\Hset, \dsol}^\top \Dsmat (\eye - \Ppigamma) \Xmat$.
Further, $\vsol = \Pi_{\Vset,\ObMat} v_\pi$ for
\begin{equation}
\Pi_{\Vset,\ObMat} = \Xmat (M^\top (\eye - \Ppigamma) \Xmat)^\inv M^\top (\eye - \Ppigamma)
.
\end{equation}
\begin{proof}
This result follows from two simple facts.
First, the $\mspbe$ can be rewritten as a linear regression problem using
$\Psimat \defeq \Pi_{\Hset, \dsol}(\eye - \Ppigamma) \Xmat$:
\begin{align*}
\mspbe(\vecw)
&= \| \Pi_{\Hset, \dsol} \Xmat \vecw - \Pi_{\Hset, \dsol} r_\pi - \Pi_{\Hset, \dsol} \Ppigamma \Xmat \vecw -  \|_{\dsol}\\
&= \| \Pi_{\Hset, \dsol}(\eye - \Ppigamma) \Xmat \vecw -  \Pi_{\Hset, \dsol} r_\pi \|_{\dsol}
= \| \Psimat \vecw -  \Pi_{\Hset, \dsol} r_\pi \|_{\dsol}
\end{align*}
with solution
\begin{align*}
\vsol = \Xmat \wsol \ \ \ \text{ for } \wsol = (\Psimat^\top \Dsmat \Psimat)^\inv \Psimat^\top \Dsmat \Pi_{\Hset, \dsol} r_\pi
\end{align*}
Second, this solution can be simplified because $\Pi_{\Hset, \dsol}^\top \Dsmat \Pi_{\Hset, \dsol}  = \Dsmat \Pi_{\Hset, \dsol} $. To see why, notice that in the linear case $\Pi_{\Hset, \dsol} = \Phimat (\Phimat^\top \Dsmat \Phimat)^\inv \Phimat^\top \Dsmat$ and so
\begin{align*}
\Pi_{\Hset, \dsol}^\top \Dsmat \Pi_{\Hset, \dsol}
&= \Dsmat \Phimat (\Phimat^\top \Dsmat \Phimat)^\inv \Phimat^\top \Dsmat \Phimat (\Phimat^\top \Dsmat \Phimat)^\inv \Phimat^\top \Dsmat \\
&= \Dsmat \Phimat (\Phimat^\top \Dsmat \Phimat)^\inv \Phimat^\top \Dsmat = \Dsmat \Pi_{\Hset, \dsol}
\end{align*}
Therefore, letting $L \defeq \eye - \Ppigamma$, we get that
\begin{align*}
\Psimat^\top \Dsmat \Psimat &= \Xmat^\top L^\top  \Pi_{\Hset, \dsol}^\top \Dsmat \Pi_{\Hset, \dsol} L \Xmat = \Xmat^\top L^\top \Dsmat \Pi_{\Hset, \dsol} L \Xmat \\
\Psimat^\top \Dsmat\Pi_{\Hset, \dsol} &=  \Xmat^\top L^\top  \Pi_{\Hset, \dsol}^\top\Dsmat\Pi_{\Hset, \dsol} = \Xmat^\top L^\top \Dsmat\Pi_{\Hset, \dsol}
\end{align*}
and so with $M^\top \defeq \Xmat^\top L^\top \Dsmat\Pi_{\Hset, \dsol}$ we have
\begin{align*}
\wsol &= (M^\top L \Xmat)^\inv M^\top r_\pi
\end{align*}
$\Xmat (M^\top L \Xmat)^\inv M^\top L$ is the definition of an oblique projection onto the span of $\Xmat$, orthogonal to the span of $L^\top M$. Since $v_\pi = L^\inv r_\pi$, we get that
\begin{align*}
\vsol &= \Xmat \wsol = \Xmat (M^\top L \Xmat)^\inv M^\top L L^\inv r_\pi = \Xmat (M^\top L \Xmat)^\inv M^\top L v_\pi
\end{align*}
with $\Pi_{\Vset,\ObMat} = \Xmat(M^\top L \Xmat)^\inv M^\top L$, completing the proof.
\proofspace
\end{proof}

\noindent
\textbf{Theorem \ref{thm_approx_h}}
For any $\vgen \in \Vset$,
\begin{align*}
\| v_\pi - \vgen \|_{\dsol}
&\le \| (\eye - \Ppigamma)^\inv \|_{\dsol} \underbrace{\|\Bo v - v \|_{\dsol}}_{\msbe}\\
&\le \| (\eye - \Ppigamma)^\inv \|_{\dsol} \Big( \underbrace{\|\Pi_{\Hset, \dsol} \Bo v - v \|_{\dsol}}_{\mspbe} + \approxerr(\Hset,v) \Big)\\
&\le \| (\eye - \Ppigamma)^\inv \|_{\dsol} \Big(\|\Pi_{\Hset, \dsol} \Bo v - v \|_{\dsol}+ \approxerr(\Hset) \Big)
.
\end{align*}
Note that if $\ctransition < 1$, then $\| (\eye - \Ppigamma)^\inv \|_{\dsol} \le (1-\ctransition)^\inv$.
\begin{proof}
For the first part, we use the same argument as in \citet[Equation 8]{maillard2010finitesample}, but for any $v$ rather than the $v$ that minimizes the $\msbe$. Notice first that $v_\pi = (\eye - \Ppigamma)^\inv r_\pi$ and so $r_\pi = (\eye - \Ppigamma) v_\pi$. This means that
\begin{align*}
\Bo v - v &= r_\pi + \Ppigamma v - v
= (\eye - \Ppigamma) v_\pi + (\Ppigamma - \eye) v\\
&= (\eye - \Ppigamma) (v_\pi - v)\\
\implies v_\pi - v &= (\eye - \Ppigamma)^\inv (\Bo v - v)\\
\implies \|v_\pi - v\|_{\dsol} &\le \| (\eye - \Ppigamma)^\inv \|_{\dsol} \|\Bo v - v \|_{\dsol}
\end{align*}

To bound the approximation error, let us define the conjugate form as $f(v, h) \defeq \sum_s \dsol(s) (2 \mathbb{E}[\delta(v) | S = s]h(s)  - h(s)^2)$. Notice first that $\msbe(v) \ge \mspbe(v)$ for any $v$ because
\begin{align*}
\msbe(v) - \mspbe(v) &= \max_{h \in \Hsetall} f(v, h) - \max_{h \in \Hset} f(v, h)
\end{align*}
and the first maximization is over a broader set, with $\Hset \subset \Hsetall$. Let us additionally define $g(v, h) \defeq \sum_s \dsol(s) (\mathbb{E}[\delta(v) | S = s]  - h(s))^2$, the minimization form for $h$ we used previously. Then $-g(v,h) = f(v,h) - \sum_s d(s) \mathbb{E}[\delta(v) | S = s]^2$. We can therefore rewrite this difference using a minimization
\begin{align*}
\msbe(v) \!-\! \mspbe(v)
&= \! \! \max_{h \in \Hsetall} f(v, h) \!-\! \sum_s d(s) \mathbb{E}[\delta(v) | S = s]^2  \!-\! \max_{h \in \Hset} f(v, h) \!+\! \sum_s d(s) \mathbb{E}[\delta(v) | S = s]^2 \\
&= \max_{h \in \Hsetall} -g(v, h) - \max_{h \in \Hset} -g(v, h) \\
&= \min_{h \in \Hset} g(v, h) - \min_{h \in \Hsetall} g(v, h)
\end{align*}
where on the last step we used the fact that $\max_{h \in \Hset} -g(v, h) = -\min_{h \in \Hset} g(v, h)$.
Notice that $\min_{h \in \Hset} g(v, h) = \| h_\vgen^* - h \|_{\dsol} = \approxerr(\Hset, v)$, because $h_\vgen^*(s) = \mathbb{E}[\delta(v) | S = s]$. Further, $\min_{h \in \Hsetall} g(v, h) = 0$, since $h_\vgen^* \in \Hsetall$. Therefore, we get that
\begin{align*}
\msbe(v) - \mspbe(v)
&= \min_{h \in \Hset} g(v, h) = \approxerr(\Hset, v) \le \approxerr(\Hset)
\end{align*}
\proofspace
\end{proof}

\section{Summary of Known Bounds for Value Estimation}

In the main paper, we derived bounds for the generalized $\mspbe$. These results built on existing bounds, in some cases for more specific settings and in others for slightly different settings. We summarize those known results in Table~\ref{tbl_known}.

\newcommand{\vhat}{\hat{v}}

\begin{longtable}{ p{.45\textwidth} | p{.51\textwidth} }
    \toprule
    Bound & Additional Comments\\
    \midrule

    \citet[Prop 3.1]{williams1993tighta} \newline\newline
        $\| v \!-\! v_\pi \|_\infty \!\leq \! \frac{1}{1 - \gamma} \! \| \Bo v \!-\! v \|_\infty$, any $v$
        & A seminal result showing the relationship between the worst case value error across states and the worst case Bellman error across states.\\

    \addlinespace\midrule\addlinespace

    \citet[Lemma 6]{tsitsiklis1997analysis} \newline\newline
        $\| v - v_\pi \|_{d_\pi} \leq \frac{1-\lambda\gamma}{1 - \gamma} \| \Pi_{\Vset, d_\pi} v_\pi - v_\pi \|_{d_\pi}$
        & A seminal result for the on-policy setting. $v = \argmin_{v \in \Vset} \| \Pi_{\Vset, d_\pi} \Bo v - v \|^2_{d_\pi}$ is obtained with on-policy TD($\lambda$) using linear function approximation. For $\lambda = 1$, the lhs and rhs are equal.
        \\
    \addlinespace\midrule\addlinespace

    \citet[Lemma 10]{antos2008learning} \newline\newline
        $\|\Bo v - v \|_{\dsol} \leq \min_{v \in \Vset} \|\Bo v - v \|_{\dsol} + \max_{v \in \Vset} \min_{v' \in \Vset} \|\Bo v - v' \|_{\dsol}$
        \newline \newline
        The first term is called the inherent Bellman error of $\Vset$ and the second is the inherent one-step Bellman error of $\Vset$.
        & Upper bound on the true Bellman error using the empirical loss that resembles the conjugate Bellman error. This work assumes $\Hset = \Vset$, and $v$ is the solution to the generalized $\mspbe$. The bound in the paper also accounts for using an empirical loss; the result here is obtained from within the proof. See Page 124, and note that $\bar{L}(v) = \min_{v \in \Vset} \bar{L}(v)$.
        \\
    \addlinespace\midrule\addlinespace

    \citet[Equation 8]{maillard2010finitesample} \newline
        $\|v - v_\pi \|_{\dsol} \leq \| (\eye - \Ppigamma)^\inv \|_{\dsol} \|\Bo v - v \|_{\dsol}$
        \newline for $v = \argmin_{v \in \Vset} \|\Bo v - v \|_{\dsol}$
        & The original result is given for a constant $\gamma$. We include the more generic result here, since it is a straightforward extension.
        \\
    \addlinespace\midrule\addlinespace

    \citet[Proposition 3]{scherrer2010should} \newline
        $\|v_M - v_\pi \|_{\dsol} $\newline $\leq \sqrt{\text{\tiny spectral-radius}(ABCB^\top)} \|\Pi_{\Vset, \dsol} v_\pi - v_\pi \|_{\dsol}$
        \newline\newline
         $v_M = \Xmat (M^\top \Xmat)^\inv M^\top v_\pi$ is the oblique projection of $v_\pi$ onto the span of $\Xmat$---linear functions with features $\vecx(s)$---orthogonally to span of matrix $L^\top M \in \mathbb{R}^{n \times d}$ for $L = (\eye - \Ppigamma)$.
        & This result assumes linear function approximation.
        If $M = \Dsmat \Xmat$, then $v_M$ is the TD fixed point. If $M = \Dsmat L \Xmat$, then $v_M$ is a solution to the $\msbe$.
        The spectral radius of a matrix is the maximum of the absolute value of the eigenvalues. This result is for matrices $A = \Xmat^\top \Dsmat \Xmat$, $B = (M^\top L \Xmat)^\inv$ and $C = M L \Dsmat^\inv L^\top M$. This result extends and simplifies the earlier results in \citep{schoknecht2003optimality} and in \citep{yu2010error}.
        \\
    \addlinespace\midrule\addlinespace

    \citet[Theorem 2]{kolter2011fixed} \newline\newline
        $\|v - v_\pi \|_{d_\mu} \leq \frac{1 + \gamma \text{cond}(\bar D)}{1 - \gamma} \|\Pi_{\Vset, d_\mu} v_\pi - v_\pi \|_{d_\mu}$
        \newline\newline
        $v$ obtained with Off-policy LSTD using linear function approximation with conditions on $d_\mu$.
        & $\text{cond}(\bar D)$ is the condition number of the matrix $\bar D \doteq D_\mu^{-1/2}D_\pi^{1/2}$. For these diagonal matrices, $\text{cond}(\bar D) = \kappa(d_\mu, d_\pi)/ \kappa(d_\pi, d_\mu)$. This bound assumes $d_\mu$ satisfies a certain relaxed contraction property, that states that applying the transition to the values and then projecting is a non-expansion under $d_\mu$ (see \citep[Theorem 2]{kolter2011fixed}).
        \\
    \addlinespace\midrule\addlinespace

    \citet[Corollary 1]{hallak2016generalized} \newline\newline
           $ \| v - v_\pi \|_{d_\mu} \!\le \! C(d_\mu, m)^{\invhalf} \| \Pi_{\Vset, m} v_\pi - v_\pi \|_{m}$
           $ \| v - v_\pi \|_{m}  $

           $\le (\gamma C(d_\mu, m))^{\invhalf} \| \Pi_{\Vset, m} v_\pi - v_\pi \|_{m}$
           \newline for $C(d_\mu, m) \defeq 1 - \tfrac{\gamma^2}{\beta} (1-\kappa(d_\mu, m))$.
        & Bounds provided for ETD($\lambda=0, \beta$) under linear function approximation to obtain $v$. Note that $\beta \le \gamma$ and $\beta = \gamma$ corresponds to ETD(0). Here, $\gamma$ is a constant less than 1. The result is also extended to $\lambda > 0$ in \citet[Corollary 1]{hallak2016generalized}, by computing $s_{m}$ for the Bellman operator.
        \\
            \addlinespace\midrule\addlinespace

    \citet[Theorem 1]{white2017unifying} \newline\newline
        $\|v - v_\pi \|_d \leq (1-s_d)^\inv \| \Pi_{\Vset, d} v_\pi - v_\pi \|_d$
        & Generalized results in \citet{hallak2016generalized}, showing the Bellman operator is a contraction for transition-based discounting. This result was shown for $d = d_\pi$ and $d = m$, under linear function approximation with $v$ the solution to the linear $\mspbe$ with weighting $d$.
        \\
    \addlinespace\midrule\addlinespace
\caption{A summary of known bounds on the $\msve$ of the solution to the $\mspbe$ and $\msbe$. }\label{tbl_known}
\end{longtable}

\section{Extensions to n-step returns and $\lambda$-returns}\label{app_nstep}

In this section, we provide the TDRC algorithm under $n$-step returns and $\lambda$-returns. We start by discussing how the generalized $\mspbe$ extends, and then the resulting algorithm.

 There are further $n$-step variants of these objectives, where bootstrapping occurs only after $n$ steps: $(G_{t,n} - \paramv[S_t])^2$ where $G_{t, n} = R_{t+1} + \gamma_{t+1} R_{t+2} + \ldots +  \gamma_{t+1:t+n} R_{t+n} +  \gamma_{t+1:t+n+1} \paramv[S_{t+n+1}]$ where $\gamma_{t+1:t+n} = \gamma_{t+1} \gamma_{t+2} \ldots \gamma_{t+n}$. The extreme of $n$-step returns is to use the full return with no bootstrapping, as in Monte Carlo methods, with the objective becoming the $\msre$. The conjugate form and derivations above extend to $n$-step returns, simply by considering the $n$-step Bellman operator and corresponding $n$-step TD error:
\begin{align*}
\delta^{(n)}(\vecw) \defeq R_{t+1} + \gamma_{t+1}  R_{t+2} + \ldots +  \gamma_{t+1:t+n} R_{t+n} +  \gamma_{t+1:t+n+1} \paramv[S_{t+n+1}] - \paramv[S_t]
\end{align*}
with importance sampling ratios included, in the off-policy setting.
The $n$-step generalized $\mspbe$ is $\max_{h \in \Hset } \mathbb{E}_d[2\CEpi{\delta^{(n)}(\vecw)}{S=s}h(s) - h(s)^2]$. The function $h$ is trying to estimate the expected $n$-step return from $s$: $\mathbb{E}_d \left[ 2\CEpi{\delta^{(n)}(\vecw)}{S=s} \right]$.
The saddlepoint update for $\vecw$ is
\begin{equation*}
\Delta \vecw \gets h(S_t)\left( \nabla_\vecw \paramv[S_t] - \paramh[S_t] \gamma_{t+1:t+n+1} \nabla_\vecw \paramv[S_{t+n+1}] \right)
\end{equation*}
and the gradient correction update is
\begin{equation*}
\Delta \vecw \gets \delta^{(n)}(\vecw) \nabla_\vecw \paramv[S_t] - \paramh[S_t]\gamma_{t+1:t+n+1} \nabla_\vecw \paramv[S_{t+n+1}]
\end{equation*}
where both use the same update for $h$:
\begin{equation*}
\Delta \vech \gets -(\delta^{(n)}(\vecw) - \paramh[S_t])\nabla_\vech \paramh[S_t]
\end{equation*}

The primary difference when considering $n$-step returns is that, for large $n$, it is less necessary to estimate $h$. For large $n$, the correlation between $\delta^{(n)}(\vecw)$ and $\paramv[S_{t+n+1}]$ becomes smaller. Consequently, it would not be unreasonable to use $\delta^{(n)}(\vecw) \nabla_\vecw \paramv[S_t] - \delta^{(n)}(\vecw)\gamma_{t+1:t+n} \nabla_\vecw \paramv[S_{t+n+1}]$, as the incurred bias is likely small. Further, if the discount per step is less than 1, then the gradient correction term also diminishes in importance, because it is pre-multiplied by $\gamma_{t+1:t+n+1}$. For example, for a constant $\gamma < 1$, we get $\gamma_{t+1:t+n} = \gamma^n$. One might expect that the gradient correction update might have an even greater advantage here over the saddlepoint update. It remains an open question as to the relationship between $n$ and some of these choices.

The other generalization, with $\lambda$-returns, involves mixing n-step returns with a potentially state-dependent parameter $\lambda$. For a constant $\lambda \in [0,1]$, the $\lambda$-return is
\begin{equation*}
G_{t, \lambda} = (1-\lambda) \sum_{i=1}^\infty \lambda^{i-1} G_{t, i}
= R_{t+1} + \gamma_{t+1}[(1-\lambda) \paramv[S_{t+1}] + \lambda G_{t+1, \lambda}]
.
\end{equation*}
More generally, for a state-dependent $\lambda: \States \rightarrow [0,1]$ \citep{maei2010gq}, with $\lambda_{t+1} \defeq  \lambda(S_{t+1})$, we have
$G_{t, \lambda} = R_{t+1} + \gamma_{t+1}[(1-\lambda_{t+1}) \paramv[S_{t+1}] + \lambda_{t+1} G_{t+1, \lambda}]$. This target is used in place of $G_{t, n}$ to define the $\mspbe$ for $\lambda$-returns.

Notice, however, that the $\lambda$-return involves an infinite sum into the future, or until the end of an episode. For episodic problems, these targets can be computed, like Monte carlo methods. However, more typically to optimize the $\mspbe$ for $\lambda$-returns, the gradient is estimated using a backwards view, rather than this forward view \citep{sutton2018reinforcement}. This reformulation only holds under linear function approximation. It remains an open question how effective the approach is in the nonlinear setting.

\section{Summary of Algorithms for the Linear $\mspbe$}

In this section we summarize the current algorithms that optimize the linear $\mspbe$ by the implicit weighting they use. The two axes indicate whether prior or posterior corrections are used, and whether the weighting includes $d_\pi$ or $d_b$. The emphatic weighting adjusts the state weighting with prior corrections, but still includes $d_b$ as part of the weighting in the objective. Alternative-life, on the other hand, removes $d_b$ all together.

The primary differences between the algorithms is in how they optimize their respective objectives. They can use gradient updates or hybrid updates, and can incorporate other additions like action-dependent bootstrapping. For example, we can easily obtain a gradient version of ETD, by incorporating emphatic weights into the GTD algorithm. The GTD($\lambda$) algorithm is
\begin{align*}
\vecz_t^{\rho} \leftarrow& ~\rho_t \left(\gamma_t \lambda \vecz_{t-1}^{\rho} + \vecx_{t}\right) \\
\hparams_{t+1} \leftarrow& ~\hparams_t + \alpha_h\bigl[\delta_t\vecz^\rho_t - (\hparams_t^{\tr}\vecx)\vecx_{t}  \bigr] \nonumber \\
\vecw_{t+1} \leftarrow& ~\vecw_t + \alpha \delta_t \vecz^\rho_t - \underbrace{\alpha \gamma_{t+1} (1-\lambda)(\hparams_t^{\tr}\vecz^\rho_t)\vecx_{t+1}}_{\text{correction term}}
\end{align*}
The Emphatic GTD($\lambda$) algorithm uses the same updates to the two sets of weights, but with emphatic weights in the trace: $\vecz_t^{\rho} \leftarrow ~\rho_t \left(\gamma_t \lambda \vecz_{t-1}^{\rho} + M_t \vecx_{t}\right)$. More simply, when considering Emphatic GTD(0), we get
\begin{align*}
F_t \leftarrow& ~\rho_{t-1}\gamma_t F_{t-1} + 1\\
\hparams_{t+1} \leftarrow& ~\hparams_t + \alpha_h \rho_t F_t \bigl[\delta_t - (\hparams_t^{\tr}\vecx)\bigr]\vecx_{t} \\
\vecw_{t+1} \leftarrow& ~\vecw_t + \alpha \rho_t F_t [\delta_t \vecx_t - \gamma_{t+1} (1-\lambda)(\hparams_t^{\tr}\vecx_t)\vecx_{t+1}]
\end{align*}
where $F_t$ is omitted for standard GTD(0).

Many of these algorithms can incorporate different weightings by pre-multiplying the update, but might use different approaches to approximate the gradient. It is important to separate the weighting and algorithmic approach, as otherwise, comparisons between algorithms (e.g., ETD and GTD) could reflect either the difference in objective (weighting by $m$ or $d_b$) or in the algorithmic approach (TD-style updates or gradient updates).
The algorithms, categorized according to weightings, are summarized in Table~\ref{tab_methods}. For a more complete list of algorithms and their updates, see \citep[Appendix A]{white2016investigating}.

\definecolor{mygray}{gray}{0.6}

\begin{table}[ht]
    \begin{tabularx}{\linewidth}{ X p{3.5cm} p{8.5cm}  }
        & \multicolumn{1}{c}{$d_\pi$} & \multicolumn{1}{c}{$d_b$} \\
        \midrule
        Posterior corrections
        &
        \begin{itemize}[wide]
            \item[$\triangleright$] N/A. \newline Alternative life \newline algorithms cannot only do posterior corrections.
        \end{itemize}
        &
        \begin{itemize}[wide]
            \item[$\triangleright$] Off-Policy TD($\lambda$)
            \item[$\triangleright$] GTD($\lambda$) \citep{sutton2009fast},
            \item[$\triangleright$] HTD($\lambda$) \citep{maei2011gradient, white2016investigating}
            \item[$\triangleright$] Action-dependent bootstrapping \begin{itemize}[wide, topsep=0ex]
                \item[$\quad-$] Tree Backup($\lambda$) \citep{precup2000eligibility}
                \item[$\quad-$] V-trace($\lambda$) \citep{espeholt2018impala}
                \item[$\quad-$] AB-Trace($\lambda$) \citep{mahmood2017multistep}
            \end{itemize}
            \item[$\triangleright$] Saddlepoint methods for GTD2($\lambda$) \begin{itemize}[wide, topsep=0ex]
                \item[$\quad-$] GTD2-MP($\lambda$) \citep{liu2016proximal}
                \item[$\quad-$] SVRG and SAGA \citep{du2017stochastic}
                \item[$\quad-$] Gradient Tree Backup($\lambda$) \citep{touati2018convergent}
            \end{itemize}

        \end{itemize} \\
        \midrule
        Prior \& Posterior corrections
        &
        \begin{itemize}[wide]
            \item[$\triangleright$] Alternative-life TD($\lambda$)
            \item[$\triangleright$] {\color{mygray} Alternative-life GTD($\lambda$), HTD($\lambda$) \& Saddlepoint methods}
        \end{itemize}
        &
        \begin{itemize}[wide]
            \item[$\triangleright$] ETD($\lambda$) \citep{sutton2016emphatic}
            \item[$\triangleright$] ETD($\lambda, \beta$) \citep{hallak2016generalized}
            \item[$\triangleright$] {\color{mygray} Emphatic GTD($\lambda$), HTD($\lambda$) \& Saddlepoint methods}
        \end{itemize}
        \\
    \end{tabularx}
    \caption{\label{tab_methods}
        A summary of off-policy value estimation methods for the linear $\mspbe$, based on weightings in the objective and whether they incorporate prior and posterior corrections. The algorithms in grey are hypothetical algorithms that can easily be derived by applying the same derivations as in their original works, but with alternative weightings.
    }
\end{table}

\section{Impact of Voting Procedure}\label{app:additional-results}
In Section~\ref{sec:benchmark-experiments} we investigate the performance of the benchmark algorithms across four simulation environments with a single hyperparameter setting chosen for each algorithm via a voting procedure.
In this section, we investigate the impact of the voting procedure used on the conclusions drawn from the results in Figure~\ref{fig:bakeoff}.
We start by investigating each algorithm given that we have the ability to tune hyperparameters individually for every domain.
We then investigate the impact of excluding Lunar Lander from the voting procedure, allowing all algorithms (especially SBEED) to choose parameters without needing to balance performance on the most challenging domain.
Finally, we demonstrate that a more simple voting procedure leads towards less favorable results for algorithms which have higher sensitivity to hyperparameters across domains.

\begin{figure}
    \centering
    \includegraphics[width=0.9\linewidth]{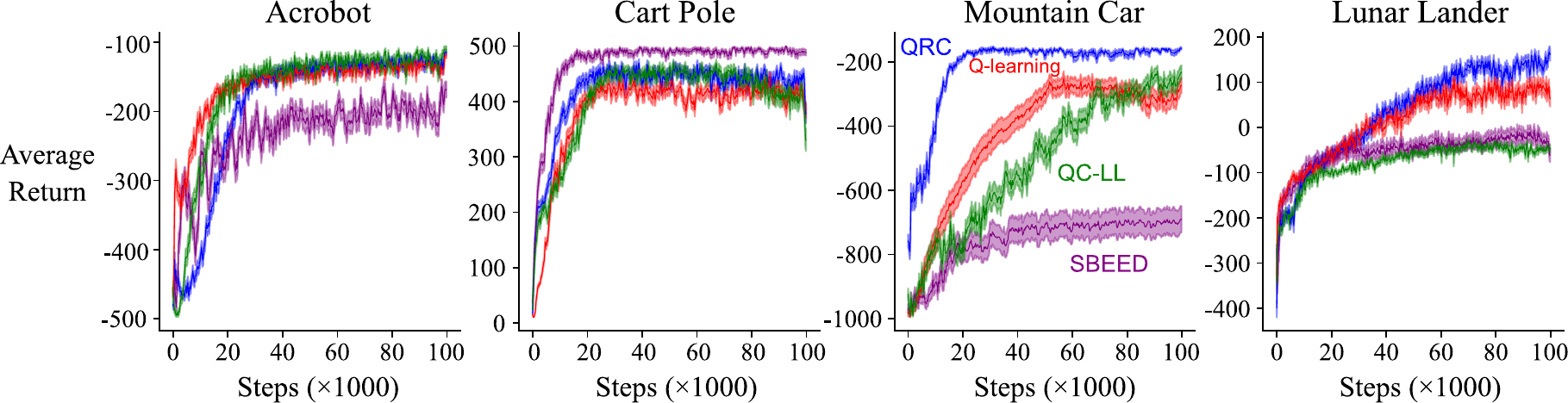}

    \caption{\label{fig:per-domain-tuned-curves}
        Average return over number of learning steps for best hyperparameters chosen for each algorithm on each domain.
        When tuning hyperparameters on a per domain basis, the gaps between algorithm performance begin to decrease.
        QRC still consistently is among the top performing algorithms in each benchmark domain.
        The performance difference between SBEED and the other competitors on Cart Pole can be explained by the forced stochastic $\epsilon$-greedy policy used by action value methods compared to the increasingly deterministic policy learned by SBEED.\@
    }
\end{figure}

Figure~\ref{fig:per-domain-tuned-curves} shows the performance for each algorithm for their best performing respective hyperparameter on each domain.
Because every algorithm is able to maximize their individual performance for each domain, and because many algorithms successfully solve each domain with their best parameters, performance gaps between algorithms shrink considerably.
Still only QRC is able to successfully solve the Mountain Car domain consistently across runs and both QC-LL and SBEED find a sub-optimal policy in Lunar Lander.

Figure~\ref{fig:per-domain-tuned-curves} is an idealized setting where every domain gets to pick their favorite parameter setting through extensive tuning.
A more realistic setting would involve picking a single ``default'' hyperparameter setting and perhaps a small exploration of parameters along a small number of axes (for example choosing default parameters in the ADAM optimizer, but sweeping a few values of the stepsize).
Because it is often computationally infeasible to tune hyperparameters extensively for each domain, this results in a compromise where some algorithms get to use parameters that perform well on only a subset of the domains.
Figure~\ref{fig:bakeoff} illustrates this tradeoff by requiring each algorithm to choose a single hyperparameter setting that balances performance across all domains.
For algorithms that are highly insensitive to their hyperparameters, this selection process is easy and the algorithm need not tradeoff performance significantly in each domain.
For algorithms that are highly sensitive to their hyperparameters, this selection process is very challenging and can result in hyperparameter settings that are catastrophic on some domains but highly performing on others.

In Section~\ref{sec:benchmark-experiments}, we used a pair-wise voting strategy, Condorcet voting, to give each algorithm maximum opportunity to balance hyperparameters across all possible pairs of domains.
Because Lunar Lander is a much more challenging domain than the rest, the inclusion of Lunar Lander significantly impacted the outcome of the voting procedure.
To understand the impact of this, in Figure~\ref{fig:vote-no-lunarlander} we use the same voting mechanism as in Figure~\ref{fig:bakeoff} except we exclude Lunar Lander from influencing the outcome of the vote.
In this case, each algorithm chooses the best hyperparameter setting with a pair-wise vote across Cart Pole, Mountain Car, and Acrobot, then we show performance for that hyperparameter setting across all four domains.
This can additionally be viewed as a form of hold-out set (i.e.\ leave-one-out cross validation), where we test how the hyperparameter choosing strategy generalizes to a hand-picked challenging sample.
In this setting, all algorithms suffer lower performance in Lunar Lander, particularly Q-learning and SBEED.\@

Finally, in Figure~\ref{fig:vote-instant-runoff} we investigate the impact of a different voting strategy across all four domains.
We use an instant runoff vote which is generally known to provide an inferior between outlying ``candidates'' than pair-wise voting procedures.
Our results show that QC-LL and SBEED are both strongly negatively impacted by this choice of voting procedure where Q-learning and QRC achieve similar performance as in Figure~\ref{fig:bakeoff}.

\begin{figure}[h!]
    \centering
    \includegraphics[width=0.9\linewidth]{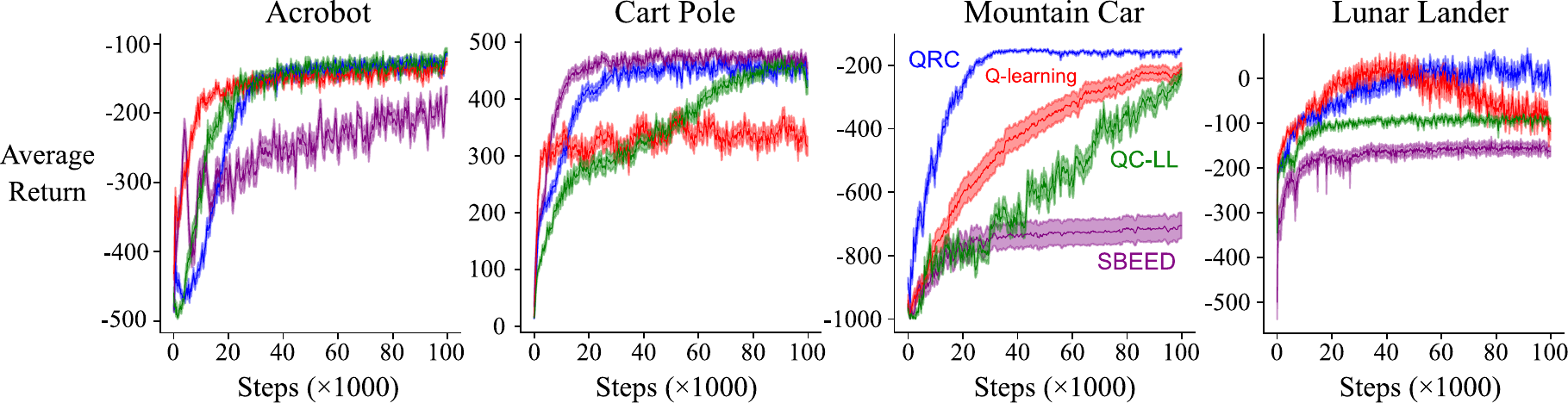}

    \caption{\label{fig:vote-no-lunarlander}
        Average return over number of learning steps for hyperparameters chosen via Condorcet voting.
        The performance of each algorithm on Lunar Lander was excluded from impacting the result of the vote and can be treated as a hold-out set.
    }
\end{figure}

\begin{figure}[h!]
    \centering
    \includegraphics[width=0.9\linewidth]{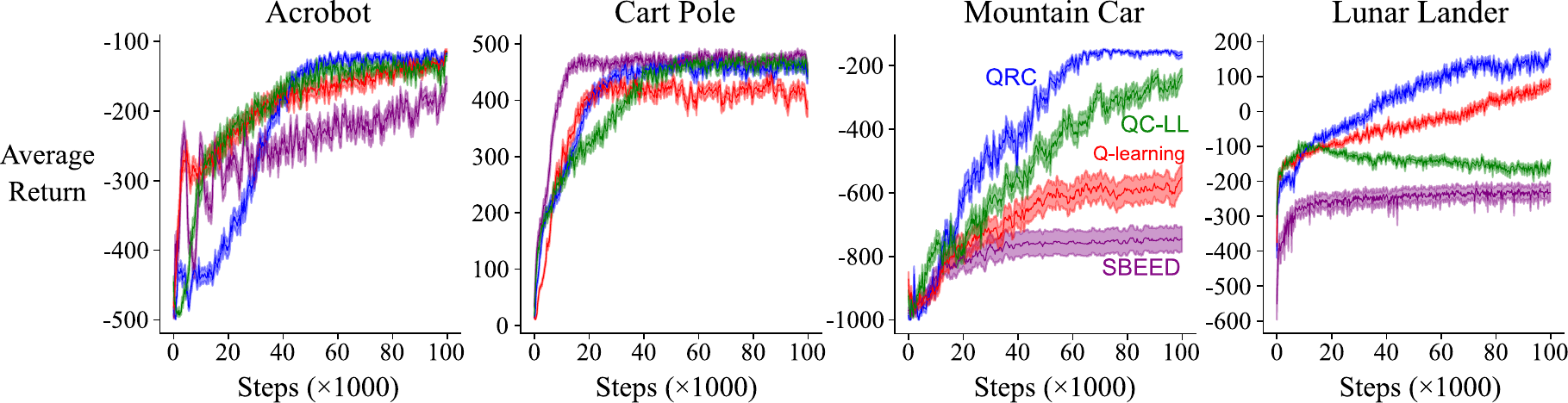}

    \caption{\label{fig:vote-instant-runoff}
        Average return over number of learning steps for hyperparameters chosen via instant runoff voting.
        For algorithms other than QRC, the instant runoff voting procedure could not find a suitable parameter setting which balanced performance across all domains.
    }
\end{figure}

\end{document}